\documentclass[letterpaper, 10 pt, journal, twoside]{IEEEtran}
\IEEEoverridecommandlockouts
\usepackage{amsmath,amssymb, amsfonts}
\usepackage{algpseudocode}
\usepackage{algorithm}
\usepackage{array}
\usepackage{textcomp}
\usepackage{stfloats}
\usepackage{url}
\usepackage{verbatim}
\usepackage{graphicx}
\usepackage{cite}
\usepackage{subcaption}
\usepackage[table]{xcolor}

\usepackage{hyperref}
\usepackage{caption}
\usepackage{hypcap}
\usepackage{mathtools}
\usepackage{bm}
\usepackage{color}
\usepackage{algorithm}
\usepackage{tabulary}
\usepackage{array,tabularx,booktabs}
\newcolumntype{L}[1]{>{\raggedright\arraybackslash}p{#1}} 
\usepackage{booktabs}
\usepackage{multirow}
\usepackage{gensymb}
\usepackage{float}
\usepackage{pifont}
\usepackage{amsfonts}
\usepackage{amsthm}
\usepackage{setspace}
\usepackage{mathrsfs}
\usepackage[normalem]{ulem} 
\usepackage{threeparttable}
\usepackage{diagbox}

\newtheorem{proposition}{Proposition}

\newtheorem{example}{Example}

\newcommand{\RED}[1]{\textcolor{red}{#1}}

\newcommand{\fullwidth}{\textwidth}
\newcommand{\halfwidth}{\columnwidth}

\definecolor{firstcolor}{RGB}{224,239,255}
\definecolor{secondcolor}{RGB}{255,224,224}

\newcommand{\cfirst}{\colorbox{firstcolor}}
\newcommand{\csecond}{\colorbox{secondcolor}}

\DeclareMathOperator*{\Exp}{Exp}
\DeclareMathOperator*{\Log}{Log}

\newcommand{\CT}{CT-RIO}

\begin{document}

\title{
    Parallel Reference-Centric Continuous-Time Relative Localization with Augmented Clamped Non-Uniform B-Splines
}

\author{
    Jiadong Lu\textsuperscript{$\dagger$ 1,2},
    Zhehan Li\textsuperscript{$\dagger$ 1,2},
    Tao Han\textsuperscript{3},
    Miao Xu\textsuperscript{4},
    Jiajun Lv\textsuperscript{1},
    Chao Xu\textsuperscript{1,2},
    and
    Yanjun Cao\textsuperscript{1,2}    
    \thanks{\textsuperscript{$\dagger$} \textbf{Equal contribution.}}
    \thanks{
        This work was supported by National Nature Science Foundation of China.
    }
    \thanks{\textsuperscript{1} State Key Laboratory of Industrial Control Technology, Institute of Cyber-Systems and Control, Zhejiang University, Hangzhou, 310027, China.}
    \thanks{\textsuperscript{2} Huzhou Institute of Zhejiang University, Huzhou, 313000, China.}
    \thanks{\textsuperscript{3} The School of Automation and Intellegent Sensing, Shanghai Jiao Tong University, Shanghai, 200240, China.}
    \thanks{\textsuperscript{4} Department of Automatic Control, Faculty of Engineering, Lund University, Lund, 22100, Sweden.}
    \thanks{
        E-mails:\tt\small \{jdlu, zhehanli, yanjunhi\}@zju.edu.cn
    }
    \vspace{-15pt}    
}


\maketitle

\begin{abstract}
    Accurate relative localization is critical for multi-robot cooperation. 
    In robot groups, measurements from different robots arrive asynchronously and with clock time-offsets. 
    Although Continuous-Time (CT) formulations have proved effective for handling asynchronous measurements in single-robot SLAM and calibration, extending CT methods to multi-robot settings faces great challenges in achieving high-accuracy, low-latency, and high-frequency performance.
    In particular, existing CT methods suffer from the inherent query-time delay of unclamped B-splines and high optimization latency.
    This paper proposes CT-RIO, a novel Continuous-Time Relative-Inertial Odometry framework.
    We adopt Clamped Non-Uniform B-splines (C-NUBS) to represent states, eliminating the query-time delay.
    We further augment C-NUBS with closed-form extension and shrinkage operations that preserve the spline shape, making it suitable for online estimation and enabling flexible knot management.
    This flexibility leads to the concept of a \textit{knot-keyknot} strategy, which supports spline extension at high frequency while retaining sparse keyknots for adaptive relative motion modeling.
    We then formulate a reference-centric sliding-window relative localization problem that operates purely on relative kinematics and inter-robot constraints.
    To enable low-latency and high-frequency estimation, we decompose the tightly coupled optimization into robot-wise subproblems and solve them in parallel using asynchronous block coordinate descent.
    Extensive experiments show that CT-RIO converges from time-offsets as large as $263\,\mathrm{ms}$ to sub-millisecond within $3\,\mathrm{s}$, and achieves RMSEs of $0.046\,\mathrm{m}$ and $1.8^\circ$.
    It consistently outperforms evaluated published methods, with improvements of up to $60\%$ under high-speed motion.
\end{abstract}



\begin{IEEEkeywords}
    Multi-robot systems,
    sensor fusion,
    localization,
    continuous-time trajectory.
\end{IEEEkeywords}


\vspace{-8pt}
\section{Introduction}
\label{Sec: Introduction}

\IEEEPARstart{R}{ecently}, multi-robot systems have been widely employed in diverse domains, including search and rescue \cite{sherman2018cooperative,tian2020search} and environmental exploration \cite{gao2022meeting,zhou2023racer}.
A fundamental capability enabling such tasks is relative localization within the group \cite{zhou2008robot,aghili2016robust,xu2022omni,chen2025relative,olson2011apriltag,faessler2014monocular,fishberg2024murp,cossette2022optimal,xun2023crepes,armani2024accurately,zhao2025robust,shalaby2024multi,li2025crepes}.
Relative localization can be divided into indirect and direct~\cite{li2025crepes}. 
Reference-centric direct relative localization, in which robot states are represented with respect to a designated reference robot, has been widely studied.
Existing approaches estimate these relative states from measurements provided by visual markers \cite{olson2011apriltag,faessler2014monocular}, UWB arrays \cite{fishberg2024murp,cossette2022optimal}, and inertial sensors \cite{xun2023crepes,armani2024accurately,zhao2025robust,li2025crepes,shalaby2024multi}.
A representative example is CREPES-X \cite{li2025crepes}, which tightly couples distance, bearing, and inertial measurements and demonstrates competitive accuracy using a keyframe-based discrete-time (DT) formulation.

\begin{figure}[t]
    
    \centering
    \includegraphics[width=\halfwidth]{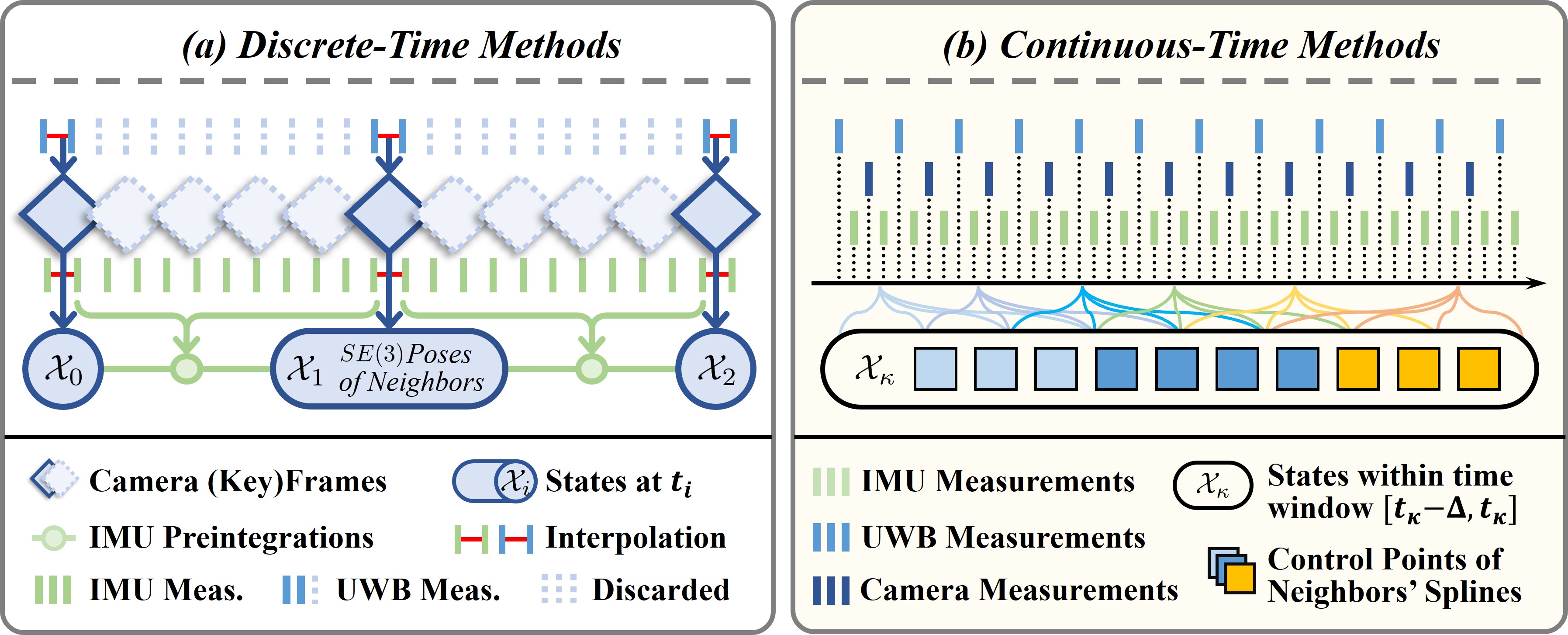}
    \caption{
        Comparison between keyframe-based DT and CT relative localization methods.
        (a) CREPES-X~\cite{li2025crepes} estimates states $\{\mathcal{X}_0, \mathcal{X}_1, \ldots\}$ only at camera keyframes.
        Asynchronous measurements (e.g., UWB) are temporally interpolated to keyframe timestamps, IMU data are summarized via preintegration, and intermediate frames and measurements are discarded.
        (b) CT methods parameterize the state $\mathcal{X}_{\kappa}$ with B-splines, allowing all measurements to directly constrain the trajectory at their native timestamps.
    }
    \label{Fig:headfigure2}
\end{figure}

However, practical multi-robot systems still pose two key challenges:
\textit{(i) Asynchronous measurements:}
Sensors on different robots operate at varying rates, so their measurements generally do not coincide with the discrete states maintained by the estimator.
Keyframe-based DT methods need to associate such measurements through interpolation, preintegration, or other temporal reconstruction, whose accuracy depends on the adopted model between discrete states.
For example, CREPES-X~\cite{li2025crepes} discards some inter-frame observations during this process, as shown in \autoref{Fig:headfigure2}(a).
\textit{(ii) Inter-robot time-offsets:}
Different robots often operate with unsynchronized clocks.
Some existing multi-robot systems rely solely on network synchronization, whose errors can reach the order of $10^2$~ms \cite{chen2023understanding}, while others use dedicated hardware such as UWB \cite{xun2023crepes,li2025crepes}.
To avoid additional hardware, several approaches estimate time-offsets by assuming short-term constant velocity~\cite{shalaby2024multi,wen2024simultaneous}.
However, this assumption is often violated in highly dynamic multi-robot settings.

\textbf{Continuous-time (CT) representation:}
In contrast to keyframe-based DT methods, B-spline-based continuous-time representation incorporates observations directly at their native timestamps without separate interpolation.
Measurements directly constrain the locally supporting control points, while the interpolation itself is jointly optimized with these control points, naturally supporting high-rate, multi-rate, and asynchronous measurements (\autoref{Fig:headfigure2}(b)).
Moreover, analytic B-spline derivatives enable joint estimation of time-offsets without introducing additional motion assumptions.
This property has been successfully exploited for single-robot sensor time synchronization~\cite{cioffi2022continuous,lv2023continuous}.
Owing to these advantages, B-spline trajectory representations have become increasingly popular in both SLAM~\cite{mueggler2018continuous,cioffi2022continuous,lv2021clins,lang2022ctrl,lv2023continuous,lang2023coco,hug2024hyperion,lang2025gaussian} and calibration tasks~\cite{rehder2016extending,sommer2020efficient,chen2025ikalibr}.
Recent advances further demonstrate that non-uniform knot placement can significantly improve both accuracy and computational efficiency~\cite{lang2023coco}.
In multi-robot systems, where motion characteristics and relative motion intensity may vary substantially across robots, Non-Uniform B-Splines (NUBS) are particularly well suited.
They allow modeling capacity to be concentrated in dynamically rich segments of motion while avoiding redundant representation in less informative regions.

Latency is a major obstacle to applying CT methods, especially NUBS, to multi-robot systems.
We distinguish three sources of latency: measurement, model-induced, and optimization latency.
Measurement latency is intrinsic to the sensing pipeline.
Model-induced latency depends on the trajectory representation.
For B-splines, the queryable trajectory is always delayed behind the newest knot by the last $k\!-\!2$ knot intervals.
Uniform B-splines can construct future control points in advance using the fixed knot interval, as shown in \autoref{Fig: basismatrix}(a).
For unclamped NUBS, however, future knot intervals are unknown, making this query delay inherent.
For example, an unclamped NUBS of order-$4$ with $100\,\mathrm{ms}$ knot spacing incurs a $200\,\mathrm{ms}$ query delay, as shown in \autoref{Fig: basismatrix}(b) and \autoref{Fig: Ablation}(a).
Optimization latency is further increased by inter-robot residuals that couple control points across trajectories.
Our method therefore targets the latter two estimator-side sources.
Unlike in single-robot systems, these delays cannot be reliably bridged by IMU propagation in robocentric relative localization, since reference-rotation uncertainty causes position error to grow with both duration and inter-robot distance, motivating near-zero-lag optimized outputs.

To eliminate the model-induced latency, we adopt classical Clamped NUBS (C-NUBS) for CT state estimation.
To the best of our knowledge, this is the first use of C-NUBS as a state representation for online state estimation.
By repeating boundary knots, C-NUBS clamp the evaluation range at the latest knot and thereby eliminate the intrinsic query delay of unclamped splines, as illustrated in \autoref{Fig: basismatrix}(c) and \autoref{Fig: Ablation}(b).

While clamping resolves this model-induced latency, it introduces a new challenge for incremental state estimation.
Specifically, adding or removing the last spline segment alters the former spline shape, breaking the consistency of previously optimized trajectory segments.
We therefore augment C-NUBS by deriving closed-form spline extension and shrinkage algorithms that preserve the spline shape exactly, allowing new spline segments to be appended or removed flexibly without affecting existing estimates.
Building on this flexibility, we propose a \textit{knot-keyknot} strategy, in which knots are added at high frequency while only sparse keyknots are retained within a sliding window to ensure computational efficiency, analogous to the frame-keyframe structure in DT systems.

For the optimization latency, we decompose the global problem into robot-wise subproblems, each solvable within milliseconds, yielding a Block Coordinate Descent (BCD) formulation with one trajectory per block.
A distinctive feature of our approach is the exploitation of the trajectory's temporal structure: as new knots are incrementally appended, the corresponding subproblems are triggered for immediate reoptimization.
Each subproblem is solved in a separate thread without waiting for other subproblems to complete their updates, eliminating the costly idle time and enabling high-frequency output for each robot with minimal latency, as illustrated in \autoref{Fig: Ablation}(c).
Despite the lack of explicit variable synchronization, we derive that this Incremental Asynchronous BCD (IA-BCD) scheme converges to a stationary point of the full problem under appropriate step-size conditions.

Based on these, we present a Continuous-Time Relative-Inertial Odometry (CT-RIO) framework which is designed for high-accuracy, low-latency, and high-frequency estimation of relative poses and inter-robot time-offsets.

\begin{figure}[t]
    
    \centering
    \includegraphics[width=.95\halfwidth]{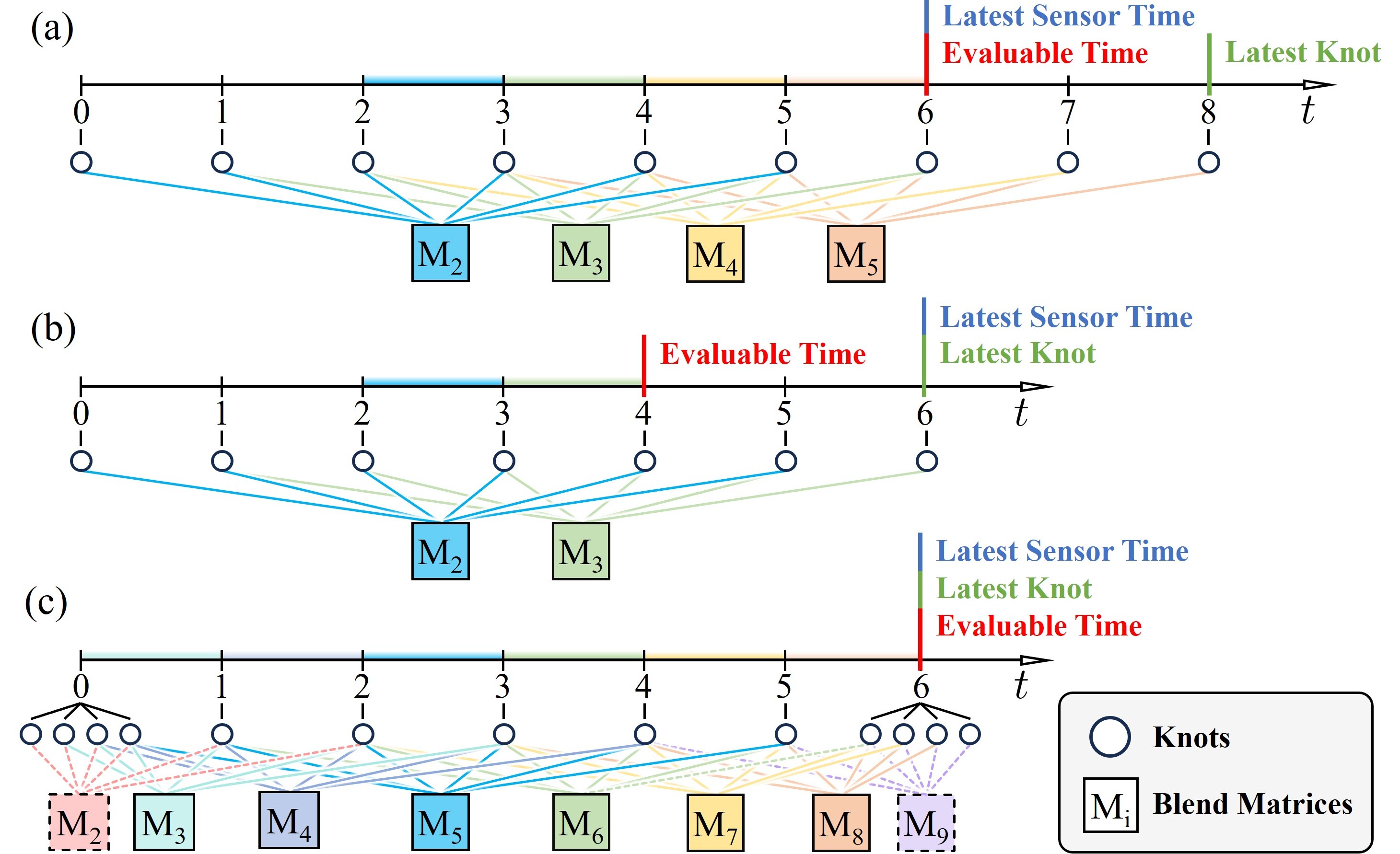}
    \caption{
        Three different spline parameterizations ($k=4$) with the earliest sensor time at $0\,\mathrm{s}$ and the latest sensor time at $6\,\mathrm{s}$.
        The unclamped B-splines benefit from the \textit{Extended Query Interval Property} (see details in \ref{eqip}), which provides one additional evaluable segment.
        (a) Uniform B-spline.
        (b) Unclamped non-uniform B-spline.
        (c) Clamped non-uniform B-spline.
        The dashed blend matrices at both ends correspond to zero-length knot spans induced by repeated boundary knots and are therefore invalid.
    }
    \label{Fig: basismatrix}
\end{figure}

\begin{figure}[t]
    
    \centering
    \includegraphics[width=.98\halfwidth]{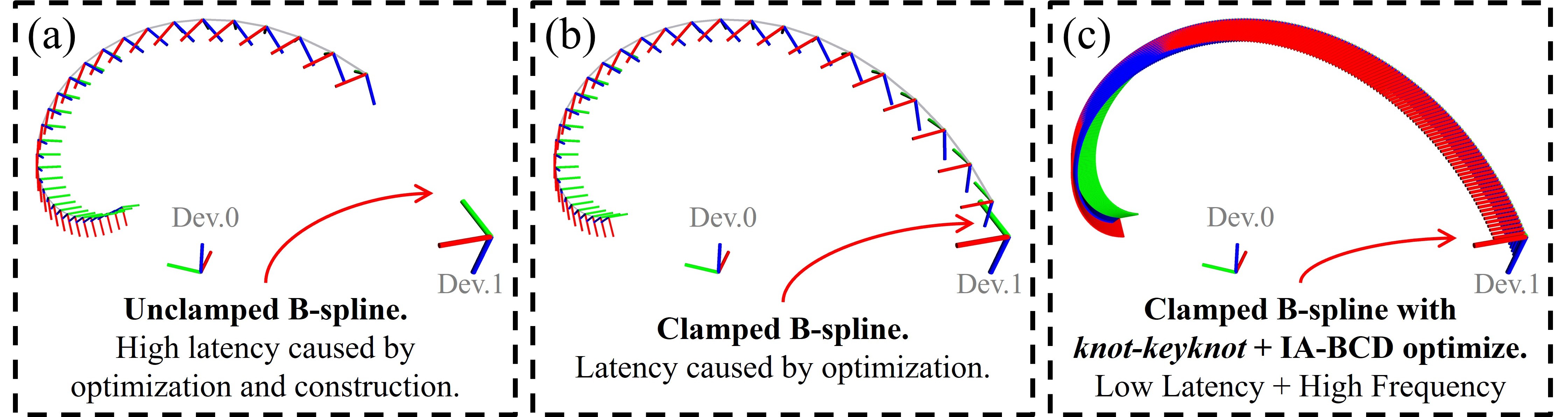}
    \caption{
        Visualization of output frequency and latency using (a) unclamped B-spline, (b) clamped B-spline, and (c) clamped B-spline with knot-keyknot strategy and IA-BCD optimization.
    }
    \label{Fig: Ablation}
\end{figure}

Our main contributions are summarized as follows:
\begin{enumerate}
    \item
          We adopt classical C-NUBS as a continuous-time state representation for online state estimation, eliminating the support-induced query delay of unclamped NUBS.
          We further derive closed-form extension and shrinkage operators that preserve the curve shape, and propose a \textit{knot-keyknot} strategy for high-rate spline extension while retaining sparse motion-informative keyknots.
    \item
          We formulate the general continuous-time reference-centric direct relative localization as a factor graph optimization problem under robocentric relative kinematics, incorporating inter-robot constraints and motion constraints, thereby enabling the joint estimation of 6-DoF relative poses and inter-robot time-offsets.
    \item
          We propose \CT{}, a C-NUBS-based framework that allocates modeling capacity via motion-intensity-aware keyknot selection adaptively, and solves robot-wise subproblems in parallel using IA-BCD, achieving real-time performance while retaining high accuracy.
    \item
          We validate \CT{} in extensive simulations and real-world experiments, demonstrating that it outperforms the evaluated published discrete-time methods.
          All code and datasets will be released as open source to facilitate reproducibility and support the research community \footnote{\url{https://github.com/anonymous/anonymous}}.
\end{enumerate}

\section{Related Works}
\label{Sec: Related Works}

\subsection{Direct Relative State Estimation}

Existing multi-sensor fusion approaches can be broadly categorized as loosely or tightly coupled.
Both paradigms model relative kinematics in the robocentric frame and fuse them with inertial measurements.
Loosely coupled approaches typically partition a multi-robot system into multiple pairwise subsystems \cite{armani2024accurately,aghili2016robust,xun2023crepes,zhao2025robust}, which limits their ability to exploit multi-robot correlations.
To leverage group-level couplings, tightly coupled relative state estimation has been widely studied.
\cite{stegagno2016ground} fused anonymous bearing measurements with inertial data via a particle filter to recover scale and relative poses.
\cite{shalaby2024multi} proposed passive UWB ranging with an on-manifold EKF that tightly couples UWB distance measurements and IMU preintegration to jointly estimate the relative state on $SE_2(3)$.
\cite{li2025crepes} proposed a tightly coupled least-squares optimization to fuse distance, bearing, and inertial measurements, and reported improved robustness in non-line-of-sight environments.
However, filtering-based methods often struggle to enforce global consistency, while discrete-state optimization typically requires an auxiliary interpolation to associate asynchronous measurements with states defined at discrete time instants.

Another challenge is the time-offset caused by independent clocks across robots.
Without dedicated synchronization hardware, such offsets can reach the order of $10^2$~ms \cite{chen2023understanding}, significantly degrading localization accuracy.
Online time-offset estimation has been studied in single-robot SLAM for cross-sensor calibration, typically using short-horizon constant-velocity assumptions to construct time-offset-related residuals \cite{huai2022observability,qin2018online}.
These ideas have also been extended to multi-robot relative localization.
\cite{shalaby2024multi} modeled the UWB clock offset as an explicit EKF state, while \cite{wen2024simultaneous} formulated time-offset-induced bearing projection errors in an indirect setting and solved them via SDP.
However, the underlying short-horizon motion assumptions can often be violated in multi-robot settings, where relative motions between robots can be highly dynamic.

These limitations motivate CT formulations that handle asynchronous measurements in a principled manner.

\subsection{Continuous-Time State Estimation}

CT representations include linear interpolation, wavelets, Gaussian processes, and splines.
In this paper, we focus on B-spline-based representation.
\cite{furgale2013unified} systematically proposed the B-spline-based CT SLAM problem and demonstrated its application in the self-calibration of the IMU-camera extrinsic parameters.
Thereafter, B-spline-based CT trajectory formulations have been widely used in sensor calibration \cite{rehder2016extending,sommer2020efficient,chen2025ikalibr} and SLAM-related applications \cite{mueggler2018continuous,lv2021clins,lang2022ctrl,lv2023continuous,lang2023coco,hug2024hyperion,lang2025gaussian}.
\cite{cioffi2022continuous} conducted a comprehensive comparison between CT and DT methods, highlighting that the knot distribution has a significant impact on estimation performance and computational efficiency.
\cite{lang2023coco} introduced a motion-adaptive scheme that outperforms uniform placement, which was later extended to a Gaussian Splatting SLAM system in \cite{lang2025gaussian}.
A similar approach was explored in \cite{sun2025ct} for UWB anchor-based odometry.
However, due to the recursive construction of B-splines, existing unclamped B-spline-based non-uniform methods inevitably suffer from a query-time delay, as discussed in \autoref{Sec: Clamped Non-Uniform B-Splines}.

Beyond motion representation, a key advantage of B-spline-based trajectory formulation is its natural support for online time-offset estimation.
Unlike DT methods, which rely on short-term constant-velocity assumptions, CT approaches exploit analytically differentiable trajectories to query pose and velocity at arbitrary timestamps and simplify the computation of associated Jacobians.
This allows time-offsets to be embedded directly into measurement models and jointly optimized with the trajectory without additional motion assumptions \cite{furgale2013unified,lv2023continuous,chen2025ikalibr}.
Empirical studies~\cite{cioffi2022continuous} further show that CT formulations achieve significantly higher time-offset estimation accuracy than DT methods.

To the best of our knowledge, no prior work has investigated CT direct relative localization.
\CT{} extends CT state estimation to multi-robot settings, preserving high-frequency sensor information while supporting online inter-robot time-offset estimation.
The closest prior work is \cite{sun2025ct}, which studies UWB anchor-based odometry with fixed infrastructure, fundamentally differing from fully mobile multi-robot relative localization.
Moreover, it adopts unclamped NUBS similar to \cite{lang2023coco}, which incur query-time latency.
In contrast, \CT{} builds upon C-NUBS to eliminate the delays and further addresses the spline shape variations introduced by C-NUBS during spline extension and shrinkage, as detailed in \autoref{Sec: Augmented C-NUBS for State Estimation}.

\vspace{-5pt}
\subsection{Acceleration Methods in Multi-Robot State Estimation}

Scalable state estimation has been extensively studied in both indirect and direct relative localization paradigms.
In indirect relative localization, distributed approaches include factor graph inference with Gaussian marginal exchange \cite{cunningham2013ddf}, ADMM-based SLAM \cite{xu2024d}, and Gaussian belief propagation \cite{murai2023robot,murai2024distributed}.
Recent work has also leveraged distributed pose-graph optimization, including SDP relaxation with Riemannian BCD for decentralized and certifiably correct estimation \cite{tian2020asynchronous}, as well as certifiably optimal collaborative localization from distance measurements \cite{thoms2025distributed}.
In direct relative localization, \cite{liu2020distributed} estimated neighbor poses locally and fused the resulting local topologies through SDP-based alignment, while \cite{wu2024scalable} formulated centralized relative state estimation as a generalized graph-realization problem and solved it using BCD.

While these approaches achieve scalability, they either rely on synchronized updates in optimization or cannot be directly extended to CT formulations.
\cite{sun2017asynchronous} established the theoretical foundation of Asynchronous BCD (A-BCD), analyzing its convergence under both bounded and unbounded delays, thereby eliminating the system latency and speed limitations induced by synchronized updates.
Motivated by this line of work, we adopt an Incremental A-BCD (IA-BCD) variant within \CT{} to enable asynchronous, parallel solution of robot-wise subproblems in a tightly coupled CT estimator.

\section{Augmented C-NUBS for State Estimation}
\label{Sec: Augmented C-NUBS for State Estimation}

This section details the trajectory representation underpinning \CT{}.
We first introduce the preliminaries of NUBS, then discuss the properties of clamping.
Based on these properties, we derive closed-form extension and shrinkage methods for C-NUBS that preserve the spline shape, augmenting it to support state estimation applications.
Finally, we propose a \textit{knot-keyknot} strategy for C-NUBS-based CT state estimation.
\subsection{Preliminaries on Non-Uniform B-Splines}
\label{Sec: Preliminary of Non-Uniform B-Splines}

In \CT{}, we employ cumulative NUBS for relative poses that separately parameterize the rotations $\mathbf{R}(t) \in \textit{SO}(3)$ and translations $\boldsymbol{p}(t) \in \mathbb{R}^3$.
The translation NUBS of order $k$ (degree $k-1$) is defined by knot vector $\mathbf{T}\!=\!\{ t_{[0]}, t_{[1]}, \ldots, t_{[n+k]} \}$ and control points $\mathbf{x_{p}} = \{\mathbf{p}_0, \mathbf{p}_1, \ldots, \mathbf{p}_n\}$.
Due to the local support property of NUBS, whereby only $k$ control points influence the curve value at any time $t$, the curve admits a compact matrix representation \cite{qin1998general,sommer2020efficient}.
The value of $\mathbf{p}(t)$ at time $t \in [t_{[i]}, t_{[i+1]})$ can be expressed as:
\vspace{-4pt}
\begin{equation}
    \begin{aligned}
        \mathbf{p}(t)=
        \left[ \mathbf{p}_{i-k+1}, \mathbf{d}_{i-k+1}^{(1)}, \cdots, \mathbf{d}_{i-k+1}^{(k-1)} \right]
        \widetilde{\mathbf{M}}_{i}^{(k)} \mathbf{u},
        \\
        \mathbf{u}=\begingroup
        \setlength{\arraycolsep}{3pt}
        \begin{bmatrix}
            1 & u & \cdots & u^{k-1}
        \end{bmatrix}
        \endgroup, u=\left(t-t_{[i]}\right)/\left(t_{[i+1]}-t_{[i]}\right),
    \end{aligned}
    \label{Equ: R3 Spline}
    \vspace{-4pt}
\end{equation}
with difference vectors $\mathbf{d}_{i}^{(j)} = \mathbf{p}_{i+j} - \mathbf{p}_{i+j-1}$.
The cumulative blend matrix $\widetilde{\mathbf{M}}^{(k)}_{i}$ is only related to order $k$ and $2(k-1)$ knots $\{t_{[i-k+2]}, \ldots, t_{[i+k-1]}\}$.

Cumulative NUBS can be generalized to Lie groups and, in particular, to $\textit{SO}(3)$ for smooth rotation generation \cite{sommer2020efficient,kim1995general}.
Similar to (\ref{Equ: R3 Spline}), we can represent the cumulative rotation NUBS with control points $\mathbf{x_{R}} = \{\mathbf{R}_0, \mathbf{R}_1, \ldots, \mathbf{R}_n\}$ of order $k$ over time $t \in [t_{[i]}, t_{[i+1]})$ as:
\vspace{-8pt}
\begin{equation}
    \mathbf{R}(t)=\mathbf{R}_{i-k+1} \cdot \prod_{j=1}^{k-1} \Exp \left(\widetilde{\mathbf{M}}_{i}^{(k)} \mathbf{u} \cdot \mathbf{d}_{i-k+1}^{(j)} \right),
    \label{Equ: SO3 Spline}
    \vspace{-5pt}
\end{equation}
where $\mathbf{d}_{i}^{(j)} = \Log \left(\mathbf{R}_{i+j-1}^{-1} \mathbf{R}_{i+j}\right)$, and $\mathbf{R}_{i} \in \textit{SO}(3)$ is the $i$-th rotation control point.
$\Exp(\cdot)$ and $\Log(\cdot)$ are the exponential and logarithmic maps of $\textit{SO}(3)$.

As can be seen from (\ref{Equ: R3 Spline}) and (\ref{Equ: SO3 Spline}), the valid query range of an NUBS is smaller than the bounds of its knot vector.
Specifically, only the intervals where the blend matrix can be constructed are valid, with each blend matrix corresponding to one spline segment.
Thus, if $t_{[n+k]}$ is the most recent knot, the spline is only defined up to $t_{[n+1]}$, introducing an inherent delay of $(t_{[n+k]} - t_{[n+1]})$.

\textit{Extended Query Interval Property}:
\label{eqip}
Although a B-spline is conventionally evaluated over $[t_{[k-1]},t_{[n+1]})$, its basis-function recursion admits one additional evaluable interval at each boundary when the adjacent control points $\mathbf p_{-1}$ and $\mathbf p_{n+1}$ are available.
Over $[t_{[k-2]},t_{[k-1]})$ and $[t_{[n+1]},t_{[n+2]})$, the recursion becomes one-sided, and the terms requiring knots outside the existing knot vector vanish due to zero support.
Hence, using the existing knot vector $\mathbf T=\{t_{[0]},\ldots,t_{[n+k]}\}$, the evaluable query interval extends to $[t_{[k-2]},t_{[n+2]})$ with additional control points $p_{-1}, p_{n+1}$.

\subsection{Clamped Non-Uniform B-Splines}
\label{Sec: Clamped Non-Uniform B-Splines}

C-NUBS are a classical variant of NUBS in which the first and last knots are repeated $k$ times, forcing the curve to interpolate its endpoints\cite{piegl2012nurbs}:
\vspace{-5pt}
\begin{align}
    \mathbf{T}_{\text{clamp}}   & :  t_{[0]} = \cdots = t_{[k-1]} \leq \cdots \leq t_{[n+1]} = \cdots = t_{[n+k]}, \notag \\
    \mathbf{T}_{\text{unclamp}} & :   t_{[0]} \leq t_{[1]} \leq \cdots \leq t_{[n+k-1]} \leq t_{[n+k]}.
    \vspace{-5pt}
\end{align}
This formulation provides the same approximation capability as the unclamped version while yielding improved boundary behavior and has been widely used in computer graphics \cite{piegl2012nurbs} and path planning \cite{christensen2025non}.
In real-time state estimation, clamping can eliminate the inherent evaluation delays at the curve endpoints, which is rarely discussed in prior works \cite{lang2023coco,lang2025gaussian,sun2025ct} but is crucial for online applications.

For uniform B-splines, subsequent knots and control points can be generated in advance.
This eliminates the inherent delay and extends the valid range of the B-spline to the latest sensor time, enabling real-time optimization.
As shown in \autoref{Fig: basismatrix}(a), evaluating the spline at time $6\,\mathrm{s}$ requires the knots and control points at time $7\,\mathrm{s}$ and $8\,\mathrm{s}$ to be available in advance.

However, for NUBS, since knot intervals are uncertain and typically determined by the latest sensor measurements \cite{lang2023coco,lang2025gaussian}, knots and control points beyond the latest sensor time cannot be determined from the currently available measurements.
As a result, estimation inevitably suffers from the inherent delay.
As shown in \autoref{Fig: basismatrix}(b), although exploiting the \textit{Extended Query Interval Property}, the spline can only be evaluated up to time $4\,\mathrm{s}$, resulting in a delay of $2\,\mathrm{s}$.

For C-NUBS, the repeated final knot allows immediate evaluation up to the latest time, directly eliminating the inherent delay.
This property enables adaptive insertion of non-uniform knots and control points in response to incoming sensor data, while still providing state estimates at the most recent knot time.
As shown in \autoref{Fig: basismatrix}(c), the latest sensor time, the latest knot time, and the evaluable time all coincide at $6\,\mathrm{s}$.

Moreover, C-NUBS also ensures that the curve passes through the final control point, which is useful for initializing new segments.
However, adding or deleting the last control point changes the shape of the last $(k-1)$ segments, which is undesirable for consistent state estimation.

\subsection{Closed-Form Extension of C-NUBS}
\label{Sec: Closed-Form Extension of C-NUBS}

A key issue with C-NUBS is that adding new control points alters the existing spline segment, thereby shifting the solution obtained from the previous optimization.
An iterative unclamping algorithm was proposed in \cite{piegl2012nurbs} and later applied to spline extension in \cite{hu2002extension}.
In \cite{hu2002extension}, the extended knot is determined by chord-length estimation to maintain low curvature in the newly added segment under varying distances to the target control point.
In state estimation, however, the objective is to accurately parameterize the state at a specified timestamp, which requires explicitly specifying both the target knot and control point.
We achieve this by adjusting only the last $k-2$ control points, a process we term closed-form extension.

\begin{proposition}[Closed-form extension of a C-NUBS]
    \label{Pro: Closed-Form Extension}
    Consider a C-NUBS curve of order $k$ defined by control points $\mathbf{x}_{p}$ and a clamped knot vector $\mathbf{T}$:
    \begin{equation}
        \begin{aligned}
            \mathbf{x_p} & = \left\{ \mathbf{p}_{0}, \mathbf{p}_{1}, \ldots, \mathbf{p}_{n-k-1}, \mathbf{p}_{n-k}, \ldots, \mathbf{p}_{n} \right\},
            \\
            \mathbf{T}   & : t_{[0]} = \cdots = t_{[k-1]} \leq \cdots \leq t_{[n+1]} = \cdots = t_{[n+k]}.
        \end{aligned}
        \label{Equ: Clamped B-spline Definition}
    \end{equation}

    An extension of this C-NUBS to a prescribed new knot $t_{+}>t_{[n+1]}$ with a prescribed new control point $\tilde{\mathbf p}_{n+1}$ is obtained through the following transformation.
    \begin{enumerate}
        \item  Replace the last $k\!-\!1$ knots by $t_{+}$, yielding $\mathbf T'=\{t'_{[i]}\}_{i=0}^{n+k}$ with
              \vspace{-6pt}
              \[
                  t'_{[i]}
                  =
                  \begin{cases}
                      t_{[i]}, & 0\leq i\leq n+1,   \\
                      t_{+},   & n+2\leq i\leq n+k.
                  \end{cases}
              \]
        \item Update the last $k-2$ control points using the closed-form
              unclamping recursion.
              Let $\tilde{\mathbf p}_{i}^{(0)}=\mathbf p_i$ for
              $i=n-k+2,\ldots,n$.
              For $r=1,\ldots,k-2$, the affected control points are recursively
              updated as
              \vspace{-3pt}
              \begin{equation}
                  \tilde{\mathbf p}_{i}^{r}
                  \!=\!
                  \frac{
                      \tilde{\mathbf p}_{i}^{(r\!-\!1)}\!-\!(1\!-\!\alpha_{i}^{r})\tilde{\mathbf p}_{i\!-\!1}^{r}
                  }{
                      \alpha_{i}^{r}
                  },
                  \alpha_{i}^{r}
                  \!=\!
                  \frac{t'_{[n+1]}\!-\!t'_{[i]}}
                  {t'_{[i+r+1]}\!-\!t'_{[i]}},
                  \label{Equ: extension update}
                  \vspace{-3pt}
              \end{equation}
              for $i=n-r+1,\ldots,n$, while $\tilde{\mathbf p}_i^{r}=\tilde{\mathbf p}_i^{r-1}$ for the remaining retained control points.
              The modified control-point set is
              \vspace{-5pt}
              \[
                  \mathbf x'_p
                  =
                  \{
                  \mathbf p_0,\ldots,\mathbf p_{n-k+2},
                  \tilde{\mathbf p}_{n-k+3}^{k-2},\ldots,
                  \tilde{\mathbf p}_{n}^{k-2}
                  \}.
              \]
        \item Append the new control point $\tilde{\mathbf p}_{n+1}$ to $\mathbf x'_p$ and the new knot $t_{+}$ to $\mathbf T'$:
              \[
                  \mathbf x''_p
                  =
                  \{
                  \mathbf p_0,\ldots,\mathbf p_{n-k+2},
                  \tilde{\mathbf p}_{n-k+3}^{k-2},\ldots,
                  \tilde{\mathbf p}_{n}^{k-2},
                  \tilde{\mathbf p}_{n+1}
                  \},
              \]
              \[
                  \mathbf T''
                  =
                  \{
                  t'_{[0]},\ldots,t'_{[n+k]},t_{+}
                  \}.
              \]
    \end{enumerate}

    The resulting C-NUBS $(\mathbf{x}_{p}^{\prime\prime}, \mathbf{T}^{\prime\prime})$ is identical to the original curve on $[t_{[0]}, t_{[n+1]}]$ and is extended to specified $t_{+}$ with $\tilde{\mathbf{p}}_{n+\!1}$.
\end{proposition}

\begin{proof}
    As shown in \cite{piegl2012nurbs}, the last \(k{-}1\) knots of a C-NUBS can be reassigned arbitrarily if the last \(k{-}2\) control points are iteratively updated to preserve the curve shape.
    Following this principle, the last \(k{-}1\) knots are first replaced with the new knot \(t_{+}\), temporarily converting the clamped spline into an unclamped one.
    The last \(k{-}2\) control points are then updated using the unclamping procedure to maintain shape invariance.
    After this operation, the curve over the existing domain \([t_{[0]}, t_{[n+1]})\) remains unchanged.
    Finally, by the local support property of B-splines, appending the new control point \(\tilde{\mathbf{p}}_{n+1}\) and knot \(t_{+}\) affects only the boundary segment, ensuring that the extended curve is identical to the original on \([t_{[0]}, t_{[n+1]}]\) and smoothly extends to \(t_{+}\).
\end{proof}

\begin{example}
    The extension procedure of a cubic C-NUBS (\(k=4\)) from
    $
        (\mathbf{x_p}, \mathbf{T}\!=\!\{0,0,0,0,1,2,3,3,3,3\}),
    $
    to
    $
        (\mathbf{x_p}^{\prime\prime}, \mathbf{T}^{\prime\prime}\!=\!\{0,0,0,0,1,2,3,4,4,4,4\})
    $
    is illustrated in \autoref{Fig: Closed-Form Extension}.
\end{example}

\begin{figure}[ht]
    
    \centering
    \includegraphics[width=1\linewidth]{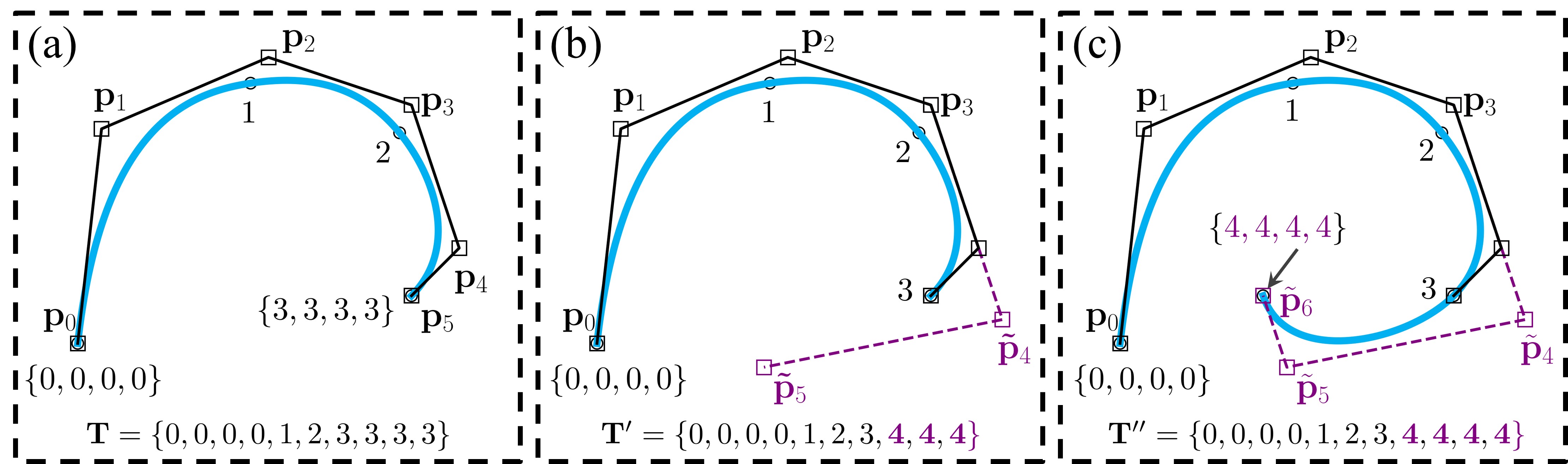}
    \caption{
        Closed-form extension of C-NUBS ($k=4$).
        The purple highlights the updated portion.
        $(k-2)$ control points are adjusted to keep the shape.
        (a) Original spline.
        (b) Step 1 and 2 of \autoref{Pro: Closed-Form Extension}.
        (c) Step 3 of \autoref{Pro: Closed-Form Extension}.
    }
    \label{Fig: Closed-Form Extension}
\end{figure}

\subsection{Closed-Form Shrinkage of C-NUBS}
\label{Sec: Closed-Form Shrinkage of C-NUBS}

Analogous to extension, the shrinkage of a C-NUBS will change its shape.
We exploit the knot insertion algorithm \cite{piegl2012nurbs} and the local support property of B-splines to derive a closed-form shrinkage algorithm.

\begin{proposition}[Closed-form shrinkage of a C-NUBS]
    \label{Pro: Closed-Form Shrinkage}
    Consider the C-NUBS defined in (\ref{Equ: Clamped B-spline Definition}), and let $u=t_{[n]}$ be the left endpoint of the terminal interval $[t_{[n]},t_{[n+1]})$.
    Removing this interval yields the following shrinkage transformation.
    \begin{enumerate}
        \item Insert the $u$ into $\mathbf{T}$ $k\!-\!1$ times.
              Since $u$ is a simple interior knot in our construction, only the last $k-2$ retained control points need to be updated.
              Let $\tilde{\mathbf p}_{i}^{0}=\mathbf p_i$ for $i=n-k+1,\ldots,n-1$.
              For $r=1,\ldots,k-2$, the affected control points are recursively updated as
              \begin{equation}
                  \tilde{\mathbf p}_{i}^{r}
                  \!=\!
                  \alpha^{r}_{i}\tilde{\mathbf p}_{i}^{r-1}
                  \!+\!
                  (1\!-\alpha_{i}^{r})\tilde{\mathbf p}_{i-1}^{r-1},\,
                  \alpha^{r}_{i}
                  \!=\!
                  \frac{u\!-\!t_{[i]}}
                  {t_{[i+k-r]}\!-\!t_{[i]}},
                  \label{Equ: cnubs_shrink_update}
              \end{equation}
              for $i=n-k+r+1,\ldots,n-1$, while $\tilde{\mathbf p}_i^{r}=\tilde{\mathbf p}_i^{r-1}$ for the remaining retained control points.
              The final insertion only increases the multiplicity of $u$ to $k$ and does not modify the retained control points.

        \item Remove the last $k$ knots and control points, yielding
              \[
                  \mathbf x_p''
                  =
                  \{
                  \mathbf p_0,\ldots,\mathbf p_{n-k+1},
                  \tilde{\mathbf p}_{n-k+2}^{k-2},\ldots,
                  \tilde{\mathbf p}_{n-1}^{k-2}
                  \},
              \]
              \[
                  \mathbf T''
                  =
                  \{
                  t_{[0]},\ldots,t_{[n-1]},
                  \underbrace{u,\ldots,u}_{k}
                  \}.
              \]
    \end{enumerate}

    The resulting C-NUBS $(\mathbf x_p'',\mathbf T'')$ coincides with the original curve on $[t_{[0]},u]$ and shrinks its endpoint from $t_{[n+1]}$ to $u$.
\end{proposition}

\begin{proof}
    Each standard knot insertion preserves the represented B-spline curve exactly~\cite{piegl2012nurbs}.
    Since $u$ is initially a simple interior knot, after $k-1$ insertions its multiplicity becomes $k$, forming the clamped endpoint of the shrunk spline.
    By the local support property, removing the terminal $k$ knots and control points discards only the interval $[u,t_{[n+1]})$ while leaving the curve on $[t_{[0]},u]$ unchanged.
\end{proof}

\begin{example}
    The shrinkage procedure of a cubic C-NUBS (\(k=4\)) from
    $
        \mathbf{T} = \{0,0,0,0,1,2,3,3,3,3\}
    $
    to
    $
        \mathbf{T} = \{0,0,0,0,1,2,2,2,2\}
    $
    is illustrated in \autoref{Fig: Closed-Form Shrinkage}.
\end{example}
\vspace{-7pt}
\begin{figure}[ht]
    
    \centering
    \includegraphics[width=1\linewidth]{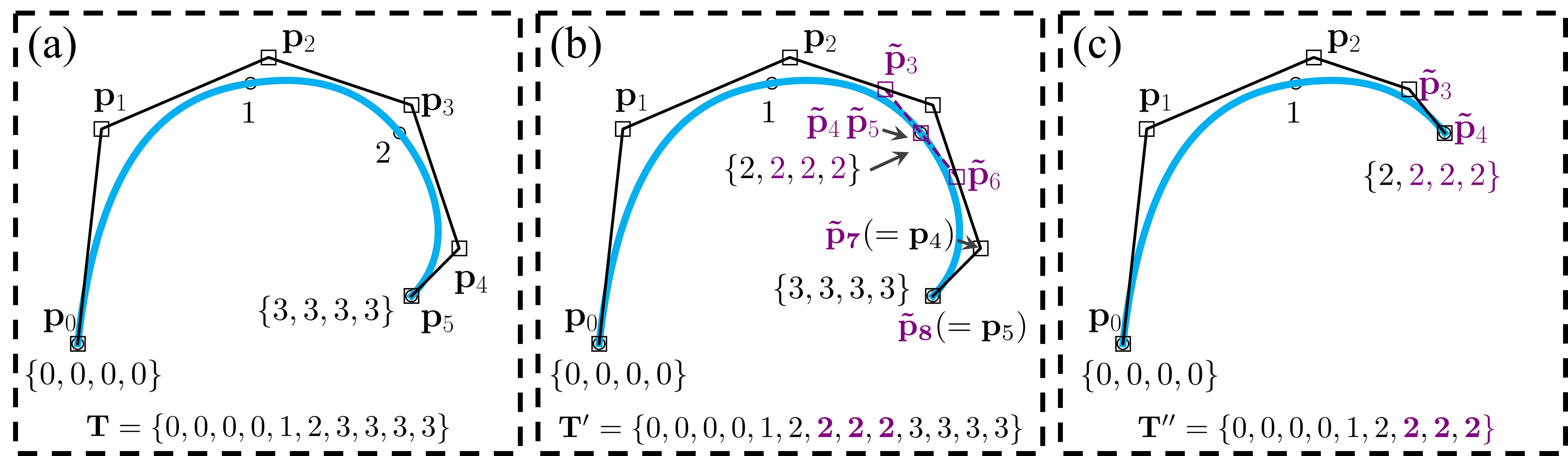}
    \caption{
        Closed-form shrinkage of C-NUBS ($k=4$).
        The purple highlights the updated portion.
        $(k-2)$ control points are adjusted to keep the shape.
        (a) Original spline.
        (b) Step 1 of \autoref{Pro: Closed-Form Shrinkage}.
        (c) Step 2 of \autoref{Pro: Closed-Form Shrinkage}.
    }
    \label{Fig: Closed-Form Shrinkage}
    \vspace{-5pt}
\end{figure}

\subsection{Adaptation to Lie Groups}
\label{Sec: Generalization to Lie Groups}

The proposed extension and shrinkage transformations naturally adapt to Lie groups.
For $\mathrm{SO}(3)$, Euclidean interpolation is replaced by group composition with scaling in the corresponding Lie algebra.

The extension update in \eqref{Equ: extension update} becomes
\begin{equation}
    \tilde{\mathbf{R}}_i^r=\tilde{\mathbf{R}}_{i-1}^r \cdot \Exp \left(\frac{1}{\alpha_{i}^{r}} \cdot \Log \left(\left(\tilde{\mathbf{R}}_{i-1}^r\right)^{\top} \tilde{\mathbf{R}}_i^{r-1}\right)\right),
    \label{Equ: SO3 extension}
\end{equation}

Similarly, the shrinkage update in \eqref{Equ: cnubs_shrink_update} becomes
\begin{equation}
    \tilde{\mathbf R}_{i}^{r}
    =
    \tilde{\mathbf R}_{i-1}^{r-1}
    \operatorname{Exp}
    \left(
    \alpha_{i}^{r}
    \operatorname{Log}
    \left(
        \left(\tilde{\mathbf R}_{i-1}^{r-1}\right)^{\top}
        \tilde{\mathbf R}_{i}^{r-1}
        \right)
    \right).
    \label{eq:so3_shrinkage}
\end{equation}

Demonstrations of the closed-form extension and shrinkage of an $\textit{SO}(3)$ C-NUBS curve are shown in \autoref{Fig: SO3 Extension Shrinkage}.
\vspace{-7pt}
\begin{figure}[htb]
    
    \centering
    \includegraphics[width=\linewidth]{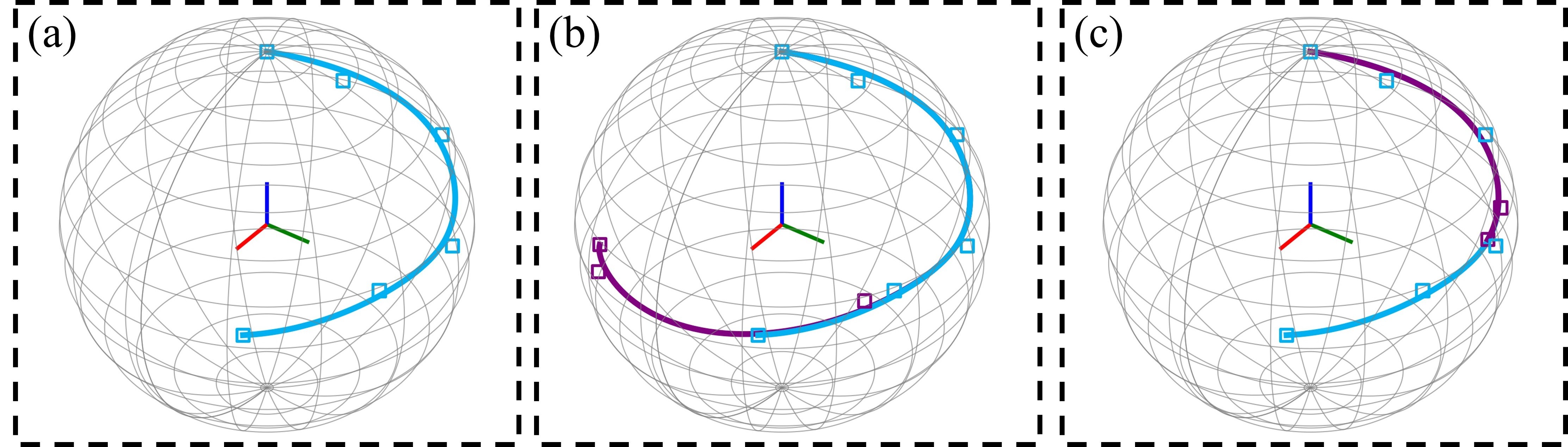}
    \caption{
        Demonstrations of the closed-form extension and shrinkage of an $\textit{SO}(3)$ C-NUBS curve.
        The purple highlights the updated portion.
        (a) Original spline.
        (b) Extended spline.
        (c) Shrunk spline.
        $(k\!-\!2)$ control points are adjusted to keep the shape.
    }
    \label{Fig: SO3 Extension Shrinkage}
\end{figure}

\vspace{-15pt}
\subsection{Knot-Keyknot Strategy for C-NUBS}
\label{Sec: Knot-Keyknot Strategy for C-NUBS}

Most B-spline-based CT state estimation methods are limited by uniformly spaced knots and fixed optimization intervals.
In contrast, discrete-time state estimation approaches, such as VINS-Mono~\cite{qin2018vins}, adopt a frame-keyframe strategy: performing per-frame optimizations for camera-rate outputs while preserving only informative keyframes in the sliding window to ensure estimation efficiency.


However, this strategy is difficult to realize with uniform B-splines.
Because the knot frequency is typically lower than the measurement rate, high-frequency outputs require extrapolation beyond the latest measurement time $t_{\text{last}}$, leaving an unconstrained interval before the next knot $t_{\text{knot}}$.
During this interval, the spline relies only on its inherent smoothness, degrading observability and potentially causing numerical instability (see \autoref{Fig: knot-keyknot}).
Non-uniform B-splines alleviate this issue by inserting knots at measurement times, but existing unclamped formulations introduce intrinsic query-time delay, as discussed in \autoref{Sec: Clamped Non-Uniform B-Splines}.

\begin{figure}[ht]
    
    \centering
    \includegraphics[width=\halfwidth]{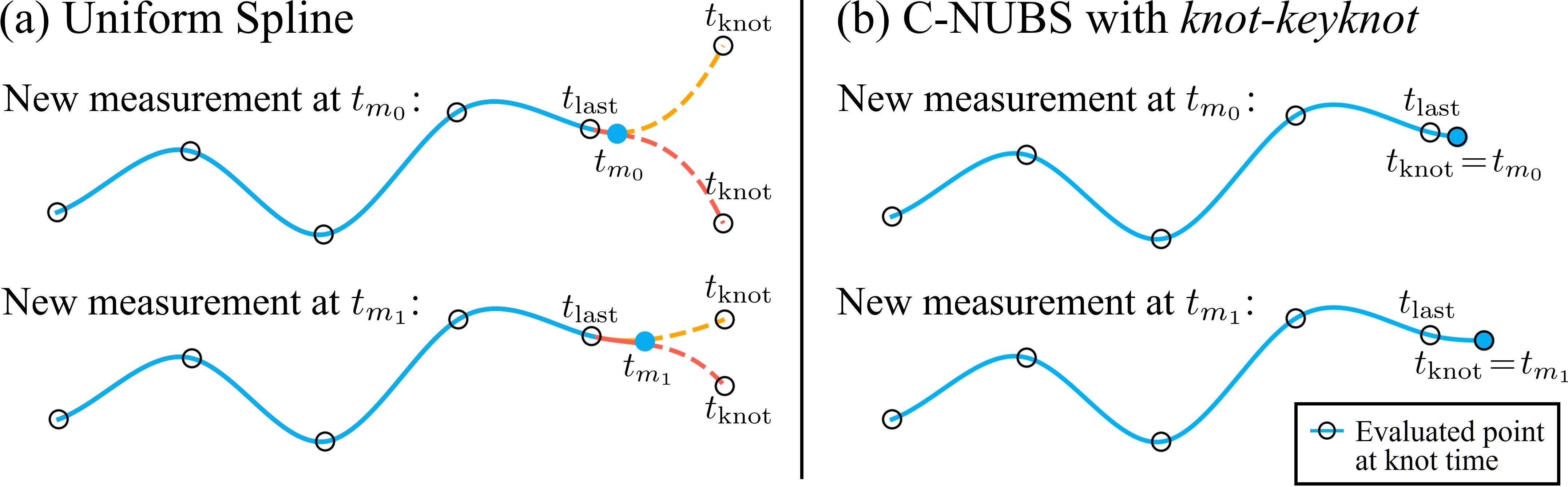}
    \caption{
        Tail extrapolation in uniform B-splines and its avoidance using C-NUBS with the knot-keyknot strategy.
        (a) Uniform knots force tail extrapolation, creating an unconstrained interval that degrades observability and stability.
        (b) C-NUBS with the \textit{knot-keyknot} strategy inserts a knot at each measurement time, directly constraining the spline tail without extrapolation.
    }
    \label{Fig: knot-keyknot}
\end{figure}

C-NUBS resolves this limitation by enabling real-time evaluation up to the latest knot with endpoint interpolation.
The proposed extension and shrinkage operations further enable flexible terminal-knot manipulation.
Leveraging these properties, we introduce a \textit{knot-keyknot} strategy for CT state estimation that bridges spline-based and keyframe-based paradigms.

At each incoming measurement, we extend the B-spline by appending a knot and control point at the measurement timestamp.
A sliding-window optimization is then performed over the spline segment span defined by this new knot and existing keyknots.
After optimization, if the new knot is not selected as a keyknot, the spline is shrunk back to the last keyknot using closed-form shrinkage, then waits for the next extension.
Otherwise, it is retained as keyknot.
This process follows an ``extend-shift-retain'' cycle, as shown in \autoref{Fig: knot-keyknot}(b).

\section{Continuous-Time Relative State Estimation Problem Formulation}
\label{Sec: B-Spline-Based Continuous-Time Multi-Robot Relative State Estimation}

In this section, we formulate a general B-spline-based continuous-time factor graph optimization problem for direct relative localization.
This section is not restricted to a specific B-spline type and supports uniform/non-uniform as well as clamped/unclamped B-splines.

Consider a group of $n$ robots indexed by $\mathcal{V}\!=\!\{0,1,\dots,n\!-\!1\}$, where each robot is equipped with an ego-motion sensor (one IMU in this work) and inter-robot sensing (one UWB and one camera in this work).
Robots convert raw sensor data into timestamped observation messages and broadcast them to the team.
A designated reference robot $s$ runs the estimator, which jointly estimates inter-robot time-offsets and relative poses from asynchronous measurement streams.

For each robot $j\in\mathcal{V}$, we represent the relative pose with a translation spline ${}^{s}\mathbf{p}_j(t)\in\mathbb{R}^3$ and a rotation spline ${}^{s}\mathbf{R}_j(t)\in\mathrm{SO}(3)$.
For consistency, the reference robot trajectory is also spline-parameterized, while its pose is anchored by enforcing ${}^{s}\mathbf{p}_s(t)=\mathbf{0}$ and ${}^{s}\mathbf{R}_s(t)=\mathbf{I}$.
Within the $\kappa$-th sliding window $[t_\kappa-\Delta,t_\kappa]$, where $\Delta$ denotes the window length, the relative trajectory of robot $j$ is parameterized by active spline control points ${}^{s}\bm{\Phi}^{\kappa}_{j}$.
The estimator optimizes the following state:
\vspace{-3pt}
\begin{equation}
    \begin{aligned}
        \mathcal{X}^{\kappa}       & = \left\{ \left\{ {\mathcal{X}^{\kappa}_{0}}, {\mathcal{X}^{\kappa}_{1}}, \cdots, {\mathcal{X}^{\kappa}_{n-1}} \right\}, \bm{\Psi}^{\kappa} \right\},
        \\[-2pt]
        {\mathcal{X}^{\kappa}_{j}} & = \left\{ {^{s}\bm{\Phi}^{\kappa}_{j}}, {^{s}{\tau}^{\kappa}_j}, {^j}\mathbf{b}^{\kappa}_{\mathbf{a}}, {^j}\mathbf{b}^{\kappa}_{\bm{\omega}} \right\}, j \in \mathcal{V}.
    \end{aligned}
\end{equation}
Here, ${}^{s}\tau^{\kappa}_{j}$ denotes the clock time-offset between robot $s$ and robot $j$.
The vectors ${}^{j}\mathbf{b}^{\kappa}_{\mathbf{a}}$ and ${}^{j}\mathbf{b}^{\kappa}_{\bm{\omega}}$ are the IMU biases of robot $j$.
The auxiliary variable set $\bm{\Psi}^{\kappa}$ collects the control points of an auxiliary IMU spline of the reference robot $s$.
The roles of these variables will be described in the following subsections.

For simplicity, sensor extrinsics are assumed known from offline calibration, following \cite{li2025crepes}, where extrinsics were included in the state but not estimated.

\vspace{-5pt}
\subsection{Measurement Time-Offset}
\label{Sec: Measurement Time-Offset}

In \CT{}, all trajectories are parameterized with respect to the clock of the reference robot $s$.
However, each robot’s measurements are stamped using its own local clock before broadcast.
Let ${^{s}t}$ denote the time of the reference clock of robot $s$, and ${^{j}t}$ denote the time of the local clock of robot $j$.
The time-offset between $s$ and $j$ is defined as ${^{s}{\varGamma}_{j}} = {^{s}t} - {^{j}t}$.

As illustrated in \autoref{Fig: Time Synchronization}, using the measurement's timestamp ${^{j}t_m}$ instead of its true time ${^{s}t_m}$ to query on the B-spline and calculate residual can cause error.
To compensate for this, we introduce ${^{s}{\tau}_j}$ as an optimization variable for robot $j$.
Each measurement's residual is then calculated by querying the B-spline at the compensated time ${^{s}{t}_m} = {^{s}{\tau}_j} + {^{j}t_m}$.
By exploiting the analytical gradient of B-spline evaluation with respect to ${^{s}{t}_m}$, the offset ${^{s}{\tau}_j}$ can be jointly optimized.

For simplicity, in the following text, we use $t_c$ to denote the corrected timestamp, which is obtained by compensating the measurement ${\widehat{\mathbf{z}_*}}$'s original timestamp $t_m$ from robot $j$ using the estimated time-offset $^s\tau_j$.

\begin{figure}[t]
    \centering
    \includegraphics[width=.95\linewidth]{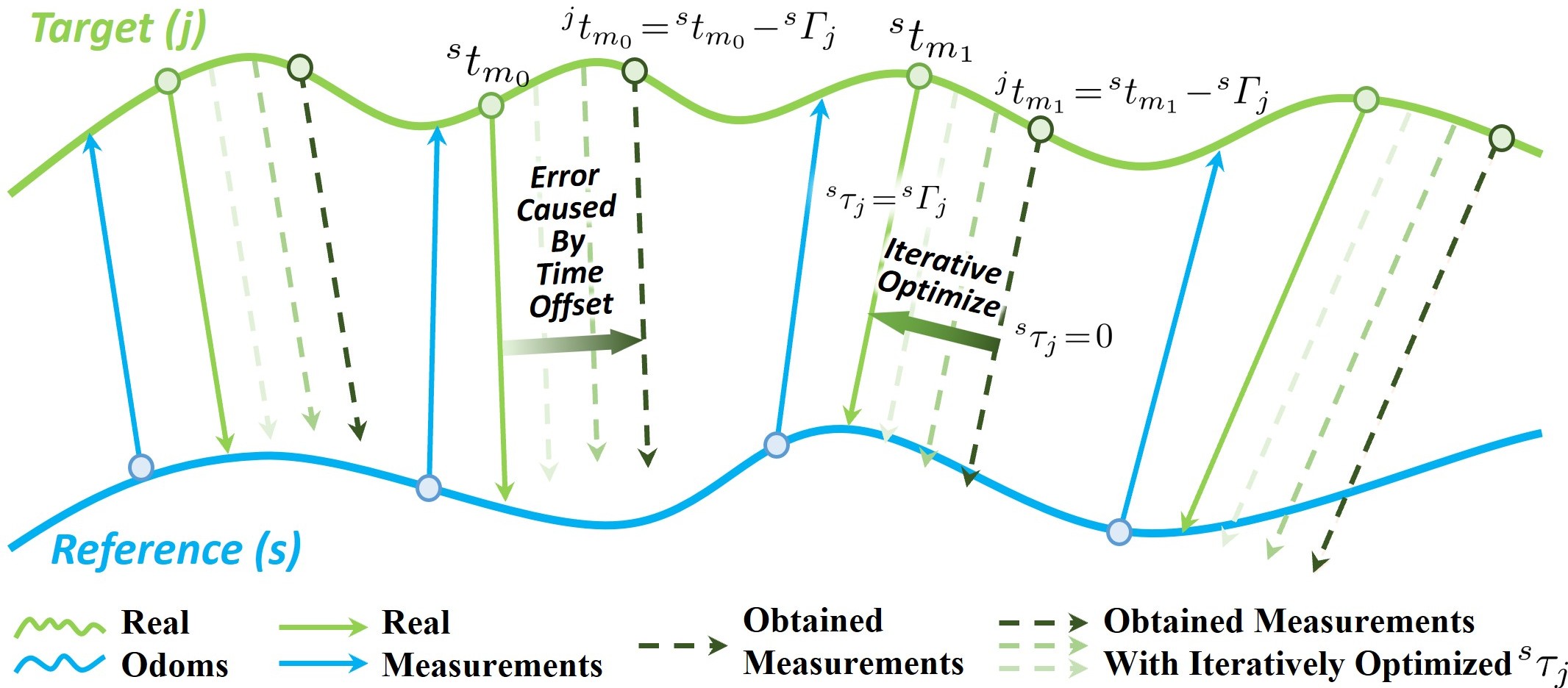}
    \caption{
        Illustration of time-offset between the local clock of the target robot and the reference clock, and the correction process of measurement time.
    }
    \label{Fig: Time Synchronization}
\end{figure}

\subsection{Inter-Robot Constraints}
\label{Sec: Inter-Robot Constraints}

The relative poses can be queried from the corresponding spline at timestamp $t_c$ as:
\begin{equation}
    \begin{aligned}
        ({^{s}\mathbf{R}^{t_c}_{j}}, {^{s}\mathbf{p}^{t_c}_{j}}) & = {^{s}\bm{\Phi}^{\kappa}_{j}}\big|_{t = t_c}.
    \end{aligned}
\end{equation}

Then, the inter-robot constraints are defined as follows.

\subsubsection{Bearing Factor}
\label{SubSec: Bearing Factor}

The camera detection model is as:
\begin{equation}
    \widehat{\mathbf{z}_b}_{\overrightarrow{jk}}^{t} = ({^{s}\mathbf{R}_{j}^{t}})^{\top} \frac{{^{s}\mathbf{p}_{k}^{t}}-{^{s}\mathbf{p}_{j}^{t}}}{\left\|{^{s}\mathbf{p}_{k}^{t}}-{^{s}\mathbf{p}_{j}^{t}}\right\|_2} + \mathbf{n}_{b},
    \label{Equ: Bearing Measurements}
\end{equation}
where $\mathbf{n}_{b} \sim \mathcal{N}(0, \Sigma_{b})$ is the noise of the bearing vector, $j$ and $k$ are indices of any two robots.
Note that the bearing measurement is a unit vector, which is the normalized vector of the relative position.
The residuals can be calculated as:
\begin{equation}
    \mathbf{r}_{b} ({\widehat{\mathbf{z}_b}_{\overrightarrow{jk}}^{t_m}}, {^{s}{\tau}^{\kappa}_{j}}, {^{s}\bm{\Phi}^{\kappa}_{j}}, {^{s}\bm{\Phi}^{\kappa}_{k}}) =
    {(^{s}\mathbf{R}^{t_c}_{j})}^{\top} \frac{{^{s}\mathbf{p}^{t_c}_{k}} - {^{s}\mathbf{p}^{t_c}_{j}}} {\left\|{^{s}\mathbf{p}^{t_c}_{k}}-{^{s}\mathbf{p}^{t_c}_{j}}\right\|_2} - {\widehat{\mathbf{z}_b}_{\overrightarrow{jk}}^{t_m}}.
    \label{Equ: Bearing Residuals}
\end{equation}

\subsubsection{Distance Factor}
\label{SubSec: Distance Factor}

The UWB measurement model is as:
\begin{equation}
    \widehat{\mathbf{z}_d}_{\overrightarrow{jk}}^{t} = \left\|{^{s}\mathbf{p}_{k}^{t}}-{^{s}\mathbf{p}_{j}^{t}}\right\|_2 + \mathbf{n}_{d},
    \label{Equ: Distance Measurements}
\end{equation}
where $\mathbf{n}_{d} \sim \mathcal{N}(0, \sigma_{d}^2)$ is the noise of the UWB, $j$ and $k$ are indices of any two robots.
The residuals can be calculated as:
\begin{equation}
    \mathbf{r}_{d} (\widehat{\mathbf{z}_d}_{\overrightarrow{jk}}^{t_m}, {^{s}{\tau}^{\kappa}_{j}}, {^{s}\bm{\Phi}^{\kappa}_{j}}, {^{s}\bm{\Phi}^{\kappa}_{k}})
    = \left\|{^{s}\mathbf{p}_{k}^{t_c}}-{^{s}\mathbf{p}_{j}^{t_c}}\right\|_2-\widehat{\mathbf{z}_d}_{\overrightarrow{jk}}^{t_m}.
    \label{Equ: Distance Residuals}
\end{equation}

\subsection{Motion Constraints}
\label{Sec: Motion Constraints}

The relative velocities and accelerations can be queried from the corresponding spline at timestamp $t_c$ as:
\vspace{-4pt}
\begin{equation}
    \begin{aligned}
        ({^{j}\bm{\omega}^{t_c}_{\overrightarrow{s j}}}, {^{s}\mathbf{v}^{t_c}_{\overrightarrow{s j}}}) & = {^{s}\bm{\dot{\Phi}}^{\kappa}_{j}}\big|_{t = t_c},
        \\
        ({^{j}\bm{\beta}^{t_c}_{\overrightarrow{s j}}}, {^{s}\mathbf{a}^{t_c}_{\overrightarrow{s j}}})  & = {^{s}\bm{\ddot{\Phi}}^{\kappa}_{j}}\big|_{t = t_c}.
    \end{aligned}
    \vspace{-3pt}
\end{equation}
It is worth noting that, for $\textit{SO}(3)$ B-spline, the derivative is defined in the body frame.
The relative motion constraints are then defined as follows.

\subsubsection{Relative Kinematics}
\label{SubSec: Relative Kinematics}

We derive the relative motion model in the non-inertial frame $s$.
The relative motion model of robot $j$ in the non-inertial frame $s$ is as follows:
{
\begin{align}
    {^{s}\mathbf{a}_{\overrightarrow{s j}}}    & = \underbrace{-[{^{s}\bm{\beta}_{s}}]_{\times} {^{s}\mathbf{p}_{j}} - 2[{^{s}\bm{\omega}_{s}}]_{\times} {^{s}\mathbf{v}_{\overrightarrow{s j}}} - [{^{s}\bm{\omega}_{s}}]^2_{\times} {^{s}\mathbf{p}_{j}}}_{\text{non-inertial acceleration due to frame rotation}}
    \notag
    \\
                                               & + \underbrace{{^{s}\mathbf{R}_{j}} {^{j}\mathbf{a}_{j}} - {^{s}\mathbf{a}_{s}}}_{\text{linear acceleration}},
    \label{Equ: Relative a}
    \\
    {^{j}{\bm{\omega}}_{\overrightarrow{s j}}} & = {^{j}\bm{\omega}_{j}} - {^{s}\mathbf{R}_{j}^{\top}} {^{s}\bm{\omega}_{s}},
    \label{Equ: Relative w}
\end{align}
}
where the gravity components of the measured accelerations are compensated in the linear acceleration term of~\eqref{Equ: Relative a}.
Additionally, ${^{s}\bm{\beta}_{s}}$ denotes the angular acceleration of reference robot $s$, which is typically difficult to obtain directly from IMU measurements.
As a result, some existing works assume ${^{s}\bm{\beta}_{s}}=\mathbf{0}$~\cite{zhang2023coni}, which inevitably introduces modeling errors.
In addition, IMU measurements from robots $j$ and $s$ are generally asynchronous, introducing a time misalignment in \eqref{Equ: Relative a} and \eqref{Equ: Relative w}.
To address both issues, we represent ${^{s}\bm{\omega}_{s}}$ and ${^{s}\mathbf{a}_{s}}$ using two $\mathbb{R}^3$ B-splines, which allows ${^{s}\bm{\beta}_{s}}$ to be obtained analytically via differentiation, while simultaneously enabling queries of the IMU state from robot $s$ at any time instant.

The IMU of reference robot $s$ at timestamp $t_m$ can be queried as:
\begin{equation}
    \begin{aligned}
        ({^{s}\bm{\omega}^{t_m}_{s}}, {^{s}\mathbf{a}^{t_m}_{s}})      & = {\bm{\Psi}^{\kappa}}\big|_{t = t_m},
        \\
        ({^{s}\bm{\beta}^{t_m}_{s}}, {^{s}\mathbf{\dot{a}}^{t_m}_{s}}) & = {\bm{\dot{\Psi}}^{\kappa}}\big|_{t = t_m}.
    \end{aligned}
\end{equation}

\subsubsection{IMU Factor}
\label{SubSec: IMU Factor}

Due to the nonlinearity of (\ref{Equ: Relative w}), the covariance of the relative IMU residuals depends on the optimization variables, which may lead to unstable convergence.
To ensure correct residual weighting under the Mahalanobis distance and MLE formulation, we introduce the reference robot's two IMU B-splines (${^{s}\bm{\omega}_{s}}$ and ${^{s}\mathbf{a}_{s}}$) as auxiliary variables in the optimization, avoiding over-linearization.

The reference robot's IMU residual is defined as:
\begin{equation}
    \begin{aligned}
        \mathbf{r}_\textit{sIMU}\!({^{s}\hat{\mathbf{a}}^{t_m}_{s}}, {^{s}\hat{\bm{\omega}}^{t_m}_{s}}, {\bm{\Psi}^{\kappa}}, {^{s}\mathbf{b}^{\kappa}_{\mathrm{imu}}})\! & =
        \begin{bmatrix}
            {^{s}\mathbf{a}^{t_m}_{s}} - \left({^{s}\hat{\mathbf{a}}^{t_m}_{s}} - {^{s}\mathbf{b}^{\kappa}_{\mathbf{a}}}\right)
            \\
            {^{s}\bm{\omega}^{t_m}_{s}}- \left({^{s}\hat{\bm{\omega}}^{t_m}_{s}} - {^{s}\mathbf{b}^{\kappa}_{\bm{\omega}}}\right)
        \end{bmatrix}.
    \end{aligned}
    \notag
\end{equation}
Here,
\(
{}^{q}\mathbf b_{\mathrm{imu}}^{\kappa}
\!\triangleq\!
[
\left({}^{q}\mathbf b_{a}^{\kappa}\right)^{\top}\!,
\left({}^{q}\mathbf b_{\omega}^{\kappa}\right)^{\top}
]^{\top}
\)
is the IMU bias of robot \(q\).

\begin{figure}[t]
    
    \centering
    \includegraphics[width=\linewidth]{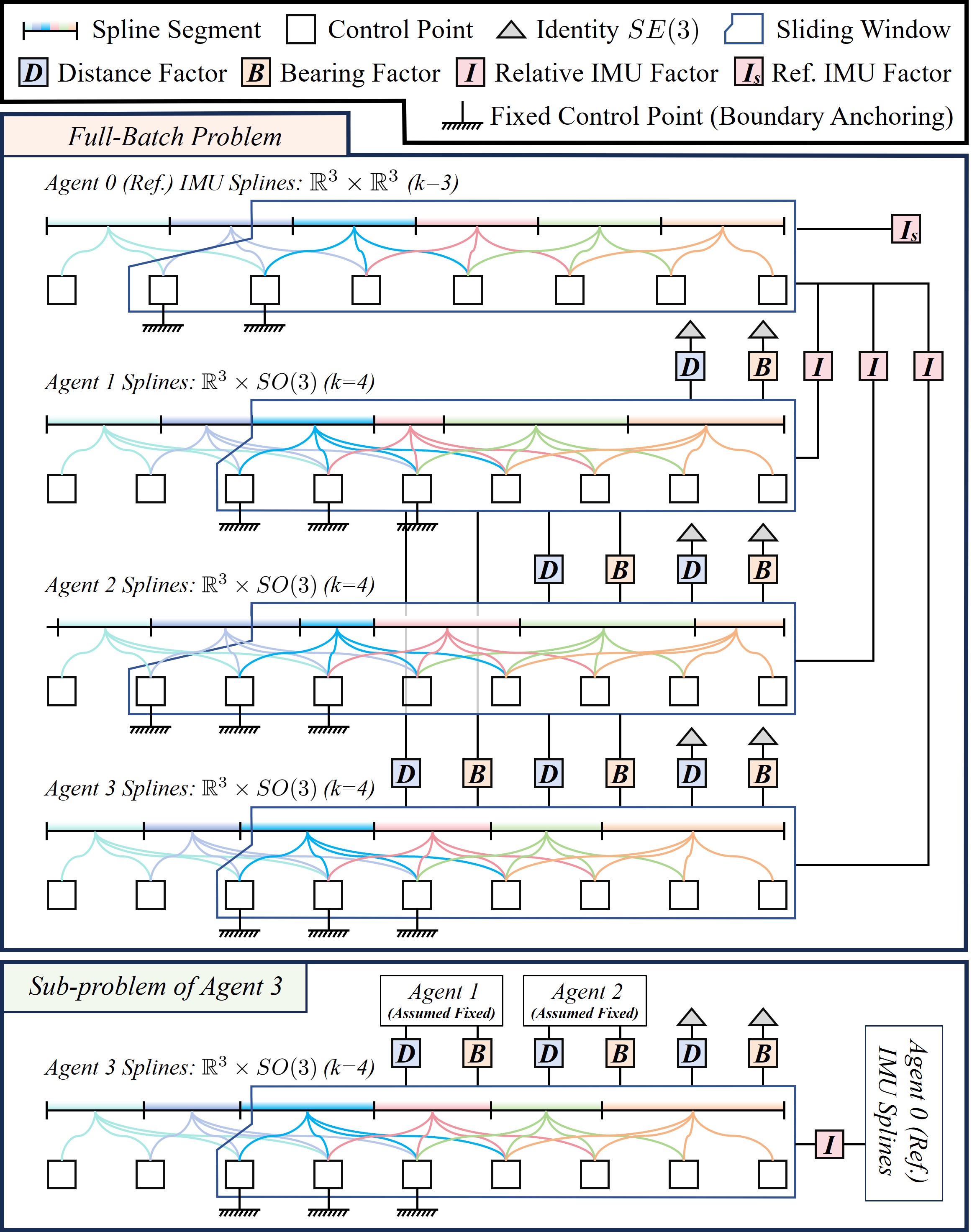}
    \caption{
        Factor-graph-like representation of the full-batch problem for a four-robot example, along with one selected subproblem.
        Each trajectory is assembled from individual segments, where each segment depends on a set of control points.
        A control point remains in the optimization problem as long as any portion of its associated segment lies within the current window.
    }
    \label{Fig: Factor Graph}
\end{figure}

The relative IMU residual is defined as:
{
\begin{align}
    \qquad \mathbf{r}_{i} & ({^{j}\hat{\mathbf{a}}_{j}^{t_m}}, {^{j}\hat{\bm{\omega}}_{j}^{t_m}}, {^{s}{\tau}^{\kappa}_{j}}, {^{s}\bm{\Phi}^{\kappa}_{j}}, {\bm{\Psi}^{\kappa}}, {^{j}\mathbf{b}^{\kappa}_{\mathrm{imu}}}) =
    \begin{bmatrix}
        \mathbf{r}_{ia} \\[-1pt]
        \mathbf{r}_{i\omega}
    \end{bmatrix},
    \\
    \mathbf{r}_{ia}       & = - \left[ {^{s}\bm{\beta}^{t_c}_{s}} \right]_{\times} {^{s}\mathbf{p}_{j}^{t_c}} - 2 \left[ {^{s}\bm{\omega}^{t_c}_{s}} \right]_{\times} {^{s}\mathbf{v}^{t_c}_{\overrightarrow{s j}}} - \left[ {^{s}\bm{\omega}^{t_c}_{s}} \right]_{\times}^2 {^{s}\mathbf{p}_{j}^{t_c}}
    \notag
    \\
                          & \quad + {^{s}}\mathbf{R}_{j}^{t_c} \left({^{j}}\hat{\mathbf{a}}^{t_m}_{j} - {^{j}\mathbf{b}^{\kappa}_{\mathbf{a}}}\right) - {^{s}}\mathbf{a}_{s}^{t_c}
    - {^{s}}\mathbf{a}_{j}^{t_c},
    \\
    \mathbf{r}_{i\omega}  & = {^{j}\bm{\omega}_{\overrightarrow{s j}}^{t_c}} - \left( {^{j}\hat{\bm{\omega}}_{j}^{t_m}} - {^{j}\mathbf{b}^{\kappa}_{\bm{\omega}}} - {{^{s}\mathbf{R}_{j}^{t_c}}}^{\top} {^{s}}\bm{\omega}_{s}^{t_c} \right).
\end{align}
}
Bias residual based on the random walk model is defined as:
\begin{equation}
    \mathbf{r}_{ib}({^j}\mathbf{b}^{\kappa}_{\mathbf{a}}, {^j}\mathbf{b}^{\kappa}_{\bm{\omega}}) =
    \begin{bmatrix}
        {^j}\mathbf{b}^{\kappa}_{\mathbf{a}} - {^j}\mathbf{b}^{\kappa-1}_{\mathbf{a}} \\
        {^j}\mathbf{b}^{\kappa}_{\bm{\omega}} - {^j}\mathbf{b}^{\kappa-1}_{\bm{\omega}}
    \end{bmatrix}.
\end{equation}

The biases are assumed constant within each sliding window and follow a random walk model across consecutive windows.

\subsection{Sliding-Window Boundary Anchoring}
\label{Sec: Marginalization}

In this factor graph structure, Schur-complement marginalization would introduce prior constraints on the retained boundary states.
In \CT{}, we instead directly fix the first $k-1$ control points of each robot's B-spline to implement boundary anchoring, as shown in \autoref{Fig: Factor Graph}.
Unlike probabilistic marginalization, this approximation does not preserve the uncertainty or cross-correlations of the removed states.

\subsection{Full-Batch Optimization Problem Formulation}
\label{Sec: Full-Batch Optimization Problem Formulation}

We formulate the continuous-time relative-inertial odometry problem in a factor graph form.
The graph consists of distance factors $\mathbf{r}_d$, bearing factors $\mathbf{r}_b$, relative IMU factors $\mathbf{r}_i$, bias factors $r_{ib}$, and auxiliary reference IMU factors $\mathbf{r}_\textit{sIMU}$, as shown in \autoref{Fig: Factor Graph}.
To limit the number of optimization variables, we restrict the update of $\tau_j$ to the interval $[-\Delta t_j,\, \Delta t_j]$ per iteration, where $\Delta t_j$ is the minimum knot interval of robot $j$'s B-spline.
Although this bound enlarges the control-point window from \([t_i,\, t_{i+k}]\) to \([t_{i-1},\, t_{i+k+1}]\), it prevents too much shift in query time that would change the associated control points for each factor.
Under this bounded update, $\tau_j$ converges to the true value after several successive optimizations.
Finally, with the measurements during $[t_{\kappa}-\Delta, t_{\kappa}]$, we formulate the following nonlinear least-squares problem:
{
\begin{align}
    \min_{\mathcal{X}^{\kappa}}
    \Big\{
      & \sum_{j \in \mathcal{N}} \sum_{\mathcal{B}^{t_m}_j \in \mathcal{B}^{\kappa}_j}\sum_{k \in \mathcal{B}^{t_m}_j}
    \left\| \mathbf{r}_{b} ({\widehat{\mathbf{z}_b}_{\overrightarrow{jk}}^{t_m}}, {^{s}{\tau}^{\kappa}_{j}}, {^{s}\bm{\Phi}^{\kappa}_{j}}, {^{s}\bm{\Phi}^{\kappa}_{k}}) \right\|_{\bm{\Sigma}_b}^2
    \notag
    \\
    + & \sum_{j \in \mathcal{N}} \sum_{\mathcal{D}^{t_m}_j \in \mathcal{D}^{\kappa}_j}\sum_{k \in \mathcal{D}^{t_m}_j}
    \left\| \mathbf{r}_{d} ({\widehat{\mathbf{z}_d}_{\overrightarrow{jk}}^{t_m}}, {^{s}{\tau}^{\kappa}_{j}}, {^{s}\bm{\Phi}^{\kappa}_{j}}, {^{s}\bm{\Phi}^{\kappa}_{k}}) \right\|_{\bm{\Sigma}_d}^2
    \notag
    \\
    + & \sum_{\mathcal{I}^{t_m}_s \in \mathcal{I}^{\kappa}_s}
    \left\| \mathbf{r}_\textit{sIMU} ({^{s}\hat{\mathbf{a}}_{s}^{t_m}}, {^{s}\hat{\bm{\omega}}_{s}^{t_m}}, {\bm{\Psi}^{\kappa}}, {^{s}\mathbf{b}^{\kappa}_{\mathrm{imu}}}) \right\|_{\bm{\Sigma}_i}^2
    \label{Equ: Optimization Problem}
    \\
    + & \sum_{j \in \mathcal{N}} \sum_{\mathcal{I}^{t_m}_j \in \mathcal{I}^{\kappa}_j}
    \left\| \mathbf{r}_{i} ({^{j}\hat{\mathbf{a}}_{j}^{t_m}}, {^{j}\hat{\bm{\omega}}_{j}^{t_m}}, {^{s}{\tau}^{\kappa}_{j}}, {^{s}\bm{\Phi}^{\kappa}_{j}}, {\bm{\Psi}^{\kappa}}, {^{j}\mathbf{b}^{\kappa}_{\mathrm{imu}}}) \right\|_{\bm{\Sigma}_i}^2
    \notag
    \\
    + & \sum_{j \in \mathcal{V}}
    \left\| \mathbf{r}_{ib} ({^j}\mathbf{b}^{\kappa}_{\mathbf{a}}, {^j}\mathbf{b}^{\kappa}_{\bm{\omega}}) \right\|_{\bm{\Sigma}_{ib}}^2
    \Big\}.
    \notag
\end{align}
}
We use Levenberg-Marquardt (LM) from Ceres Solver \cite{Agarwal_Ceres_Solver_2022} and employ analytical derivatives to solve the problem.

\subsection{Incremental Asynchronous Block Coordinate Descent}
\label{Sec: Incremental Asynchronous Block Coordinate Descent}

The full optimization problem (\ref{Equ: Optimization Problem}) is large-scale and highly nonlinear, especially when the number of robots $n$ is large.
To improve computational efficiency, we simplify the problem and decompose it into smaller components.
Specifically, we break the full problem into $n{-}1$ subproblems, where each subproblem involves only the trajectory of a single robot together with its associated factors, as illustrated by the red box in \autoref{Fig: Factor Graph}.
The reference robot's IMU splines and bias estimates, together with the time-offsets, are prefitted and shared across all subproblems.
Formally, each subproblem is formulated as:
{
\begin{equation}
    \begin{aligned}
        \min_{{\mathcal{X}^{\kappa}_{j}}}
        \Big\{
          & \sum_{\mathcal{B}^{t_m}_j \in \mathcal{B}^{\kappa}_j}\sum_{k \in \mathcal{B}^{t_m}_j}
        \left\| \mathbf{r}_{b}  (\widehat{\mathbf{z}_b}^{t_m}_{\overrightarrow{jk}}, {^{s}{\tau}^{\kappa}_{j}}, {{^s}\bm{\Phi}^{\kappa}_{j}}, {{^s}\bm{\Phi}^{\kappa}_{k}}) \right\|_{\bm{\Sigma}_b}^2
        \\
        + & \sum_{\mathcal{D}^{t_m}_j \in \mathcal{D}^{\kappa}_j}\sum_{k \in \mathcal{D}^{t_m}_j}
        \left\| \mathbf{r}_{d} (\widehat{\mathbf{z}_d}^{t_m}_{\overrightarrow{jk}}, {^{s}{\tau}^{\kappa}_{j}}, {{^s}\bm{\Phi}^{\kappa}_{j}}, {{^s}\bm{\Phi}^{\kappa}_{k}}) \right\|_{\bm{\Sigma}_d}^2
        \\
        + & \sum_{\mathcal{I}^{t_m}_j \in \mathcal{I}^{\kappa}_j}
        \left\| \mathbf{r}_{i} ({^{j}\hat{\mathbf{a}}_{j}^{t_m}}, {^{j}\hat{\bm{\omega}}_{j}^{t_m}}, {^{s}{\tau}^{\kappa}_{j}},  {{^s}\bm{\Phi}^{\kappa}_{j}}, {\bm{\Psi}^{\kappa}}, {^{j}\mathbf{b}^{\kappa}_{\mathrm{imu}}}) \right\|_{\bm{\Sigma}_i}^2
        \\
        + & \left\| \mathbf{r}_{ib} ({^j}\mathbf{b}^{\kappa}_{\mathbf{a}}, {^j}\mathbf{b}^{\kappa}_{\bm{\omega}}) \right\|_{\bm{\Sigma}_{ib}}^2
        \Big\},
        \label{Equ: Optimization Problem BCD}
    \end{aligned}
\end{equation}
}
where $j \in \mathcal{V} \setminus \{s\}$ denotes any non-reference robot.


We solve the subproblems using an IA-BCD algorithm.
IA-BCD has a similar implementation to regular BCD.
The difference is that each subproblem is solved in parallel and asynchronously, without waiting for others to finish and share their updates, but continues to update with the most recent variables available.
This greatly improves computational efficiency but also introduces a potential mismatch between the variables used during local optimization and the latest global state.
Moreover, the extension/shrinkage of C-NUBS and the incremental growth of residuals and variables change the objective over time, which further affects convergence.
We provide convergence analyses in \hyperref[App: Convergence Analysis for Asynchronous BCD]{Appendix}, showing that despite these inconsistencies, IA-BCD converges to a first-order stationary point of the full problem.
In practice, we adopt Ceres’ trust-region LM strategy to achieve a balance between numerical stability and convergence speed.

\begin{figure*}[t]
    
    \centering
    \includegraphics[width=\fullwidth]{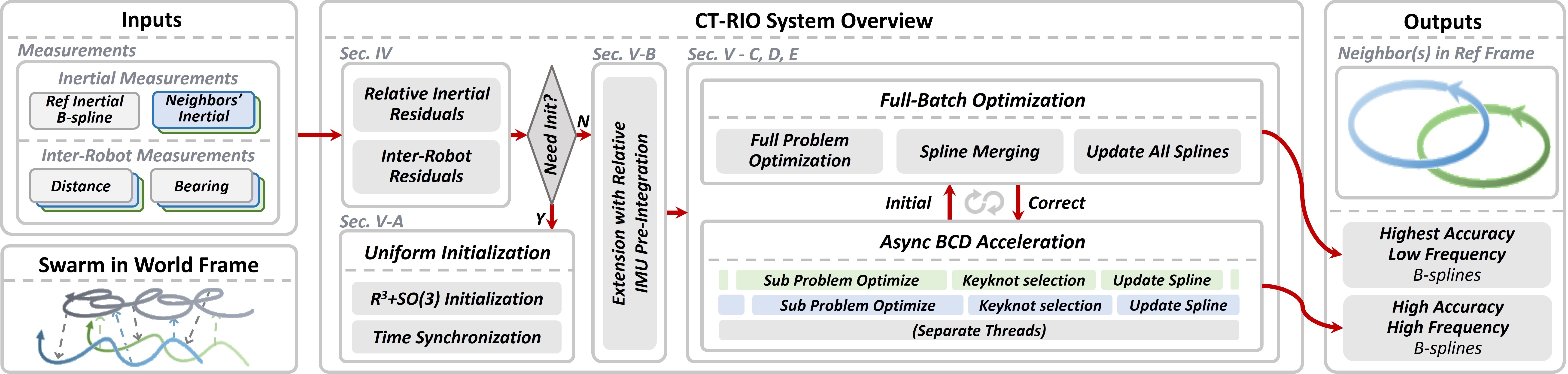}
    \caption{
        \CT{} system overview.
        At the beginning, trajectories are uniformly extended and jointly initialized via full-batch optimization.
        They are then extended at the IMU rate via relative preintegration with adaptive keyknot selection.
        Finally, a hybrid optimization scheme is adopted.
        Asynchronous IA-BCD threads provide high-frequency updates, while periodic full-batch optimization ensures global consistency.
    }
    \label{Fig: System Overview}
\end{figure*}

\section{\CT{} Implementation with C-NUBS}
\label{Sec: CT Implementation with C-NUBS}

Tailored to the heterogeneous motion characteristic of multi-robot systems, we propose a C-NUBS-based continuous-time multi-robot relative-inertial odometry framework, \CT{}, to solve the problem introduced in \autoref{Sec: B-Spline-Based Continuous-Time Multi-Robot Relative State Estimation}.
We use $k=4$ C-NUBS to represent robots' trajectories and $k=3$ unclamped uniform B-splines to represent the reference robot's IMU.
The overall pipeline is illustrated in \autoref{Fig: System Overview}.

At the beginning of the operation or when a new robot joins the network, the uninitialized trajectory is uniformly extended.
Then the spline and the time-offset are initialized via full-batch optimization.
The full-batch optimization will periodically update the time-offset, as detailed in \autoref{Sec: System Initialization}.

Once initialization is complete, the related spline is extended at each incoming IMU measurement, with the new control point initialized using relative IMU preintegration, following \autoref{Sec: Extension with IMU Preintegration Prediction}.
To reduce redundancy while preserving informative knots, we adopt the \textit{knot-keyknot} strategy to retain only keyknots within the sliding window, based on relative motion intensity as described in \autoref{Sec: Relative-Motion-Intensity-Aware Keyknot Strategy}.

To achieve high-frequency and low-latency updates, a two-layer optimization scheme is adopted.
A group of IA-BCD threads continuously optimizes the trajectories by solving the subproblems (\ref{Equ: Optimization Problem BCD}).
Each IA-BCD thread is executed once its corresponding trajectory is extended, providing instant updates at the IMU rate.
Meanwhile, a full-batch optimization thread consistently refines the splines to ensure global consistency.

\subsection{System Initialization}
\label{Sec: System Initialization}

At system startup or when a new robot joins the network, the system performs a full-batch optimization to initialize both the spline trajectory and the time-offset.

When a new robot is detected, an initial estimate of the time-offset is obtained via the Network Time Protocol (NTP), with an accuracy of tens of milliseconds \cite{mani2016mntp}.
This estimate is incorporated as a prior in the optimization.
Then the new robot's trajectory is uniformly extended (10Hz in this paper), and a prior pose sequence from a closed-form single-frame estimator \cite{li2025crepes} is used to provide an informed initialization.


Then, the full-batch optimization (\ref{Equ: Optimization Problem}) is solved iteratively until the estimated trajectories and time-offsets jointly converge.
Since time-offsets drift over long-term operation due to discrepancies in the robot's crystal oscillator, this optimization is periodically re-invoked in a separate thread to continuously refine and update the time-offset.
For efficiency, subsequent operations no longer optimize the time-offsets.

\subsection{Extension with IMU Preintegration Prediction}
\label{Sec: Extension with IMU Preintegration Prediction}

In this work, we extend each robot's spline at each of its incoming IMU measurements to ensure IMU-rate output.
By leveraging the endpoint interpolation property of clamped B-splines, the target extension control point can be initialized using IMU.
We construct single robot's IMU preintegration on the basis of \cite{forster2016manifold}.
Following the relative kinematics formulation in \cite{li2025crepes}, the target new control point of robot $j$'s spline can be computed using the IMU preintegrations of robots $s$ and $j$ from the last keyknot time $t_0$ to the new knot time $t_1$:
\begin{equation}
    \begin{aligned}
        {^{s}\mathbf{p}_{j}^{t_1}}  = & ({\Delta\mathbf{R}^{t_0 \to t_1}_s})^{\top}
        \\
                                      & ({^{s}\mathbf{p}_{j}^{t_0}} + {^{s}\bm{\mu}_{\overrightarrow{sj}}^{t_0}}{\Delta t} + {^{s}\mathbf{R}_{j}^{t_0}}{\Delta\mathbf{p}^{t_0 \to t_1}_j} - {\Delta\mathbf{p}^{t_0 \to t_1}_s}),
        \\
        {^{s}\mathbf{R}_{j}^{t_1}}  = & ({\Delta\mathbf{R}^{t_0 \to t_1}_s})^{\top} ({^s}\mathbf{R}_j^{t_0} {\Delta\mathbf{R}^{t_0 \to t_1}_j}),
    \end{aligned}
    \label{Equ: Relative Kinematics}
\end{equation}
where ${^{s}\bm{\mu}_{\overrightarrow{sj}}^{t_0}} \triangleq {^{s}\mathbf{R}_{w}^{t_0}} {^{w}\bm{\mu}_{\overrightarrow{sj}}^{t_0}}={^{s}\mathbf{v}_{\overrightarrow{s j}}^{t_0}} + {^s}\bm{w}_{\overrightarrow{s j}} \times {^s}\mathbf{p}_{j}^{t_0}$ is the relative translation velocity between robot $s$ and robot $j$ that does not contain the relative velocity generated by the rotation of robot $s$.
The states ${^{s}\mathbf{p}_{j}^{t_0}}$, ${^{s}\mathbf{v}_{\overrightarrow{s j}}^{t_0}}$, and ${^{s}\mathbf{R}_{j}^{t_0}}$ are evaluated from robot $j$'s spline at the previous keyknot time $t_0$.

Then we use the closed-form extension in \autoref{Sec: Closed-Form Extension of C-NUBS} to extend the spline to this new control point and knot.
This provides a reliable starting point for optimization.

\subsection{Relative-Motion-Intensity-Aware Keyknot Strategy}
\label{Sec: Relative-Motion-Intensity-Aware Keyknot Strategy}

Although the spline is extended to a new knot at every IMU timestamp to support high-rate outputs, only keyknots are retained within the sliding window.
Each IA-BCD optimization includes only these keyknots and the new knot.
After IA-BCD optimization, we assess whether to promote the new knot as a keyknot based on relative motion.
Specifically, we compare the relative states at the new knot time $t_1$ with those at the last keyknot time $t_0$, and apply the following criteria:
\begin{itemize}
    \item Relative rotation: if  $\|\Log({^{s}\mathbf{R}_{j}^{t_0}}^{\top} {^{s}}\mathbf{R}_{j}^{t_1})\| > \epsilon_R$, retain the new knot as a keyknot.
    \item Relative translation: if $\|{^{s}\mathbf{p}_{j}^{t_1}} - {^{s}\mathbf{p}_{j}^{t_0}}\|_2 > \epsilon_p$, retain the new knot as a keyknot.
    \item Lower bound: if the time since the last keyknot exceeds $t_\text{max}=200ms$, retain the new knot as a keyknot.
\end{itemize}
Here, $\epsilon_R$ and $\epsilon_p$ are thresholds for relative rotation and translation.
We set them to $\epsilon_p = 0.08\,\mathrm{m}$ and $\epsilon_R = 2.5^\circ$ as reference values used in all experiments.
If the newly added knot is selected as a keyknot, the IMU preintegration is reset to start from this keyknot, and the next knot can be initialized via IMU preintegration from this keyknot efficiently.
This process follows the ``extend-shift-retain'' cycle in \autoref{Sec: Knot-Keyknot Strategy for C-NUBS}.
\subsection{Virtual Extension for Optimization}
\label{Sec: Virtual Extension for Optimization}

Since each trajectory is extended asynchronously, some trajectories may lag behind the current optimization timestamp while already appearing in other robots' factors.
For factor construction, lagging trajectories are \textit{virtually} extended to the latest timestamp with control points set to their most recent IMU-predicted poses, as shown in \autoref{Fig: Virtual}.
After optimization, the \textit{virtual} spline segments are always shrunk back to the latest keyknot using the closed-form shrinkage algorithm. 

\begin{figure}[ht]
    \centering
    \includegraphics[width=\halfwidth]{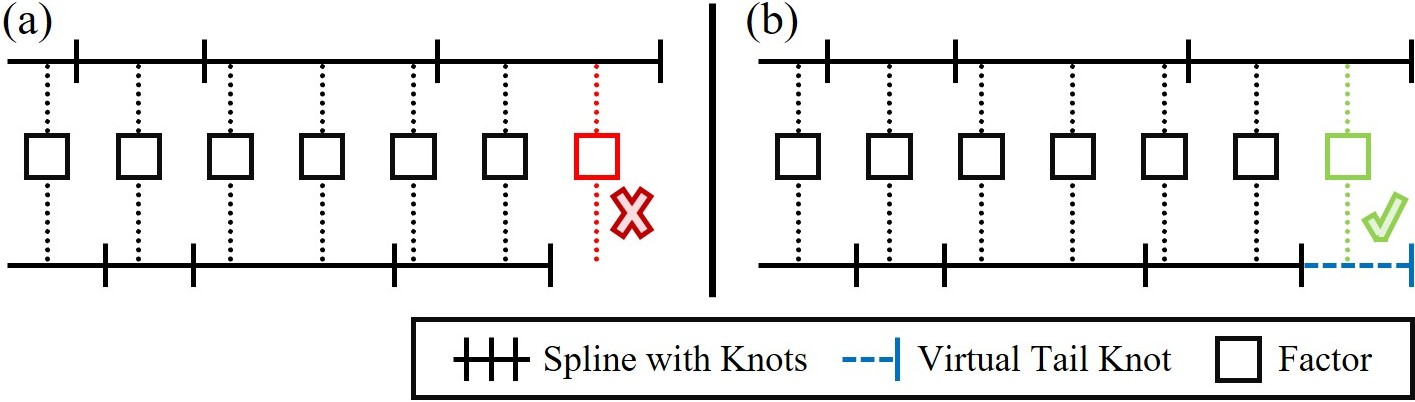}
    \caption{
        (a) Without virtual extension, some factors cannot be constructed due to the lack of the newest spline segment.
        (b) With virtual extension, all splines are aligned to the same time window, and all factors can be constructed.
    }
    \label{Fig: Virtual}
\end{figure}

\vspace{-8pt}
\subsection{Hybrid Optimization Strategy}
\label{Sec: Hybrid Optimization Strategy}

While the \textit{knot-keyknot} strategy enables outputs at the sensor rate at the constructive level, the actual update frequency remains bounded by the optimization speed.
B-spline-based continuous-time estimation suffers from low efficiency due to numerous observations and tightly coupled control points \cite{cioffi2022continuous}.
This becomes more challenging in \CT{} because the inter-robot measurements further couple multiple trajectories.

To balance accuracy and efficiency, we adopt a hybrid strategy:
Once a spline ${^s}\Phi_{j}$ is extended, its corresponding subproblem \eqref{Equ: Optimization Problem BCD} is immediately optimized.
The subproblems \eqref{Equ: Optimization Problem BCD} are optimized asynchronously in parallel to allow immediate updates with minimal latency.
In this way, IA-BCD locally refines newly introduced control points prior to global fusion, providing better initialization than relying solely on IMU preintegration.
At the same time, a full-batch optimization \eqref{Equ: Optimization Problem} is consistently triggered in a separate thread as frequently as possible to enforce global consistency through variable synchronization.
This corrects the incomplete convergence of IA-BCD, ensuring long-term stability.

After full-batch optimization, the latest trajectory may have been extended several times.
We therefore merge the optimized spline with the newly added segments by fitting a fused spline to pose samples from both trajectories, while additionally matching their local linear and angular velocities to preserve smoothness and local shape.

\vspace{-5pt}
\section{Benchmark Analysis}
\label{Sec: Simulation and Public Benchmark Analysis}

Prior to real-world experiments, we conduct a thorough analysis of CT-RIO on synthetic datasets and a patched public dataset.
We evaluate CT-BCD~\eqref{Equ: Optimization Problem BCD} and CT-Full~\eqref{Equ: Optimization Problem} against a DT optimization baseline, namely the MFTO estimator of CREPES-X~\cite{li2025crepes}, in terms of scalability, accuracy, and time-offset estimation.
We further evaluate these methods on the public MILUV dataset~\cite{shalaby2025miluv}, augmented with simulated inter-robot bearing measurements.
We report both the latest optimized output and the complete continuous trajectory.
Runtime includes both problem construction and solving.

All experiments selected Device 0 as the reference device, estimating the relative poses of other devices with respect to it.
Following the evaluation protocol in \cite{zhang2018tutorial}, we compute the Absolute Trajectory Error (ATE) of the estimated relative poses in the local frame of Device 0, adopting the same formulation as \cite{li2025crepes}.
The computation platform is Intel NUC equipped with Intel i7-1260P CPU and 8\,GB of RAM.

\subsection{Synthetic Benchmark Configuration}
\label{Sec: Simulation Setup}

We generate random $6$-order $SE(3)$ B-splines to simulate the trajectories.
The knot interval of the B-spline is set to $1$ second.
The position control points are sampled uniformly within a  $10\!\times\!10\!\times\!10 m^3$ space, while the rotation control points are generated following \cite{yershova2010generating}.
Mutual measurements are generated by sampling positions and rotations along the trajectory.
IMU measurements are derived from the derivatives of the B-spline.
Gaussian noise is manually added to the measurements according to \autoref{Tab: Simulation Configuration}, and the simulated measurement frequencies match the nominal operating rates of the CREPES-X~\cite{li2025crepes} devices used in the real-world experiments.
The noise levels are chosen to provide controlled and reproducible conditions rather than to fully reproduce scenario-dependent hardware noise, and measurements are generated periodically using their original sampling timestamps.
Robot clocks are assumed synchronized except in the dedicated time-offset experiments in~\autoref{Sec: Time-Offset Estimation Analysis}; timestamp jitter, transmission delay, and scenario-dependent measurement unavailability (e.g., occlusion or missed detections) are not modeled.
Unless otherwise specified, the keyknot and keyframe frequencies for CT and DT are set to $10\,\mathrm{Hz}$ for controlled comparison.

\begin{table}[ht]
    \centering
    \caption{Simulation Configuration.}
    \vspace{-4pt}
    \label{Tab: Simulation Configuration}
    \resizebox{0.75\linewidth}{!}{
        \begin{tabular}{lcccc}
            \toprule
            \multirow{2}{*}{Metric} & \multirow{2}{*}{Bearing} & \multirow{2}{*}{Distance} & \multicolumn{2}{c}{IMU}                          \\
                                    &                          &                           & Acc.                    & Gyro.                  \\
            \midrule
            Frequency               & $50\,\mathrm{Hz}$        & $100\,\mathrm{Hz}$        & $100\,\mathrm{Hz}$      & $100\,\mathrm{Hz}$     \\
            Noise {[}$\sigma${]}    & $2.00^\circ$             & $0.10\,\mathrm{m}$        & $0.10\,\mathrm{m/s^2}$  & $0.01\,\mathrm{rad/s}$ \\
            \bottomrule
        \end{tabular}
    }
\end{table}

\vspace{-5pt}
\subsection{Scalability and Runtime Analysis}
\label{Sec: Scalability and Runtime Analysis}

\subsubsection{Scalability under Fixed Knot Density}
To evaluate the scalability of the proposed CT optimization algorithms, we analyze their runtime behavior under varying numbers of robots and time window sizes, and compare them against the DT baseline.
\autoref{Fig: Scalability Results} summarizes the results.
\vspace{-2pt}
\begin{figure}[htb]
    \centering
    \includegraphics[width=\halfwidth]{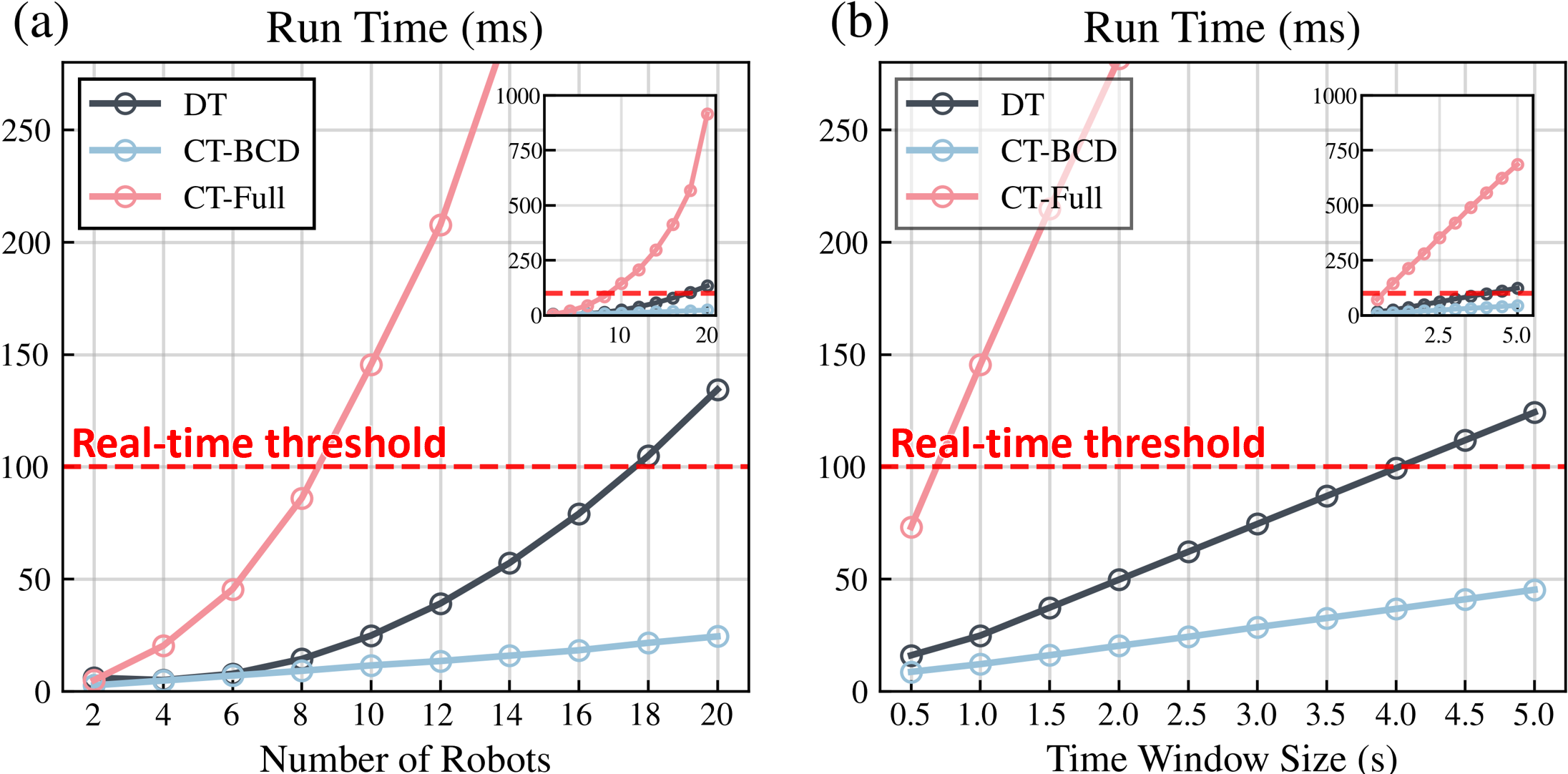}
    \caption{
        Scalability analysis of DT, CT-BCD, and CT-Full.
        Run time versus (a) the number of robots with a fixed $1\,\mathrm{s}$ time window, and (b) the time window size with 10 robots.
    }
    \label{Fig: Scalability Results}
\end{figure}


\autoref{Fig: Scalability Results} (a) shows that DT exhibits exponential runtime growth with the increase of the number of devices, surpassing the real-time threshold ($100\,\mathrm{ms}$) at around 18 devices.
CT-Full also grows exponentially, exceeding real-time constraints at about 10 devices and rising sharply to approach $1\,\mathrm{s}$ once the device count exceeds 20.
In contrast, CT-BCD scales more gently and consistently remains below $25\,\mathrm{ms}$ even with 20 devices.
With respect to time window size, all three methods exhibit linear runtime growth, as illustrated in \autoref{Fig: Scalability Results}~(b).
Despite linear complexity, CT-Full’s large baseline runtime prevents real-time operation even with a $1\,\mathrm{s}$ window.

We further analyze the composition of run time using 10 devices, as shown in \autoref{Tab: Exp: Total Processing Time}.
For a fair comparison with DT, we present the time taken by CT methods to reconstruct the optimization problem in Ceres \cite{Agarwal_Ceres_Solver_2022} from scratch each time (column ``Standard'' of \autoref{Tab: Exp: Total Processing Time}), consistent with DT and other Ceres-based methods.
DT discards part of the observations and requires less time for problem construction, and also its weaker variable coupling leads to a shorter solving time compared with CT-Full.
Since all measurements are included in the optimization, CT-Full exhibits significantly higher construction and solving times than the other methods.
Meanwhile, CT-BCD achieves a lower construction time than CT-Full and the lowest solving time due to its block-wise strategy.
\begin{table}[ht]
    \centering
    \caption{Total Processing Time in Milliseconds ($\mathrm{ms}$)}
    \vspace{-4pt}
    \label{Tab: Exp: Total Processing Time}
    \resizebox{\halfwidth}{!}{
        \begin{tabular}{|c|c|c|c|c|c|c|}
            \hline \rule{0pt}{8pt}
            \multirow{3}{*}{Window}     &
            \multicolumn{2}{c|}{DT}     &
            \multicolumn{2}{c|}{CT-BCD} &
            \multicolumn{2}{c|}{CT-Full}
            \\
            \cline{2-7} \rule{0pt}{8pt}
                                        &
            \multirow{2}{*}{Build}      &
            \multirow{2}{*}{Solve}      &
            {Build}                     &
            \multirow{2}{*}{Solve}      &
            {Build}                     &
            \multirow{2}{*}{Solve}
            \\
            \cline{4-4} \cline{6-6}
                                        &      &       & {Stand. / Incre.} &      & {Stand. / Incre.} &
            \\
            \hline \rule{0pt}{8pt}
            1s                          & 3.0  & 20.2  & {6.8 / 3.2}       & 8.7  & {18.3 / 11.1}     & 134.4 \\
            \hline \rule{0pt}{8pt}
            3s                          & 8.7  & 57.3  & {18.9 / 3.1}      & 25.2 & {53.4 / 11.1}     & 409.8 \\
            \hline \rule{0pt}{8pt}
            5s                          & 14.4 & 102.8 & {31.7 / 3.1}      & 41.9 & {91.2 / 11.1}     & 674.5 \\
            \hline
        \end{tabular}
    }
\end{table}

However, reconstructing the optimization problem each time is computationally expensive and can even approach the solving time in CT-BCD.
Therefore, we implement an incremental problem construction strategy for CT methods and report its effect in \autoref{Tab: Exp: Total Processing Time} (column ``Incremental'').
This strategy updates only the parts that change between optimizations by adding newly introduced variables and factors and removing out-of-window variables and factors.
As a result, the problem construction time for CT methods is significantly reduced and becomes nearly constant with respect to the time window size, which is particularly advantageous for large-window scenarios.

\subsubsection{Runtime under Adaptive Non-Uniform Knots}
To further analyze the source of the speedup and the effect of the proposed non-uniform strategy on optimization efficiency, we conduct a study with $10$ robots under smooth, moderate, and aggressive motion.
All three benchmarks use the same sensor rates, noise model, $1\,\mathrm{s}$ sliding window, optimization settings, and keyknot selection rule.
The resulting problem sizes and runtime statistics are summarized in Tab.~\ref{tab:controlled_runtime}.

\begin{table}[ht]
    \centering
    \caption{
        Controlled runtime comparison.
        Average speed is pooled over all robots; average relative speed is computed with respect to robot~0 from the ground truth.
    }
    \vspace{-3pt}
    \label{tab:controlled_runtime}
    \resizebox{.87\linewidth}{!}{%
        \begin{tabular}{l@{\hspace{4pt}}c@{\hspace{3pt}}c@{\hspace{3pt}}c@{\hspace{6pt}}ccc@{\hspace{3pt}}c@{\hspace{4pt}}c@{}}
            \toprule
            Motion
             & \begin{tabular}[c]{@{}c@{}}Avg.\\Speed\\$[\mathrm{m/s}]$\end{tabular}
             & \begin{tabular}[c]{@{}c@{}}Avg. Rel.\\Speed\\$[\mathrm{m/s}]$\end{tabular}
             & \begin{tabular}[c]{@{}c@{}}Keyknot\\Freq.\\$[\mathrm{Hz}]$\end{tabular}
             & Optimizer
             & \begin{tabular}[c]{@{}c@{}}Active\\Spline\\DoF\end{tabular}
             & \begin{tabular}[c]{@{}c@{}}Residual\\Blocks\end{tabular}
             & \begin{tabular}[c]{@{}c@{}}Median\\Iteration\end{tabular}
             & \begin{tabular}[c]{@{}c@{}}Median\\Solve Time\\$[\mathrm{ms}]$\end{tabular} \\
            \midrule
            \multirow{2}{*}{Smooth}
             & \multirow{2}{*}{0.04}
             & \multirow{2}{*}{0.06}
             & \multirow{2}{*}{4.77}
             & CT-BCD & 36   & 2706  & 3 & 8.97   \\
             &        &       &      & CT-Full   & 438  & 14600 & 4 & 103.19 \\
            \midrule
            \multirow{2}{*}{Moderate}
             & \multirow{2}{*}{0.53}
             & \multirow{2}{*}{0.74}
             & \multirow{2}{*}{11.76}
             & CT-BCD & 78   & 2713  & 3 & 10.55  \\
             &        &       &       & CT-Full   & 858  & 14600 & 4 & 130.51 \\
            \midrule
            \multirow{2}{*}{Aggressive}
             & \multirow{2}{*}{2.82}
             & \multirow{2}{*}{6.37}
             & \multirow{2}{*}{18.41}
             & CT-BCD & 126  & 2721  & 3 & 11.59  \\
             &        &       &       & CT-Full   & 1230 & 14600 & 4 & 162.24 \\
            \bottomrule
        \end{tabular}%
    }
\end{table}

As motion becomes more intense, the adaptive keyknot frequency increases accordingly, leading to more active spline DoFs for both CT-BCD and CT-Full.
With the number of residual blocks and LM iterations remaining nearly unchanged, the median runtime of CT-BCD increases by approximately $17\%/29\%$, while that of CT-Full increases by approximately $26\%/57\%$, compared with the mildest motion condition.
Meanwhile, across all three motion conditions, CT-BCD achieves approximately $11.5$--$14.0\times$ lower median runtime than CT-Full because each update activates only one robot trajectory and its associated factors, substantially reducing the active state dimension and the number of residual blocks.

\subsection{Accuracy Analysis}
\label{Sec: Accuracy Analysis}

In this section, we analyze the performance of IA-BCD and full-batch optimization to characterize their accuracy and clarify the motivation for the Hybrid strategy.
To enable a controlled comparison, we isolate the IA-BCD and Full-Batch components from the Hybrid strategy (see \autoref{Sec: Hybrid Optimization Strategy}), constructing IA-BCD-Only and Full-Batch-Only variants, and compare them with DT.
The corresponding outputs are denoted as CT-BCD\textsubscript{(H)}, CT-Full\textsubscript{(H)}, CT-BCD\textsubscript{(I)}, and CT-Full\textsubscript{(F)}.


\subsubsection{Full-Batch-Only}

As shown in \autoref{Fig: online}, CT-Full\textsubscript{(F)} consistently achieves the lowest error, with estimation accuracy improving monotonically as the number of robots increases.
Meanwhile, the accuracy of the complete trajectory is consistently higher than that of the latest output.
This behavior contrasts with DT, where the complete trajectory does not exhibit comparable improvements.
We attribute this gap to differences in measurement time alignment:
the evaluated DT baseline requires measurement interpolation or discarding, which limits information utilization and attainable accuracy, whereas CT-Full\textsubscript{(F)} evaluates all measurements at their original timestamps, providing substantially richer constraints.
In addition, CT-Full\textsubscript{(F)} effectively leverages accumulated historical information to further improve complete trajectory accuracy by refining intermediate states using temporal context.
However, due to its tightly coupled optimization structure, CT-Full\textsubscript{(F)} shows significantly poorer real-time scalability than DT, as indicated by the background colors in \autoref{Fig: online}(a)(b).

\begin{figure}[t]
    
    \definecolor{rtBlue}{HTML}{6fc3df}   
    \definecolor{rtGreen}{HTML}{6cc76c}  
    \definecolor{rtYellow}{HTML}{f0c370} 
    \definecolor{rtRed}{HTML}{FF9494}    
    \centering
    \includegraphics[width=\halfwidth]{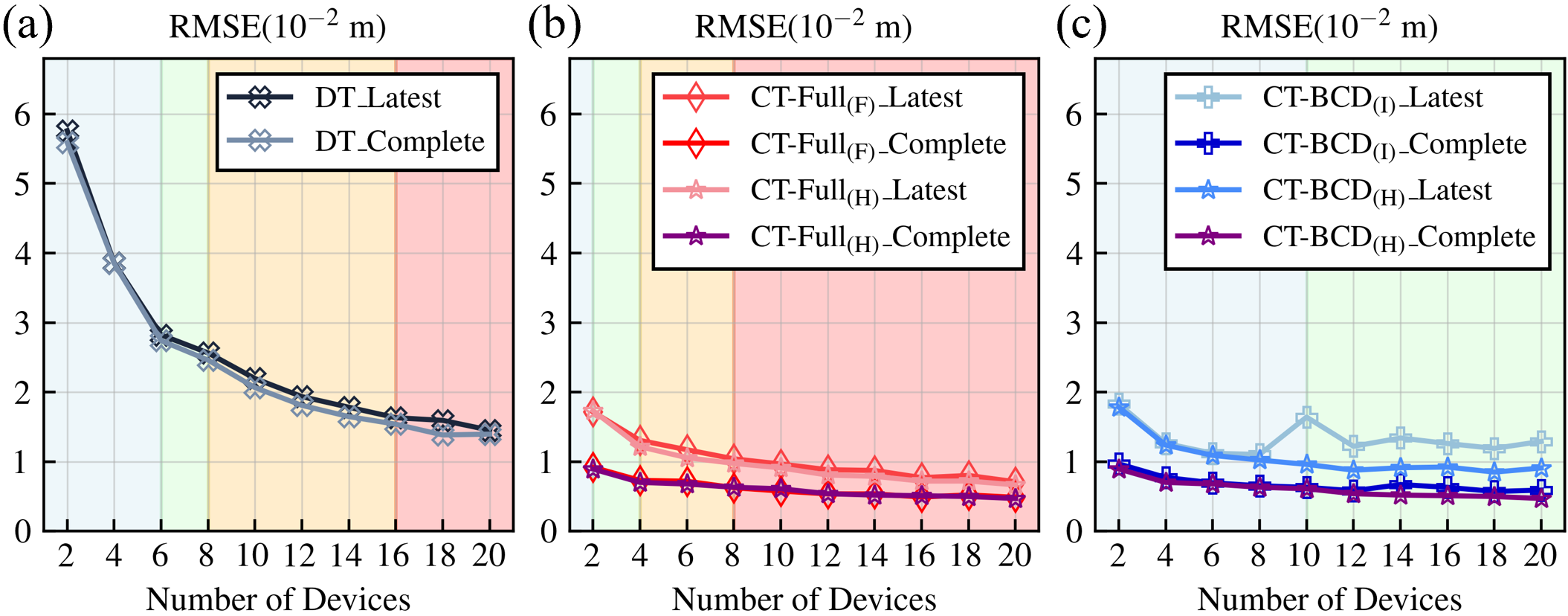}
    \caption{
        RMSE of different CT optimization strategies and DT versus the number of robots, with a fixed $1\,\mathrm{s}$ time window.
        Background colors indicate real-time capability: \textbf{\textcolor{rtBlue}{Blue}}: IMU-rate ($<10\,\mathrm{ms}$); \textbf{\textcolor{rtGreen}{Green}}: Camera-rate ($<20\,\mathrm{ms}$); \textbf{\textcolor{rtYellow}{Yellow}}: Real-time ($<100\,\mathrm{ms}$); \textbf{\textcolor{rtRed}{Red}}: Non-real-time ($>100\,\mathrm{ms}$).
    }
    \label{Fig: online}
\end{figure}

\subsubsection{IA-BCD-Only}

In IA-BCD-Only strategy, CT-BCD\textsubscript{(I)} exhibits significantly different behavior compared to CT-Full\textsubscript{(F)}.
As shown in \autoref{Fig: online}(c), for the latest output, the estimation error initially decreases as the number of robots increases, but begins to rise once the number of robots exceeds 10.
Meanwhile, the solver can no longer perform optimization at every extension step when beyond 10 robots, as shown in the background color of \autoref{Fig: online}(c).
This suggests that fewer optimization updates limit the convergence of CT-BCD\textsubscript{(I)}.
For the complete trajectory, each variable undergoes more subproblem optimization iterations.
As a result, CT-BCD\textsubscript{(I)} is able to converge to an accuracy comparable to CT-Full\textsubscript{(F)} when a sufficient number of iterations is allowed.

To further analyze this effect, we eliminate the influence of varying computation times by running all methods at fixed rates ($100\,\mathrm{Hz}$ for IA-BCD and $10\,\mathrm{Hz}$ for Full-Batch and DT) and increasing the time window size up to $5.5\,\mathrm{s}$.
Under fixed rates, a larger window effectively provides each variable with more optimization updates, allowing us to analyze convergence behavior.
As shown in \autoref{Fig: hearmaps small}(b)(c), as the number of devices increases, the iterations required by CT-BCD\textsubscript{(I)} to reach an accuracy comparable to CT-Full\textsubscript{(F)} grow significantly.
This trend indicates that although IA-BCD also benefits from the richer information of a larger group, its convergence slows due to the increased problem scale and weaker coupling to the global objective.
Despite this limitation, CT-BCD\textsubscript{(I)} demonstrates superior real-time scalability compared to other methods, as illustrated by the background colors in \autoref{Fig: online}(c).

\begin{figure}[h]
    
    \centering
    \includegraphics[width=\halfwidth]{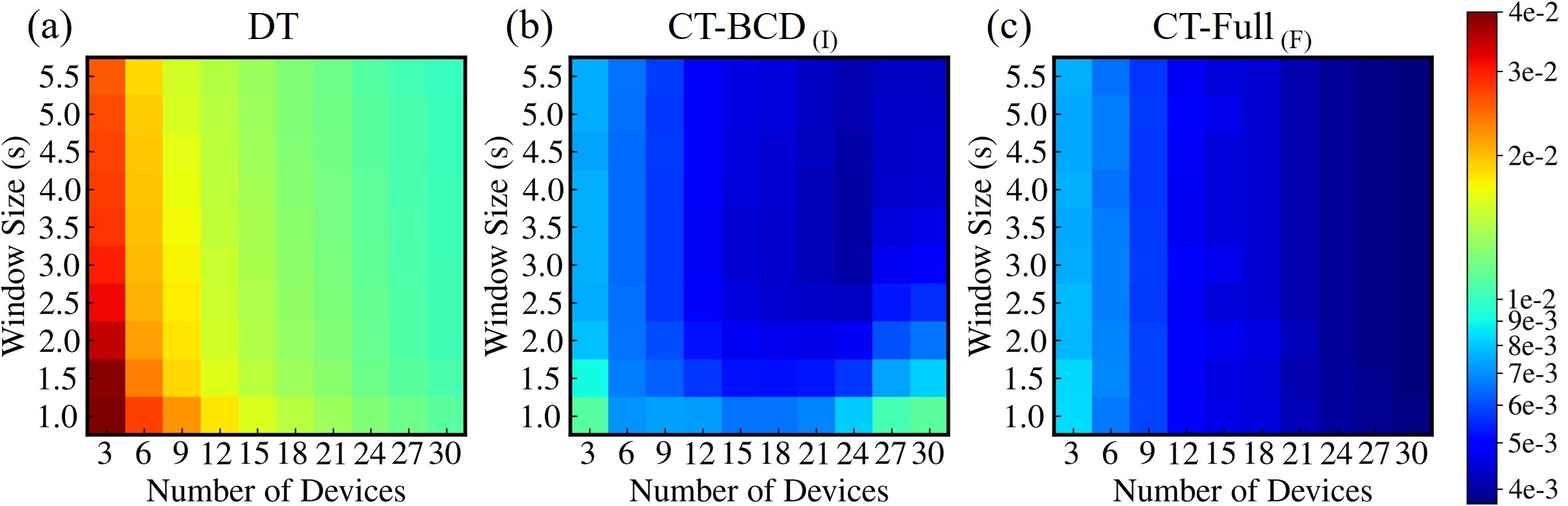}
    \caption{
        Comparison of complete trajectory RMSE for different CT optimization strategies and DT under fixed update rates, across varying numbers of robots and sliding-window sizes.
    }
    \label{Fig: hearmaps small}
\end{figure}

\subsubsection{Hybrid}
For the latest output, periodic full-batch refinements accelerate the convergence of CT-BCD by globally correcting the trajectory.
As shown in \autoref{Fig: online}(c), this allows CT-BCD\textsubscript{(H)} to achieve accuracy close to CT-Full\textsubscript{(H)} (\autoref{Fig: online}(b)) and significantly outperform standalone CT-BCD\textsubscript{(I)}.
Meanwhile, the CT-Full\textsubscript{(H)} consistently maintains its high accuracy.
For the complete trajectory, the Hybrid strategy achieves accuracy that closely matches pure full-batch optimization.
Overall, the Hybrid strategy inherits the fast global convergence and high accuracy of full-batch optimization, while leveraging the computational efficiency of IA-BCD for real-time updates.

\subsection{Time-Offset Estimation Analysis}
\label{Sec: Time-Offset Estimation Analysis}
This experiment evaluates the effectiveness and robustness of the proposed time-offset estimation mechanism under varying levels of inter-device clock misalignment.
We construct a controlled benchmark by artificially adding additional time-offsets of each device in the simulation spline.
For a given time-offset magnitude, the time-offsets are distributed uniformly within the range to emulate realistic clock discrepancies, as shown in the start of \autoref{Fig: Time-Offset Results}.

\autoref{Tab: Exp: Time-Offset Benchmark} reports the position RMSE of DT and CT-Full, together with the time-offset estimation error and the number of optimization steps required for convergence.
Each step corresponds to a complete full-batch optimization, in which the time-offset variable $\tau$ is jointly updated within the bounded interval $[-\Delta t_j,\, \Delta t_j]$, as formulated in \autoref{Sec: Full-Batch Optimization Problem Formulation}.
When no time-offset is introduced, CT-Full consistently achieves lower RMSE than DT across all window sizes and numbers of devices.
Meanwhile, the estimated time-offset remains tightly concentrated around zero with small variance, indicating that the introduction of time-offset variables does not induce bias or instability in the nominal synchronized setting.

\begin{table}[t]
    \centering
    \caption{Time-offset estimation under different window sizes
        and robot counts.
        Values in parentheses denote the number of optimization steps
        required for convergence.}
    \label{tab:time_offset_window}
    \resizebox{\halfwidth}{!}{

        \begin{tabular}{cccccccc}
            \toprule
            \multirow{2}{*}{\vspace{-5pt}\begin{tabular}[c]{@{}c@{}} Time\\Offset\end{tabular}}
                 &
            \multirow{2}{*}{\vspace{-5pt}\begin{tabular}[c]{@{}c@{}} Window\\Size \end{tabular}}
                 &
            \multicolumn{3}{c}{2 Robots}
                 &
            \multicolumn{3}{c}{5 Robots}
            \\
            \cmidrule(lr){3-5} \cmidrule(lr){6-8}
                 &
                 &
            DT
                 &
            CT
                 &
            $\tau$ Err. (Steps)
                 &
            DT
                 &
            CT
                 &
            $\tau$ Err. (Steps)
            \\
            (ms) & (s)   & (cm) & (cm)                   & (ms) & (cm) & (cm) & (ms)
            \\
            \midrule

            \multirow{3}{*}{$0$}
                 & 1
                 & 5.0   & 1.4  & $-2.63{\pm}2.69\,(1)$
                 & 3.2   & 0.9  & $0.04{\pm}1.30\,(1)$
            \\
                 & 3
                 & 3.7   & 1.1  & $-1.26{\pm}1.39\,(1)$
                 & 2.6   & 0.9  & $-0.22{\pm}0.55\,(1)$
            \\
                 & 5
                 & 3.7   & 1.0  & $-0.40{\pm}1.12\,(1)$
                 & 2.7   & 0.9  & $-0.27{\pm}0.40\,(1)$
            \\

            \midrule

            \multirow{3}{*}{$100$}
                 & 1
                 & 20.6  & 1.1  & $-2.78{\pm}3.11\,(12)$
                 & 18.4  & 0.9  & $-0.05{\pm}1.28\,(6)$
            \\
                 & 3
                 & 25.9  & 1.1  & $-1.31{\pm}1.35\,(2)$
                 & 19.5  & 0.8  & $-0.23{\pm}0.56\,(2)$
            \\
                 & 5
                 & 29.2  & 1.0  & $-0.42{\pm}1.13\,(2)$
                 & 19.7  & 0.9  & $-0.26{\pm}0.40\,(2)$
            \\

            \midrule

            \multirow{3}{*}{$1000$}
                 & 1
                 & 338.5 & 1.4  & $-1.09{\pm}3.45\,(34)$
                 & 209.3 & 0.9  & $-0.17{\pm}1.59\,(29)$
            \\
                 & 3
                 & 342.5 & 1.0  & $0.27{\pm}0.37\,(27)$
                 & 220.8 & 0.9  & $-0.15{\pm}0.71\,(27)$
            \\
                 & 5
                 & 349.2 & 1.1  & $0.40{\pm}0.34\,(20)$
                 & 214.9 & 0.9  & $-0.03{\pm}0.37\,(28)$
            \\

            \bottomrule
        \end{tabular}}
    \label{Tab: Exp: Time-Offset Benchmark}
\end{table}

\begin{figure}[t]
    
    \centering
    \includegraphics[width=0.9\halfwidth]{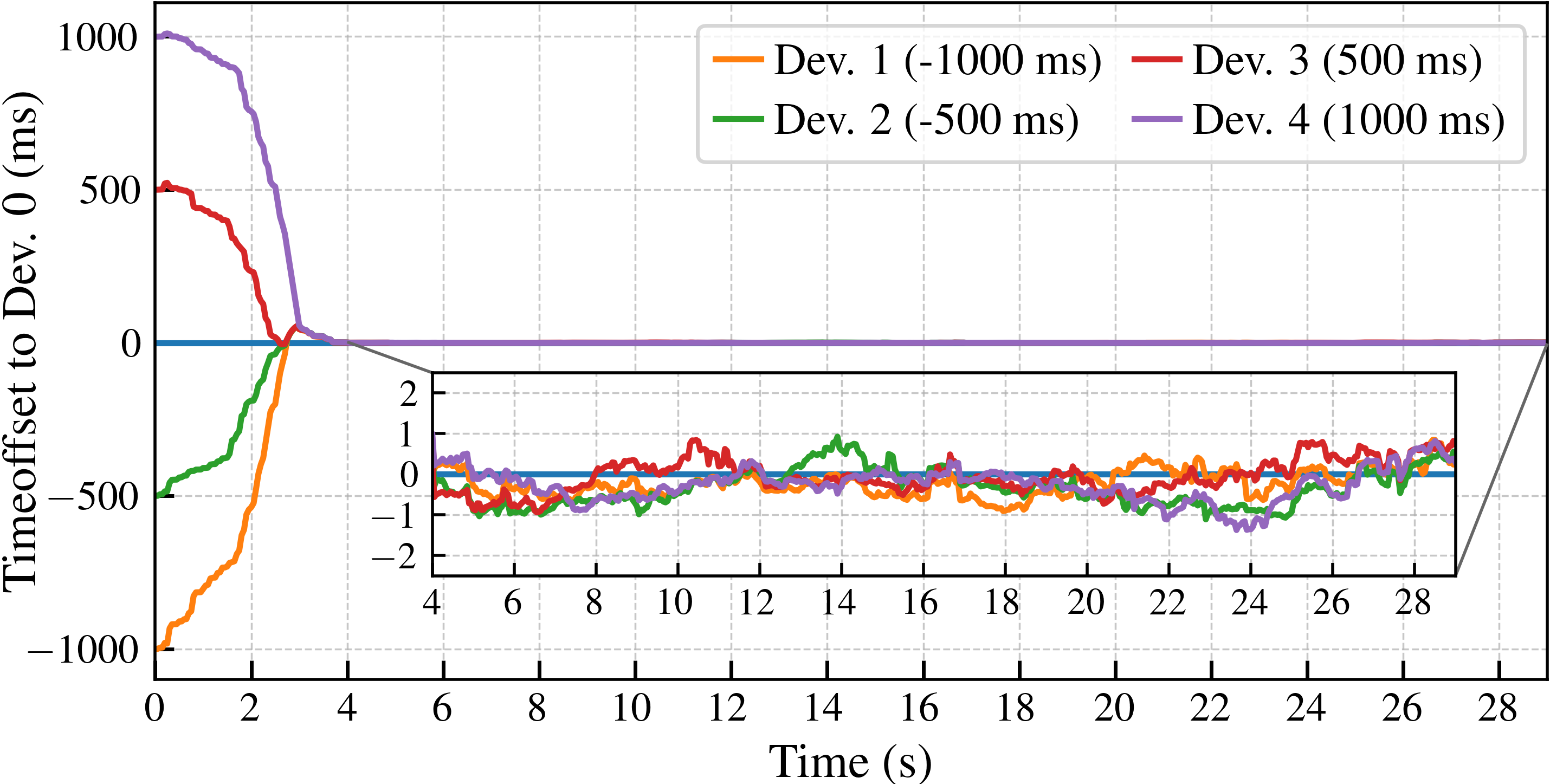}
    \caption{
        Time-offset estimation results of five devices under $5\,\mathrm{s}$ time window.
    }
    \label{Fig: Time-Offset Results}
    \vspace{2pt}
\end{figure}

As the time-offset increases to $100\,\mathrm{ms}$, DT exhibits significant performance degradation due to temporal misalignment, whereas CT-Full maintains stable accuracy with RMSE comparable to the zero-offset case.
The estimated time-offset converges to the ground truth with a millisecond mean error, indicating reliable temporal calibration.
Under severe time-offsets of $1000\,\mathrm{ms}$, DT fails completely with meter-level RMSE, while CT-Full consistently achieves accurate estimation, keeping RMSE below $0.014\,\mathrm{m}$.
Although larger time-offsets require more optimization steps due to bounded ranges, the final accuracy of CT-Full is not compromised.
\autoref{Fig: Time-Offset Results} illustrates the time-offset estimation process of the five devices using a $5\,\mathrm{s}$ sliding time window.
The time-offsets converge within the first $4\,\mathrm{s}$ and remain stable for the remaining time.
Meanwhile, larger window sizes and more devices improve the stability and convergence of time-offset estimation by providing richer constraints and redundancy, demonstrating the robustness and scalability of the proposed continuous-time formulation with joint time-offset estimation.

\vspace{-5pt}
\subsection{Public Benchmark Evaluation on MILUV}
\label{Sec: Public Benchmark Evaluation on MILUV}

To further evaluate CT-RIO beyond synthetic and self-collected data, we conduct experiments on the public MILUV multi-UAV dataset~\cite{shalaby2025miluv}.
MILUV provides three UAVs with 6-DoF motion, onboard IMUs, inter-robot UWB ranging, and motion-capture ground truth.
Because processed inter-robot bearings are unavailable, we synthesize them from ground truth at $50\,\mathrm{Hz}$ with $2^\circ$ Gaussian noise.
We use only direct two-way ranging (TWR) measurements from one fixed UWB tag per UAV (tags $10$, $20$, and $30$), consistent with CT-RIO's single-extrinsic scalar-range model, and discard passive-listening records and measurements involving other tags.
Thus, the benchmark retains real flight dynamics, inertial sensing, UWB ranging, and asynchronous measurement timing, while only the bearing front end is simulated.
We use eight of the 16 sequences: six are excluded due to motion-capture discontinuities, and two due to sparse UWB measurements ($\sim0.13\,\mathrm{Hz}$ per robot pair with gaps exceeding $15\,\mathrm{s}$), which fall outside the intended operating regime.
Statistics of the retained sequences are summarized in Table~\ref{tab:miluv}.

We compare against DT (the MFTO estimator of CREPES-X~\cite{li2025crepes}) and CREPES~\cite{xun2023crepes}, where CREPES is a discrete-time direct relative localization pipeline combining pairwise ESKFs with a pose-graph optimization (PGO) back end.
Following its original implementation, CREPES uses a complementary filter to estimate the gravity direction from accelerometer and gyroscope measurements, whereas DT and CT-RIO directly use the raw IMU measurements.

\begin{table}[ht]
    \centering
    \caption{
        Sequence statistics and RMSE of the latest output and complete trajectory for \textbf{\texttt{MILUV}} sequences.
        The best result of each group is in bold.
        The abbreviations D, MT, R, Z, N, and OT denote \texttt{default}, \texttt{movingTriangle}, \texttt{random}, \texttt{zigzag}, \texttt{noAprilTags}, and \texttt{oneTagPerRobot}, respectively.
    }
    \vspace{-4pt}
    \label{tab:miluv}
    \resizebox{\halfwidth}{!}{%
        \begin{tabular}{l*{3}{>{\centering\arraybackslash}m{0.8cm}}cccccc}
            \toprule
            \multirow{4}{*}{\vspace{-5pt}Sequence}
             & \multicolumn{3}{c}{Sequence Statistics}
             & \multicolumn{4}{c}{Latest Output}
             & \multicolumn{2}{c}{Complete Trajectory} \\
            \cmidrule(lr){2-4}\cmidrule(lr){5-8}\cmidrule(lr){9-10}
             & \multirow{2}{*}{Time}
             & \multirow{2}{*}{\makebox[0.8cm][c]{\begin{tabular}[c]{@{}c@{}}Avg. Rel.\\Speed\end{tabular}}}
             & \multirow{2}{*}{\begin{tabular}[c]{@{}c@{}}UWB\\Rate\end{tabular}}
             & CREPES
             & DT
             & CT-BCD
             & CT-Full
             & DT
             & CT \\
             &
             &
             &
             & Pos / Rot
             & Pos / Rot
             & Pos / Rot
             & Pos / Rot
             & Pos / Rot
             & Pos / Rot \\
             & ($\mathrm{s}$)
             & ($\mathrm{m/s}$)
             & ($\mathrm{Hz}$)
             & ($\mathrm{cm}$) / ($\degree$)
             & ($\mathrm{cm}$) / ($\degree$)
             & ($\mathrm{cm}$) / ($\degree$)
             & ($\mathrm{cm}$) / ($\degree$)
             & ($\mathrm{cm}$) / ($\degree$)
             & ($\mathrm{cm}$) / ($\degree$) \\
            \midrule
            \texttt{D\_MT\_0b} & 205.3 & 0.32 & 2.8 & 40.0 / 4.2 & 12.7 / 2.6 & 13.9 / 1.3 & \textbf{12.1 / 1.3} & 12.2 / 2.6 & \textbf{9.4 / 1.2} \\
            \texttt{D\_R3\_0b} & 222.8 & 0.80 & 2.7 & 42.6 / 6.7 & 22.1 / 2.9 & 17.1 / 1.9 & \textbf{14.9 / 1.9} & 21.7 / 2.8 & \textbf{8.1 / 1.8} \\
            \texttt{D\_R3\_2}  & 204.8 & 0.88 & 2.5 & 56.9 / 4.4 & 27.6 / 2.9 & 20.0 / 1.7 & \textbf{15.5 / 1.6} & 26.7 / 2.9 & \textbf{9.2 / 1.4} \\
            \texttt{D\_R\_0}   & 80.9  & 0.86 & 2.8 & 51.6 / 4.2 & 28.6 / 2.9 & 18.0 / 1.8 & \textbf{14.8 / 1.6} & 27.6 / 2.9 & \textbf{8.7 / 1.6} \\
            \texttt{D\_R\_0b}  & 91.9  & 0.70 & 2.8 & 76.3 / 4.2 & 25.2 / 2.9 & 26.3 / 1.8 & \textbf{23.1 / 1.8} & 24.6 / 2.9 & \textbf{7.4 / 1.7} \\
            \texttt{D\_Z\_2}   & 275.6 & 0.99 & 2.4 & 53.5 / 4.7 & 30.7 / 3.3 & 25.7 / 1.9 & \textbf{17.3 / 1.8} & 29.8 / 3.3 & \textbf{8.6 / 1.8} \\
            \texttt{N\_R3\_2b} & 206.0 & 0.87 & 2.5 & 48.3 / 4.3 & 25.2 / 2.8 & 19.6 / 1.6 & \textbf{16.5 / 1.5} & 24.6 / 2.8 & \textbf{9.3 / 1.3} \\
            \texttt{OT\_R\_0}   & 101.0 & 0.76 & 7.2 & 36.4 / 4.0 & 20.7 / 2.9 & 13.9 / 2.0 & \textbf{13.5 / 1.8} & 20.5 / 2.9 & \textbf{12.5 / 1.8} \\
            \midrule
            \textbf{Mean} & 173.5 & 0.77 & 3.2 & 50.7 / 4.6 & 24.1 / 2.9 & 19.3 / 1.8 & \textbf{16.0 / 1.7} & 23.5 / 2.9 & \textbf{9.2 / 1.6} \\
            \bottomrule
        \end{tabular}%
    }
\end{table}

As shown in Table~\ref{tab:miluv}, CT-Full reduces the mean position and rotation RMSE by approximately $34\%$ and $43\%$, respectively, compared with DT, and by approximately $68\%$ and $64\%$ compared with CREPES.
For the complete trajectory, CT-RIO further reduces the mean position/rotation RMSE from $23.5\,\mathrm{cm}/2.9^\circ$ for DT to $9.2\,\mathrm{cm}/1.6^\circ$.
These results confirm that the performance advantage observed in the controlled simulations is retained under real multi-UAV motion, inertial disturbances, and sparse asynchronous UWB measurements.

\section{Real-World Experiments}
\label{Sec: Real-World Experiments}
In this section, we evaluate \CT{} through real-world experiments covering heterogeneous platforms, diverse relative motion patterns, LOS/NLOS conditions, different numbers of robots, and inter-robot clock unsynchronization.
\vspace{-2pt}

\begin{figure}[ht]
      \centering
      \includegraphics[width=\halfwidth]{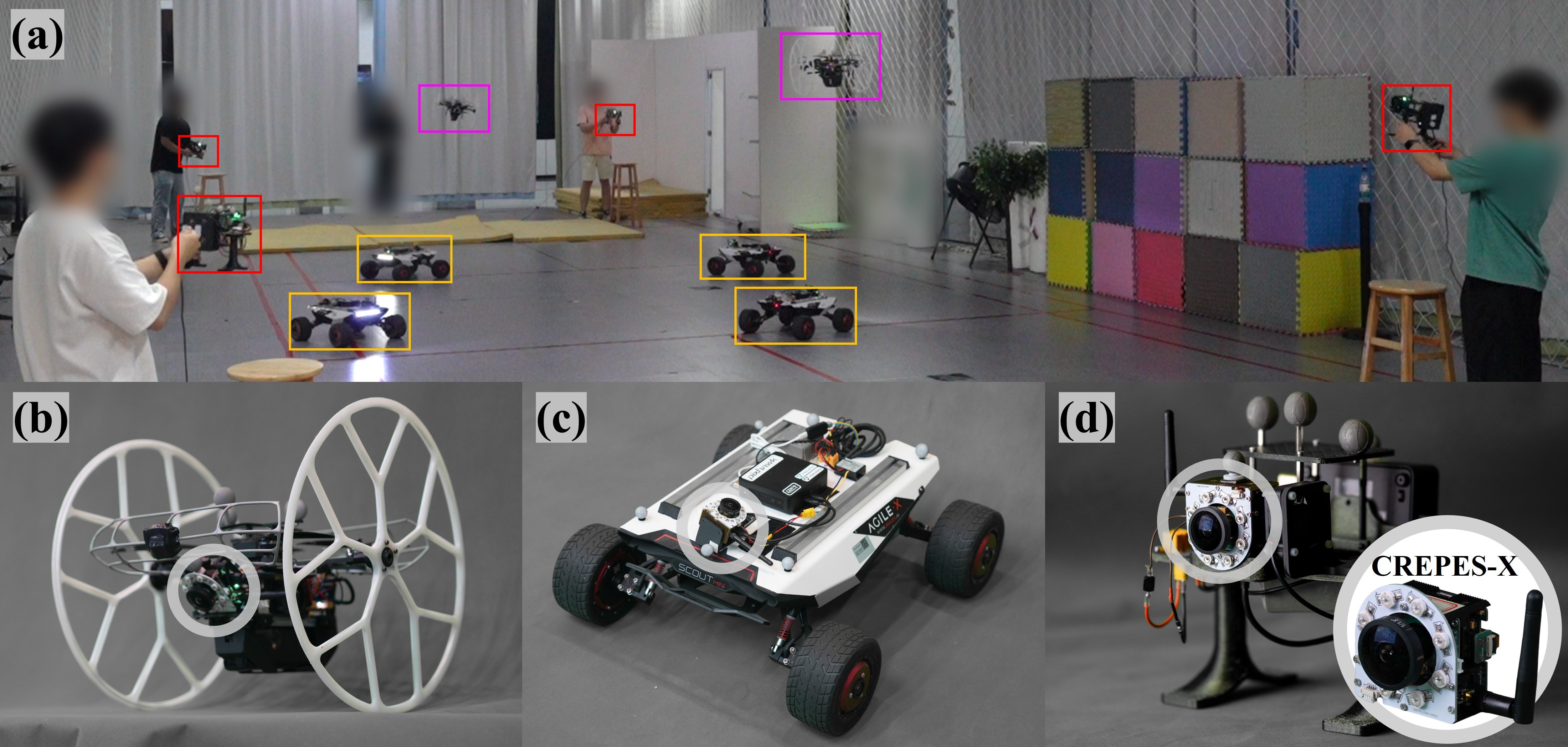}
      \caption{
            The ten-device experimental environment and CREPES-X platforms.
            (a) Ten-device experimental environment, with the handheld, UAV, and UGV platforms highlighted in red, purple, and yellow, respectively.
            (b) UAV-mounted platform.
            (c) UGV-mounted platform.
            (d) Handheld platform.
      }
      \label{Fig: Exp: Real-World Environment}
      \vspace{-2pt}
\end{figure}

All real-world experiments are conducted in a $10\times20\times3\,\mathrm{m}^3$ motion-capture environment using CREPES-X~\cite{li2025crepes} devices deployed on handheld, UAV, and UGV platforms, as illustrated in \autoref{Fig: Exp: Real-World Environment}.
The IMU and UWB measurements incur negligible front-end processing, whereas bearing observations are produced by a vision front end~\cite{li2025crepes} that detects infrared markers, decodes robot IDs, and outputs calibrated normalized 3-D bearing vectors.
Bearings are timestamped approximately $0.6\,\mathrm{ms}$ after the camera trigger, and complete bearing-ID observations become available within approximately $7.5\,\mathrm{ms}$ in our timing measurements.
If no valid detection is obtained due to occlusion or insufficient observations, no bearing factor is added at that timestamp.
Sensor streams are transmitted over TCP, and \CT{} handles arrival skew through asynchronous per-robot spline extension and virtual extension, while delayed measurements can still be incorporated at their acquisition timestamps if those timestamps remain within the active window.
Longer communication stalls effectively act as temporary measurement outages, and controlled robustness to network delay and dropout is beyond the scope of this work.
The real-world evaluation covers six scenarios: smooth five-device motion under LOS/NLOS conditions (\texttt{\textbf{LOS}} and \texttt{\textbf{NLOS}} scenarios); Shift-speed Dynamic Motion (\texttt{\textbf{SDM}}), where one device moves with time-varying velocity while the others remain stationary; High-speed Dynamic Motion (\texttt{\textbf{HDM}}) involving aggressive relative motion; TEN-device heterogeneous (\texttt{\textbf{TEN}}) experiments with two UAVs, four UGVs, and four handheld devices; and Time-offset UnSynchronized (\texttt{\textbf{TUS}}) experiments, where UWB time synchronization used in other scenarios is disabled and the devices rely only on Wi-Fi synchronization, whose accuracy is around tens of milliseconds~\cite{chen2023understanding}.
Wired synchronization provides sub-millisecond time-offset ground truth.
Ground-truth poses are provided by a motion-capture system, and sensor extrinsics are calibrated offline following~\cite{li2025crepes}, summarized in \autoref{Tab: Experiment Configuration}.

\begin{table}[h]
      \centering
      \caption{Experiment Configuration}
      \label{Tab: Experiment Configuration}
      \resizebox{\halfwidth}{!}{
            \begin{tabular}{lccccccc}
                  \toprule
                  \multirow{2}{*}{Metric} & \multirow{2}{*}{Freq.} & \multicolumn{6}{c}{RMSE in different scenarios}                                                                                                                                            \\
                                          &                        & \texttt{\textbf{LOS}}                           & \texttt{\textbf{NLOS}}    & \texttt{\textbf{SDM}}     & \texttt{\textbf{HDM}}     & \texttt{\textbf{TUS}}     & \texttt{\textbf{TEN}}    \\
                  \midrule
                  Bearing                 & $50\,\mathrm{Hz}$      & $3.661^\circ$                                   & $3.263^\circ$             & $5.145^\circ$             & $2.892^\circ$             & $1.571^\circ$             & $2.951^\circ$            \\
                  Distance                & $100\,\mathrm{Hz}$     & $0.055\,\mathrm{m}$                             & $0.057\,\mathrm{m}$       & $0.046\,\mathrm{m}$       & $0.107\,\mathrm{m}$       & $0.129\,\mathrm{m}$       & $0.302\,\mathrm{m}$      \\
                  Acc.                    & $100\,\mathrm{Hz}$     & $1.946\,\mathrm{m/s^2}$                         & $2.062\,\mathrm{m/s^2}$   & $1.967\,\mathrm{m/s^2}$   & $0.895\,\mathrm{m/s^2}$   & $1.162\,\mathrm{m/s^2}$   & $5.009\,\mathrm{m/s^2}$  \\
                  Gyro.                   & $100\,\mathrm{Hz}$     & $12.226^\circ/\mathrm{s}$                       & $12.492^\circ/\mathrm{s}$ & $16.235^\circ/\mathrm{s}$ & $10.353^\circ/\mathrm{s}$ & $15.508^\circ/\mathrm{s}$ & $9.986^\circ/\mathrm{s}$ \\
                  \bottomrule
            \end{tabular}
      }
\end{table}

The following Secs.~VII-A--E evaluate discrete-time and continuous-time baselines, closed-loop control, online time-offset estimation, and computational scheduling, respectively.

\subsection{Comparison with Discrete-Time Baselines}
\label{Sec: Comparison of Continuous-time and Discrete-time}

We compare the pose estimation accuracy of \CT{} against two discrete-time baselines on all experiments except \texttt{\textbf{TUS}}: CREPES~\cite{xun2023crepes}, whose pipeline combines pairwise ESKFs with a PGO back end, and CREPES-X-MFTO~\cite{li2025crepes}.
For \CT{}, similar to \autoref{Sec: Accuracy Analysis}, we evaluate CT-BCD under the IA-BCD-only strategy (CT-BCD\textsubscript{(I)}), and both CT-BCD and CT-Full under the Hybrid strategy (CT-BCD\textsubscript{(H)} and CT-Full\textsubscript{(H)}).
The Full-Batch-Only strategy is omitted, as benchmark results indicate only marginal performance differences compared to the Hybrid strategy.
The time window for CREPES-X and CT-Full is set to $3\,\mathrm{s}$, and for CT-BCD is set to $1\,\mathrm{s}$.
For \texttt{\textbf{SDM}} and \texttt{\textbf{HDM}}, RMSE is computed only for the moving device (Dev.~1).

\subsubsection{Latest Output}

\begin{table}[t]
    \centering
    \caption{
        Real-World RMSE of CREPES, CREPES-X, and CT (Latest Output, Color Definition: \cfirst{best}, \csecond{second-best}, determined by position RMSE, and by rotation RMSE if equal)
    }
    \label{Tab: Real-World Experiment Statistics and RMSE of DT and CT (Real-time Output)}
    \resizebox{\halfwidth}{!}{
        \begin{tabular}{
                w{c}{0.1cm}
                w{l}{1.1cm}
                *{5}{w{c}{1.36cm}}
                w{c}{2.3cm}
            }
            \toprule
            \multicolumn{2}{c}{\multirow{3}{*}{\vspace{-5pt}Exp.}}
                             &
            \multirow{1}{*}{$E_{\text{CREPES}}$}
                             &
            \multirow{1}{*}{$E_{\text{CREPES-X}}$}
                             &
            \multirow{1}{*}{\begin{tabular}[c]{@{}c@{}}$E_{\text{CT-BCD}_\text{(I)}}$\end{tabular}}
                             &
            \multirow{1}{*}{\begin{tabular}[c]{@{}c@{}}$E_{\text{CT-BCD}_\text{(H)}}$\end{tabular}}
                             &
            \multirow{1}{*}{\begin{tabular}[c]{@{}c@{}}$E_{\text{CT-Full}_\text{(H)}}$\end{tabular}}
                             &
            \multirow{2}{*}{%
                \vspace{-4pt}\resizebox{2.8cm}{!}{%
                    $\displaystyle\left|\frac{E_{\text{CREPES-X}}-E_{\text{CT-Full}}}{E_{\text{CREPES-X}}}\right|_{\%}$%
                }%
            }
            \\
            \cmidrule(r){3-3} \cmidrule(r){4-4} \cmidrule(r){5-5} \cmidrule(r){6-6} \cmidrule(r){7-7}
            \multicolumn{2}{c}{} & Pos / Rot
                             & Pos / Rot
                             & Pos / Rot
                             & Pos / Rot
                             & Pos / Rot
                             &
            \\
            \multicolumn{2}{c}{} & ($\mathrm{cm}$) / ($\degree$)
                             & ($\mathrm{cm}$) / ($\degree$)
                             & ($\mathrm{cm}$) / ($\degree$)
                             & ($\mathrm{cm}$) / ($\degree$)
                             & ($\mathrm{cm}$) / ($\degree$)
                             & Pos / Rot
            \\
            \midrule
            \multirow{10}{*}{\vspace{-10pt}\rotatebox[origin=c]{90}{\strut\textbf{Smooth}}} & \texttt{LOS\_1} & {6.4 / 2.1} & {6.3 / 1.9} & {6.0 / 1.9} & \csecond{5.7 / 1.9} & \cfirst{5.7 / 1.9} & $10\% \RED{\uparrow}$\hspace{0.5em}$~0\% \RED{\uparrow}$ \\
             & \texttt{LOS\_2} & {8.6 / 2.4} & {7.8 / 2.1} & {7.4 / 2.0} & \csecond{6.7 / 2.1} & \cfirst{6.7 / 2.0} & $14\% \RED{\uparrow}$\hspace{0.5em}$~3\% \RED{\uparrow}$ \\
             & \texttt{LOS\_3} & {8.5 / 2.8} & {6.9 / 2.1} & {6.3 / 2.0} & \csecond{6.1 / 2.1} & \cfirst{6.1 / 2.1} & $12\% \RED{\uparrow}$\hspace{0.5em}$~1\% \RED{\uparrow}$ \\
             & \texttt{LOS\_4} & {7.2 / 2.1} & {8.0 / 2.5} & {7.2 / 2.5} & \csecond{7.1 / 2.5} & \cfirst{7.0 / 2.5} & $12\% \RED{\uparrow}$\hspace{0.5em}$~3\% \RED{\uparrow}$ \\
             & \textbf{MEAN} & {7.7 / 2.4} & {7.3 / 2.2} & {6.7 / 2.1} & \csecond{6.4 / 2.1} & \cfirst{6.4 / 2.1} & $12\% \RED{\uparrow}$\hspace{0.5em}$~2\% \RED{\uparrow}$ \\
            \cmidrule(l){2-8}
             & \texttt{NLOS\_1} & {6.5 / 2.2} & {5.0 / 1.8} & {7.4 / 1.8} & \csecond{4.8 / 1.8} & \cfirst{4.6 / 1.8} & $~8\% \RED{\uparrow}$\hspace{0.5em}$~0\% \RED{\uparrow}$ \\
             & \texttt{NLOS\_2} & {8.4 / 2.4} & {6.7 / 2.0} & {11.1 / 2.0} & \csecond{5.9 / 2.0} & \cfirst{5.6 / 2.0} & $16\% \RED{\uparrow}$\hspace{0.5em}$~1\% \RED{\uparrow}$ \\
             & \texttt{NLOS\_3} & {6.8 / 2.3} & {5.9 / 2.0} & {9.3 / 2.0} & \csecond{5.5 / 2.0} & \cfirst{5.4 / 2.0} & $~8\% \RED{\uparrow}$\hspace{0.5em}$~0\% \RED{\uparrow}$ \\
             & \texttt{NLOS\_4} & {8.0 / 2.5} & {7.0 / 2.1} & {7.3 / 2.1} & \csecond{6.4 / 2.1} & \cfirst{6.1 / 2.0} & $13\% \RED{\uparrow}$\hspace{0.5em}$~2\% \RED{\uparrow}$ \\
             & \textbf{MEAN} & {7.4 / 2.3} & {6.2 / 2.0} & {8.8 / 2.0} & \csecond{5.7 / 2.0} & \cfirst{5.4 / 2.0} & $12\% \RED{\uparrow}$\hspace{0.5em}$~1\% \RED{\uparrow}$ \\
            \midrule
            \multirow{5}{*}{\vspace{-5pt}\rotatebox[origin=c]{90}{\strut\textbf{Shift-Speed}}} & \texttt{SDM\_1} & {7.4 / 6.5} & {6.7 / 2.5} & {6.0 / 2.3} & \csecond{5.1 / 2.3} & \cfirst{4.6 / 2.4} & $31\% \RED{\uparrow}$\hspace{0.5em}$~7\% \RED{\uparrow}$ \\
             & \texttt{SDM\_2} & {6.3 / 3.6} & {5.3 / 1.8} & {4.9 / 1.9} & \csecond{4.2 / 1.8} & \cfirst{3.9 / 1.8} & $26\% \RED{\uparrow}$\hspace{0.5em}$~0\% \RED{\uparrow}$ \\
             & \texttt{SDM\_3} & {11.0 / 4.9} & {6.8 / 2.7} & {5.2 / 2.5} & \csecond{4.5 / 2.5} & \cfirst{4.1 / 2.4} & $40\% \RED{\uparrow}$\hspace{0.5em}$10\% \RED{\uparrow}$ \\
             & \texttt{SDM\_4} & {6.3 / 2.6} & {6.4 / 2.1} & {5.0 / 2.2} & \csecond{4.1 / 1.9} & \cfirst{3.8 / 2.0} & $41\% \RED{\uparrow}$\hspace{0.5em}$~6\% \RED{\uparrow}$ \\
             & \textbf{MEAN} & {7.8 / 4.4} & {6.3 / 2.3} & {5.3 / 2.2} & \csecond{4.5 / 2.1} & \cfirst{4.1 / 2.2} & $35\% \RED{\uparrow}$\hspace{0.5em}$~6\% \RED{\uparrow}$ \\
            \midrule
            \multirow{5}{*}{\vspace{-5pt}\rotatebox[origin=c]{90}{\strut\textbf{High-speed}}} & \texttt{HDM\_1} & {26.0 / 5.5} & {23.9 / 2.7} & {27.7 / 3.1} & \csecond{19.6 / 2.8} & \cfirst{17.6 / 2.7} & $26\% \RED{\uparrow}$\hspace{0.5em}$~1\% \RED{\uparrow}$ \\
             & \texttt{HDM\_2} & {20.2 / 3.4} & {17.1 / 2.1} & {19.3 / 2.6} & \csecond{13.6 / 2.2} & \cfirst{11.8 / 2.1} & $31\% \RED{\uparrow}$\hspace{0.5em}$~1\% \RED{\uparrow}$ \\
             & \texttt{HDM\_3} & {43.7 / 3.9} & {20.5 / 2.6} & {14.5 / 2.7} & \csecond{10.5 / 2.2} & \cfirst{8.9 / 1.9} & $57\% \RED{\uparrow}$\hspace{0.5em}$28\% \RED{\uparrow}$ \\
             & \texttt{HDM\_4} & {22.7 / 6.6} & {18.2 / 2.6} & {19.6 / 2.6} & \csecond{14.5 / 2.4} & \cfirst{12.4 / 2.2} & $32\% \RED{\uparrow}$\hspace{0.5em}$16\% \RED{\uparrow}$ \\
             & \textbf{MEAN} & {28.2 / 4.9} & {19.9 / 2.5} & {20.3 / 2.7} & \csecond{14.5 / 2.4} & \cfirst{12.7 / 2.2} & $36\% \RED{\uparrow}$\hspace{0.5em}$12\% \RED{\uparrow}$ \\
            \midrule
            \multirow{5}{*}{\vspace{-5pt}\rotatebox[origin=c]{90}{\strut\textbf{Ten-device}}} & \texttt{TEN\_1} & {13.7 / 2.7} & {12.5 / 2.2} & {6.6 / 2.0} & \csecond{6.4 / 2.0} & \cfirst{5.1 / 2.0} & $59\% \RED{\uparrow}$\hspace{0.5em}$10\% \RED{\uparrow}$ \\
             & \texttt{TEN\_2} & {14.4 / 2.5} & {11.9 / 2.1} & {6.6 / 2.0} & \csecond{6.4 / 2.0} & \cfirst{5.8 / 2.0} & $51\% \RED{\uparrow}$\hspace{0.5em}$~4\% \RED{\uparrow}$ \\
             & \texttt{TEN\_3} & {65.0 / 8.5} & {12.9 / 2.7} & {6.4 / 2.7} & \csecond{6.3 / 2.7} & \cfirst{5.6 / 2.7} & $57\% \RED{\uparrow}$\hspace{0.5em}$~1\% \RED{\uparrow}$ \\
             & \texttt{TEN\_4} & {58.5 / 8.0} & {12.5 / 2.7} & {6.7 / 2.5} & \csecond{6.6 / 2.5} & \cfirst{5.8 / 2.2} & $54\% \RED{\uparrow}$\hspace{0.5em}$20\% \RED{\uparrow}$ \\
             & \textbf{MEAN} & {37.9 / 5.4} & {12.4 / 2.4} & {6.6 / 2.3} & \csecond{6.4 / 2.3} & \cfirst{5.6 / 2.2} & $55\% \RED{\uparrow}$\hspace{0.5em}$~9\% \RED{\uparrow}$ \\
            \bottomrule
        \end{tabular}
    }
\end{table}

Results are summarized in \autoref{Tab: Real-World Experiment Statistics and RMSE of DT and CT (Real-time Output)}.
\CT{} with the Hybrid strategy consistently outperforms both discrete-time baselines across all experiments.
Compared with CREPES and CREPES-X, CT-Full\textsubscript{(H)} achieves overall position/rotation improvements of $62\%/45\%$ and $34\%/6\%$, respectively.
The advantage of \CT{} becomes more pronounced as motion intensity increases.
Compared with the best DT baseline, CREPES-X, CT-Full\textsubscript{(H)} achieves modest position/rotation improvements of $12\%/1\%$ in smooth-motion scenarios, which increase to $35\%/6\%$ in \texttt{\textbf{SDM}} and $36\%/12\%$ in \texttt{\textbf{HDM}}.
As shown in \autoref{Fig: Exp: Real-World HDM Results}, \CT{} achieves both lower overall errors and smaller maximum errors than CREPES-X on \texttt{\textbf{HDM}} sequences.
On the \texttt{\textbf{TEN}} sequences covering multiple platforms, CT-Full\textsubscript{(H)} further improves position/rotation accuracy by $55\%/9\%$.
As shown in \autoref{Fig: Exp: Real-World TEN Results}, the largest improvement is observed on UAVs while maintaining low overall errors across platforms.
These results demonstrate the advantage of \CT{} under fast motion, highly asynchronous observations, and high-vibration platforms.

Across different \CT{} optimization strategies, consistent performance trends are observed among CT-BCD\textsubscript{(I)}, CT-BCD\textsubscript{(H)}, and CT-Full\textsubscript{(H)}, with estimation errors decreasing in that order.
CT-BCD\textsubscript{(H)} achieves comparable performance to CT-Full\textsubscript{(H)} in smooth and shift-speed motion scenarios, and only slightly lags behind in high-speed motion scenarios.
Meanwhile, CT-BCD\textsubscript{(I)} exhibits a more pronounced performance gap relative to CT-Full\textsubscript{(H)}.
In well-conditioned \texttt{\textbf{LOS}} and \texttt{\textbf{SDM}} scenarios, as well as the observation-rich \texttt{\textbf{TEN}} scenario, CT-BCD\textsubscript{(I)} still outperforms CREPES-X.
However, in challenging \texttt{\textbf{NLOS}} and \texttt{\textbf{HDM}} scenarios, its performance degrades significantly, and can even fall below that of CREPES-X.
This behavior indicates that instant BCD updates are sensitive to observation quality, whereas the Hybrid strategy effectively mitigates this limitation by incorporating full-batch updates.
Noting that CT-BCD\textsubscript{(H)} can output at IMU frequency, these results highlight the advantages of the Hybrid strategy in delivering high-accuracy, low-latency, and high-frequency outputs.

\begin{figure*}[tb]
    
    \centering
    \begin{minipage}{\fullwidth}
        \centering
        \includegraphics[width=\textwidth]{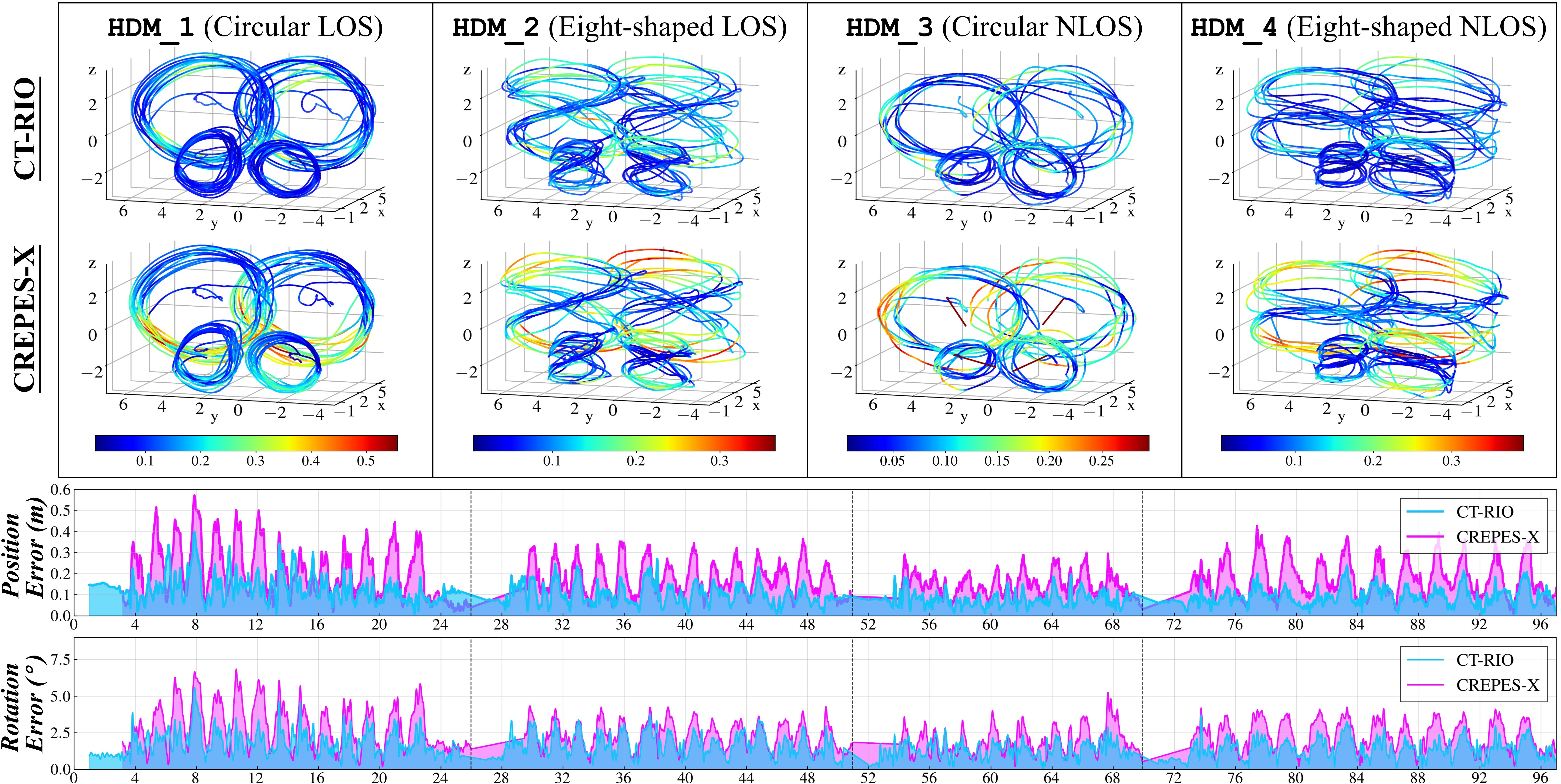}
        \caption{
            High-speed dynamic motion (\texttt{\textbf{HDM}}) results.
            Upper: The estimated outputs in reference frame (Device 0) by CT-Full\textsubscript{(H)} and CREPES-X, respectively.
            Lower: Position and rotation errors of Device 1.
        }
        \label{Fig: Exp: Real-World HDM Results}
    \end{minipage}

    \vspace{6pt}
    \begin{minipage}{\fullwidth}
        \centering
        \includegraphics[width=\textwidth]{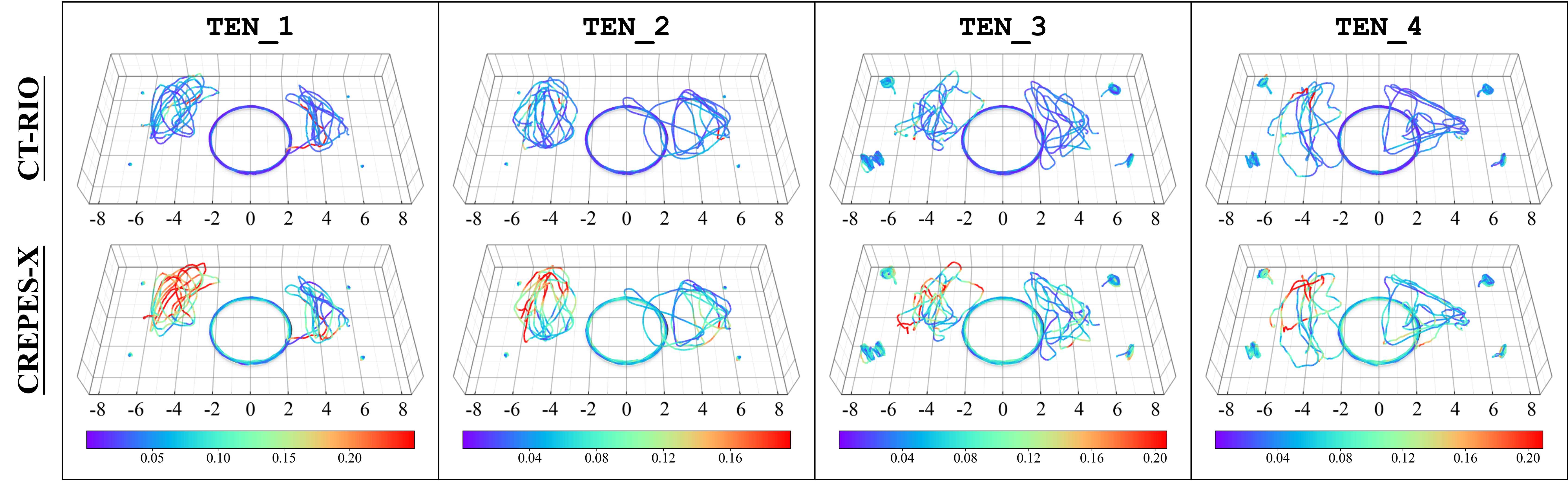}
        \caption{
            Trajectory comparison of CT-Full\textsubscript{(H)} and CREPES-X on the four ten-device \texttt{\textbf{TEN}} sequences.
            The estimated outputs are transformed into the world frame.
            The circular paths correspond to the UGVs, the irregular paths on both sides to the UAVs, and the paths near the four corners to the handheld devices.
        }
        \label{Fig: Exp: Real-World TEN Results}
    \end{minipage}
\end{figure*}

\begin{table}[t]
    \centering
    \caption{
        Mean RMSE per Scenario (Complete Trajectory, Color Definition: \cfirst{best}, \csecond{second-best}, determined by position RMSE, and by rotation RMSE if equal)
        }
    \label{Tab: RMSE of DT and CT in Real-World Experiment (Complete Trajectory)}
    \resizebox{\halfwidth}{!}{%
        \begin{tabular}{lccccw{c}{2.3cm}}
            \toprule
            \multirow{3}{*}{\vspace{-5pt}Exp.}
                          &
            $E_{\text{CREPES}}$
                          &
            $E_{\text{CREPES-X}}$
                          &
            $E_{\text{CT-BCD}_\text{(I)}}$
                          &
            $E_{\text{CT-Full}_\text{(H)}}$
                          &
            \multirow{2}{*}{%
                \vspace{-4pt}\resizebox{2.8cm}{!}{%
                    $\displaystyle\left|\frac{E_{\text{CREPES-X}}-E_{\text{CT-Full}}}{E_{\text{CREPES-X}}}\right|_{\%}$%
                }%
            }
            \\
            \cmidrule(r){2-2} \cmidrule(r){3-3} \cmidrule(r){4-4} \cmidrule(r){5-5}
                          &
            Pos / Rot
                          &
            Pos / Rot
                          &
            Pos / Rot
                          &
            Pos / Rot
                          &
            \\
                          &
            ($\mathrm{cm}$) / ($\degree$)
                          &
            ($\mathrm{cm}$) / ($\degree$)
                          &
            ($\mathrm{cm}$) / ($\degree$)
                          &
            ($\mathrm{cm}$) / ($\degree$)
                          &
            Pos / Rot
            \\
            \midrule
            \textbf{LOS} & {7.7 / 2.4} & {7.1 / 2.2} & \csecond{6.1 / 2.1} & \cfirst{6.0 / 2.1} & $16\% \RED{\uparrow}$\hspace{0.5em}$~5\% \RED{\uparrow}$ \\
            \textbf{NLOS} & {7.4 / 2.3} & {6.1 / 2.0} & \csecond{5.5 / 2.0} & \cfirst{5.3 / 2.0} & $13\% \RED{\uparrow}$\hspace{0.5em}$~1\% \RED{\uparrow}$ \\
            \textbf{SDM} & {7.7 / 4.4} & {8.0 / 2.9} & \csecond{3.4 / 2.1} & \cfirst{3.3 / 1.9} & $59\% \RED{\uparrow}$\hspace{0.5em}$32\% \RED{\uparrow}$ \\
            \textbf{HDM} & {28.2 / 4.9} & {19.3 / 2.5} & \csecond{14.4 / 2.4} & \cfirst{10.3 / 2.1} & $47\% \RED{\uparrow}$\hspace{0.5em}$18\% \RED{\uparrow}$ \\
            \textbf{TEN} & {37.9 / 5.4} & {12.4 / 2.4} & \csecond{5.6 / 2.3} & \cfirst{5.5 / 2.3} & $56\% \RED{\uparrow}$\hspace{0.5em}$~7\% \RED{\uparrow}$ \\
            \bottomrule
        \end{tabular}%
    }
    \vspace{5pt}
\end{table}

\subsubsection{Complete Trajectory}

Since CT-BCD\textsubscript{(H)} and CT-Full\textsubscript{(H)} share the same complete trajectory, only CT-Full\textsubscript{(H)} is reported in \autoref{Tab: RMSE of DT and CT in Real-World Experiment (Complete Trajectory)}.
On complete trajectories, CT-Full\textsubscript{(H)} consistently achieves lower position RMSE than its latest-output estimates, especially in \texttt{\textbf{SDM}} and \texttt{\textbf{HDM}}.
In contrast, CREPES-X shows little improvement and even degrades in \texttt{\textbf{SDM}}.
Moreover, after sufficient updates, CT-BCD\textsubscript{(I)} approaches CT-Full\textsubscript{(H)} in position accuracy except in \texttt{\textbf{HDM}}, confirming the value of full-batch refinement under aggressive motion.

\vspace{-5pt}
\subsection{Comparison with Continuous-Time Baselines}
\label{Sec: Comparison with Continuous-Time Baselines}

\subsubsection{Comparison with Uniform Knot Frequency}
\label{Subsec: Comparison with Uniform Knot Frequency}

We evaluate the necessity of non-uniform knot insertion on the real-world scenarios.
\autoref{Tab: RMSE of Non-Uniform} reports the RMSE and optimization time of B-splines under uniform knot spacing at frequency \textit{f} Hz (\textit{Uni-f}) and under the proposed relative-motion-intensity-aware non-uniform scheme (\textit{Non-Uniform}).

As shown in \autoref{Tab: RMSE of Non-Uniform}, with uniform knot frequency, increasing knot density improves trajectory accuracy at the expense of higher computation.
The densest setting, \textit{Uni-50}, yields the best accuracy across all scenarios, but \textit{Uni-20} achieves comparable performance in \texttt{\textbf{LOS}}, \texttt{\textbf{HDM}}, and \texttt{\textbf{TEN}} scenarios with much lower cost, showing that extreme density is often unnecessary.
Overly sparse distributions struggle under fast motion: while \textit{Uni-5} performs well in mild scenarios (\texttt{\textbf{LOS}}, \texttt{\textbf{NLOS}}, and \texttt{\textbf{TEN}}), it degrades significantly for dynamic motion (\texttt{\textbf{SDM}} and \texttt{\textbf{HDM}}).
In many cases, \textit{Uni-10} maintains high computational efficiency without sacrificing too much accuracy, consistent with prior findings reported in \cite{cioffi2022continuous}.
Meanwhile, \textit{Non-Uniform} achieves the best or second-best accuracy across all scenarios, while notably improving efficiency compared with the densest uniform knot configuration, which also yields high accuracy.
Specifically, it reduces computation time by 13\% to 69\%.

\begin{table}[t]
    \centering
    \caption{
        Position RMSE and Optimization Time of Non-Uniform and Uniform B-splines with Different Knot Frequency (Complete Trajectory, Color Definition: \cfirst{best}, \csecond{second-best}, determined by position RMSE, and by optimization time if equal)
    }
    \vspace{-5pt}
    \label{Tab: RMSE of Non-Uniform}
    \resizebox{\halfwidth}{!}{
        \begin{tabular}{lccccccc}
            \toprule
            \multirow{3}{*}{\vspace{-5pt}Exp.}
                          &
            \textit{Uni-50}
                          &
            \textit{Uni-20}
                          &
            \textit{Uni-10}
                          &
            \textit{Uni-5}
                          &
            \multicolumn{2}{c}{\textit{Non-Uniform}}
            \\
            \cmidrule(r){2-2}\cmidrule(r){3-3}\cmidrule(r){4-4}\cmidrule(r){5-5}\cmidrule(r){6-7}
                          &
            Pos / Time
                          &
            Pos / Time
                          &
            Pos / Time
                          &
            Pos / Time
                          &
            Pos / Time
                          &
            Freq.
            \\
                          &
            ($\mathrm{cm}$) / ($\mathrm{ms}$)
                          &
            ($\mathrm{cm}$) / ($\mathrm{ms}$)
                          &
            ($\mathrm{cm}$) / ($\mathrm{ms}$)
                          &
            ($\mathrm{cm}$) / ($\mathrm{ms}$)
                          &
            ($\mathrm{cm}$) / ($\mathrm{ms}$)
                          &
            ($\mathrm{Hz}$)
            \\
            \midrule
            \textbf{LOS}  & {6.1 / 120}         & \csecond{6.1 / 115}  & {6.2 / 107}  & {6.6 / ~99}  & \cfirst{6.0 / 105}  & 12.8 \\
            \textbf{NLOS} & \cfirst{5.1 / 130}  & {5.3 / 100}          & {5.5 / ~91}  & {6.2 / ~83}  & \csecond{5.3 / ~97} & 10.9 \\
            \textbf{SDM}  & \csecond{3.5 / 102} & {3.8 / ~81}          & {4.0 / ~79}  & {10.1 / ~81} & \cfirst{3.3 / ~73}  & 18.5 \\
            \textbf{HDM}  & {10.4 / 155}        & \csecond{10.4 / ~96} & {10.5 / ~90} & {16.1 / ~93} & \cfirst{10.3 / ~95} & 30.3 \\
            \textbf{TEN}  & {5.5 / 1469}        & \csecond{5.5 / 624}  & {5.7 / 462}  & {5.9 / 402}  & \cfirst{5.5 / 459}  & ~7.8 \\
            \bottomrule
        \end{tabular}
    }
    \vspace{3pt}
\end{table}

\autoref{Fig: Non Uniform Results} further shows that adaptive knot placement concentrates
modeling capacity in high-motion regions and sparsifies
low-motion segments, thereby suppressing the large error
fluctuations observed with fixed-frequency splines.
Overall, \textit{Non-Uniform} adapts to motion complexity, reduces errors, and improves computational efficiency compared with fixed-frequency baselines.

\begin{figure}[tb]
    
    \centering
    \includegraphics[width=.97\linewidth]{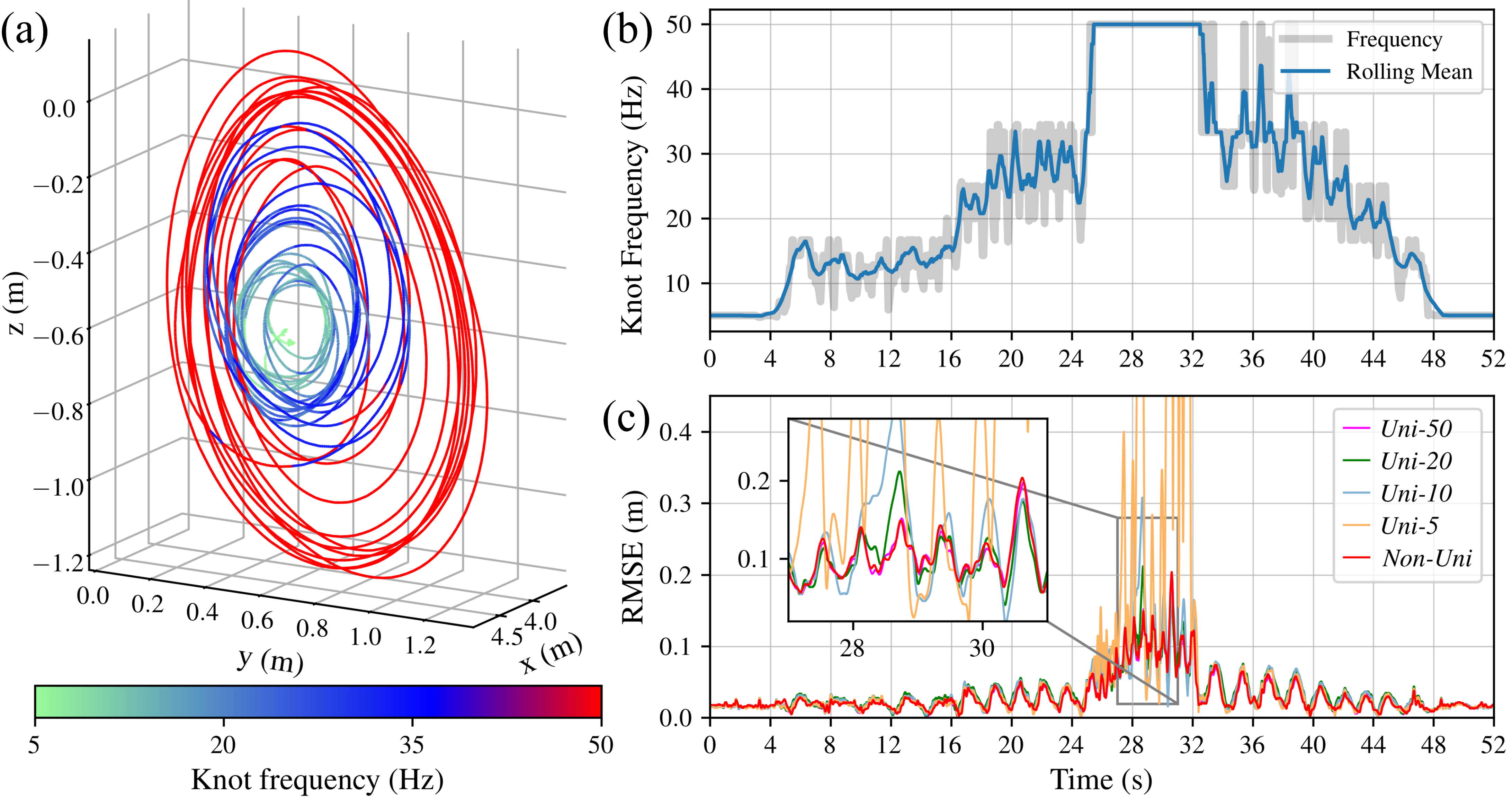}
    \caption{
        Results of CT-Full on \texttt{SDM\_1}.
        (a) Trajectory of Device 1, colored by knot frequency.
        (b) Knot frequency along the full trajectory of the non-uniform B-spline.
        (c) Position RMSE of non-uniform and uniform B-splines with different knot frequencies.
    }
    \label{Fig: Non Uniform Results}
\end{figure}

\subsubsection{Boundary-Update Ablation}
\label{Subsec: Boundary-Update Ablation}

We further evaluate whether exact shape preservation is necessary for online C-NUBS manipulation by comparing naive boundary updates, a local tail-refit strategy, and the proposed analytic operators.
For the tail-refit baseline, after naive extension, we sample target poses from the original trajectory over its last $k-1$ intervals and reoptimize only the last $k-2$ existing control points.
For the order-$4$ spline used in our system, the tail-refit baseline uses four internal Gauss-Legendre samples per interval together with the interval boundaries, and solves the resulting problem using Ceres with \texttt{DENSE\_QR}.

In 5,000 controlled order-$4$ spline-extension trials, naive extension reshapes the previously represented trajectory, with maximum position and rotation deviations of $0.534~\mathrm{m}$ and $1.12\times10^{-2}~\mathrm{rad}$, respectively.
Tail refitting preserves the curve with deviations below $10^{-10}$, while the analytic operation preserves it exactly by construction.
However, the analytic operation requires only $0.94~\mu\mathrm{s}$ on average, compared with $150.84~\mu\mathrm{s}$ for tail refitting, requiring approximately $1/161$ of the computation time.

We further evaluate how update methods affect the accuracy of CT-Full and CT-BCD outputs.
For this experiment, all analytic operations are replaced by the corresponding naive or tail-refit boundary updates.
\autoref{Tab: RMSE of Boundary-Update Ablation} shows that replacing the analytic operation with naive boundary updates increases the CT-BCD position RMSE by approximately $19\%$--$54\%$ across the five scenarios, while tail refitting achieves essentially identical accuracy to the analytic operation.
CT-Full is generally less sensitive to naive boundary updates, showing minor changes in most scenarios.
However, on the ten-device \texttt{\textbf{TEN}} sequences, tail refitting additionally requires $1.543~\mathrm{ms}$ per CT-BCD update and accounts for $25.0\%$ of the total update time.
These results indicate that shape preservation avoids unintended perturbations of the retained trajectory, while the proposed closed-form operators provide exact preservation with negligible deterministic overhead compared with optimization-based tail refitting.

\begin{table}[t]
    \centering
    \caption{
        Mean RMSE per Scenario for Different Boundary-Update Methods (Latest-Output).
    }
    \vspace{-5pt}
    \label{Tab: RMSE of Boundary-Update Ablation}
    \resizebox{0.95\halfwidth}{!}{%
        \begin{tabular}{lcccccc}
            \toprule
            \multirow{4}{*}{\vspace{-5pt}Exp.}
             & \multicolumn{3}{c}{CT-BCD}
             & \multicolumn{3}{c}{CT-Full} \\
            \cmidrule(lr){2-4}\cmidrule(lr){5-7}
             & Naive
             & Analytic
             & Tail-Refit
             & Naive
             & Analytic
             & Tail-Refit \\
             & Pos / Rot
             & Pos / Rot
             & Pos / Rot
             & Pos / Rot
             & Pos / Rot
             & Pos / Rot \\
             & ($\mathrm{cm}$) / ($\degree$)
             & ($\mathrm{cm}$) / ($\degree$)
             & ($\mathrm{cm}$) / ($\degree$)
             & ($\mathrm{cm}$) / ($\degree$)
             & ($\mathrm{cm}$) / ($\degree$)
             & ($\mathrm{cm}$) / ($\degree$) \\
            \midrule
            \textbf{LOS}  & 7.8 / 2.1  & 6.4 / 2.1  & 6.4 / 2.1  & 6.4 / 2.1  & 6.4 / 2.1  & 6.4 / 2.1  \\
            \textbf{NLOS} & 6.8 / 2.0  & 5.7 / 2.0  & 5.7 / 2.0  & 5.5 / 2.0  & 5.4 / 2.0  & 5.4 / 2.0  \\
            \textbf{SDM}  & 6.2 / 3.6  & 4.5 / 2.1  & 4.5 / 2.1  & 4.9 / 3.6  & 4.1 / 2.2  & 4.1 / 2.2  \\
            \textbf{HDM}  & 17.3 / 2.6 & 14.5 / 2.4 & 14.5 / 2.4 & 12.2 / 2.4 & 12.7 / 2.2 & 12.7 / 2.2 \\
            \textbf{TEN}  & 9.9 / 2.3  & 6.4 / 2.3  & 6.4 / 2.3  & 5.6 / 2.2  & 5.6 / 2.2  & 5.6 / 2.2  \\
            \bottomrule
        \end{tabular}%
    }
    \vspace{3pt}
\end{table}

\subsubsection{Comparison with Propagation Baselines}
\label{Subsec: Comparison with Propagation Baselines}

We further compare \CT{} against the commonly used strategy of propagating the optimized state to the current IMU timestamp using \eqref{Equ: Relative Kinematics}.
For conciseness, CT-BCD and CT-Full hereafter denote the corresponding IA-BCD and full-batch outputs under the Hybrid strategy.
We additionally evaluate an unclamped-spline variant, CT-Uncl.
For both CT-Full and CT-Uncl, the suffix \textit{-P} denotes propagation of the latest optimized state to the current IMU timestamp.
CT-Full and CT-Uncl use the same non-uniform knot strategy and optimization objective.
However, without clamping, CT-Uncl incurs a query-time delay equivalent to the duration of the final two knot segments.

\begin{table}[t]
    \centering
    \caption{
        Mean RMSE per Scenario for CT-RIO and the Propagation Baselines.
        (Latest-Output, the last two columns report relative error changes to CT-BCD.)
    }
    \label{Tab: RMSE of Propagation Baselines}
    \newcommand{\propdeltapair}[2]{%
        \makebox[1.8em][r]{#1}\hspace{0.2em}%
        \makebox[1.8em][r]{#2}%
    }
    \resizebox{\halfwidth}{!}{%
        \begin{tabular}{lccccccc}
            \toprule
            \multirow{4}{*}{\vspace{-5pt}Exp.}
             & \multicolumn{2}{c}{Full-Batch}
             & \multicolumn{1}{c}{IA-BCD}
             & \multicolumn{2}{c}{IMU Propagation}
             & \multicolumn{2}{c}{Relative Error Change} \\
            \cmidrule(lr){2-3}\cmidrule(lr){4-4}\cmidrule(lr){5-6}\cmidrule(l){7-8}
             & $E_{\text{CT-Full}}$
             & $E_{\text{CT-Uncl}}$
             & $E_{\text{CT-BCD}}$
             & $E_{\text{CT-Full-P}}$
             & $E_{\text{CT-Uncl-P}}$
             & $\Delta_{\text{CT-Full-P}}$
             & $\Delta_{\text{CT-Uncl-P}}$ \\
             & Pos / Rot
             & Pos / Rot
             & Pos / Rot
             & Pos / Rot
             & Pos / Rot
             & \propdeltapair{Pos}{Rot}
             & \propdeltapair{Pos}{Rot} \\
             & ($\mathrm{cm}$) / ($\degree$)
             & ($\mathrm{cm}$) / ($\degree$)
             & ($\mathrm{cm}$) / ($\degree$)
             & ($\mathrm{cm}$) / ($\degree$)
             & ($\mathrm{cm}$) / ($\degree$)
             & \propdeltapair{($\%$)}{($\%$)}
             & \propdeltapair{($\%$)}{($\%$)} \\
            \midrule
            \textbf{LOS}
             & 6.4 / 2.1
             & 6.7 / 2.2
             & 6.4 / 2.1
             & 8.0 / 2.2
             & 8.0 / 2.3
             & \propdeltapair{+25}{+5}
             & \propdeltapair{+25}{+10} \\
            \textbf{NLOS}
             & 5.4 / 2.0
             & 5.6 / 2.0
             & 5.7 / 2.0
             & 7.1 / 2.0
             & 9.9 / 2.0
             & \propdeltapair{+25}{+0}
             & \propdeltapair{+74}{+0} \\
            \textbf{SDM}
             & 4.1 / 2.2
             & 4.4 / 2.3
             & 4.5 / 2.1
             & 6.0 / 2.3
             & 8.5 / 2.4
             & \propdeltapair{+33}{+10}
             & \propdeltapair{+89}{+14} \\
            \textbf{HDM}
             & 12.7 / 2.2
             & 13.3 / 2.2
             & 14.5 / 2.4
             & 18.0 / 2.6
             & 24.5 / 2.7
             & \propdeltapair{+24}{+8}
             & \propdeltapair{+69}{+13} \\
            \textbf{TEN}
             & 5.6 / 2.2
             & 5.6 / 2.5
             & 6.4 / 2.3
             & 7.9 / 2.3
             & 17.0 / 2.3
             & \propdeltapair{+23}{+0}
             & \propdeltapair{+166}{+0} \\
            \midrule
            \textbf{MEAN}
             & 6.8 / 2.1
             & 7.1 / 2.2
             & 7.5 / 2.2
             & 9.4 / 2.3
             & 13.6 / 2.3
             & \propdeltapair{+25}{+5}
             & \propdeltapair{+81}{+5} \\
            \bottomrule
        \end{tabular}%
    }
    \vspace{5pt}
\end{table}

As shown in \autoref{Tab: RMSE of Propagation Baselines}, CT-Full and CT-Uncl achieve comparable optimized-state accuracy, with mean position/rotation RMSEs of $6.8\,\mathrm{cm}/2.1\degree$ and $7.1\,\mathrm{cm}/2.2\degree$, respectively.
This indicates that clamping does not inherently affect spline-fitting accuracy.
For current-time outputs, however, CT-BCD achieves a lower mean position RMSE of $7.5\,\mathrm{cm}$, compared with $9.4\,\mathrm{cm}$ for CT-Full-P and $13.6\,\mathrm{cm}$ for CT-Uncl-P.
Comparing the two propagation baselines directly, CT-Full-P reduces the mean current-time position RMSE by approximately $31\%$ relative to CT-Uncl-P, isolating the benefit of removing the support-induced propagation interval.
Relative to CT-BCD, propagation increases the position RMSE by $23\%$-$33\%$ for CT-Full-P and $25\%$-$166\%$ for CT-Uncl-P across the five scenarios.
This trend is consistent with first-order robocentric error propagation, where a reference-rotation error $\delta\boldsymbol{\theta}_{s}$ induces
$\delta{}^{s}\mathbf{p}_{j,\mathrm{rot}}
    \simeq-[{}^{s}\mathbf{p}_{j}]_{\times}\delta\boldsymbol{\theta}_{s}$,
so the resulting position error scales with the inter-robot distance and accumulates over longer propagation intervals.
The larger degradation of CT-Uncl-P is consistent with its longer propagation interval due to the support-induced query-time delay.

These results clarify the role of clamping in \CT{}.
Its benefit lies not in enhancing the modeling capability of the spline, but in eliminating the support-induced propagation interval.
For high-rate outputs, fixed-lag optimization followed by IMU propagation remains a viable alternative when the propagation error is acceptable.
By contrast, clamping combined with IA-BCD enables the latest relative state to directly incorporate available inter-robot measurements rather than relying solely on open-loop propagation.

\vspace{-5pt}
\subsection{Control Coupling Experiment}
\label{Sec: Control Coupling Experiment}

To further assess the practical impact of different baselines, we conduct a closed-loop UAV relative-control experiment with three UGVs and one UAV equipped with CREPES-X devices.
As shown in \autoref{fig:r5-control-setup}, each trial runs one relative localization variant with CoNi-MPC~\cite{zhang2023coni} onboard the UAV, which completes three laps of a $1.5\,\mathrm{m}$-radius circular trajectory in the reference UGV frame.
All variants use identical measurement weights and MPC parameters.

\begin{figure*}[tb]
    \centering
    \includegraphics[width=\fullwidth]{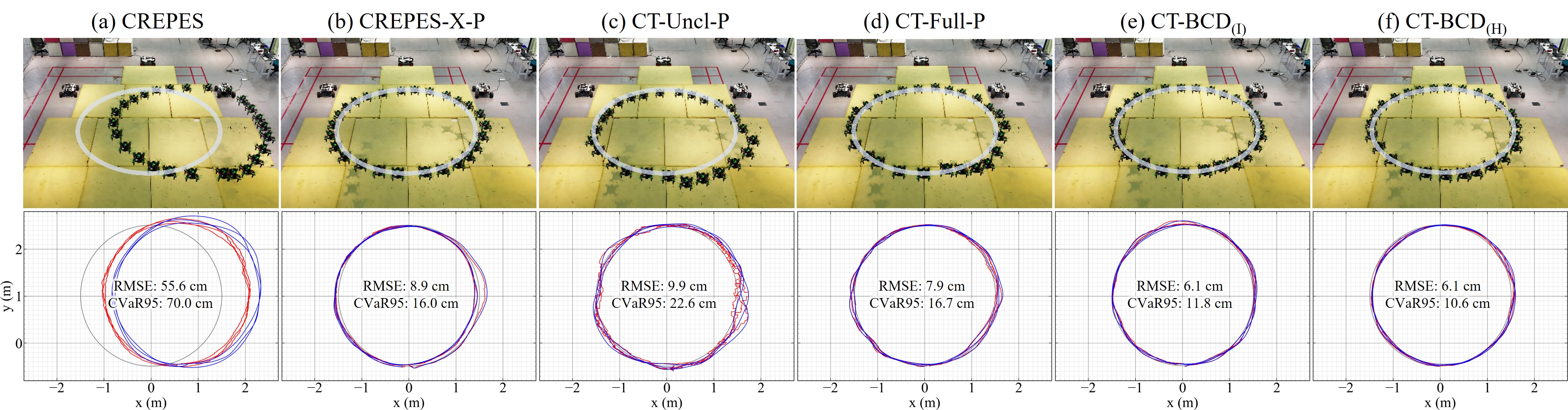}
    \caption{
        Closed-loop planar tracking results.
        The gray curve is the reference trajectory, the red curve is the trajectory estimated by the corresponding relative localization method, and the blue curve is the UAV trajectory recorded by the motion-capture system.
        The upper images are stacked at successive timestamps.
    }
    \label{fig:control-results}
    \vspace{-3pt}
\end{figure*}

\begin{figure}[htb]
    \centering
    \includegraphics[width=\halfwidth]{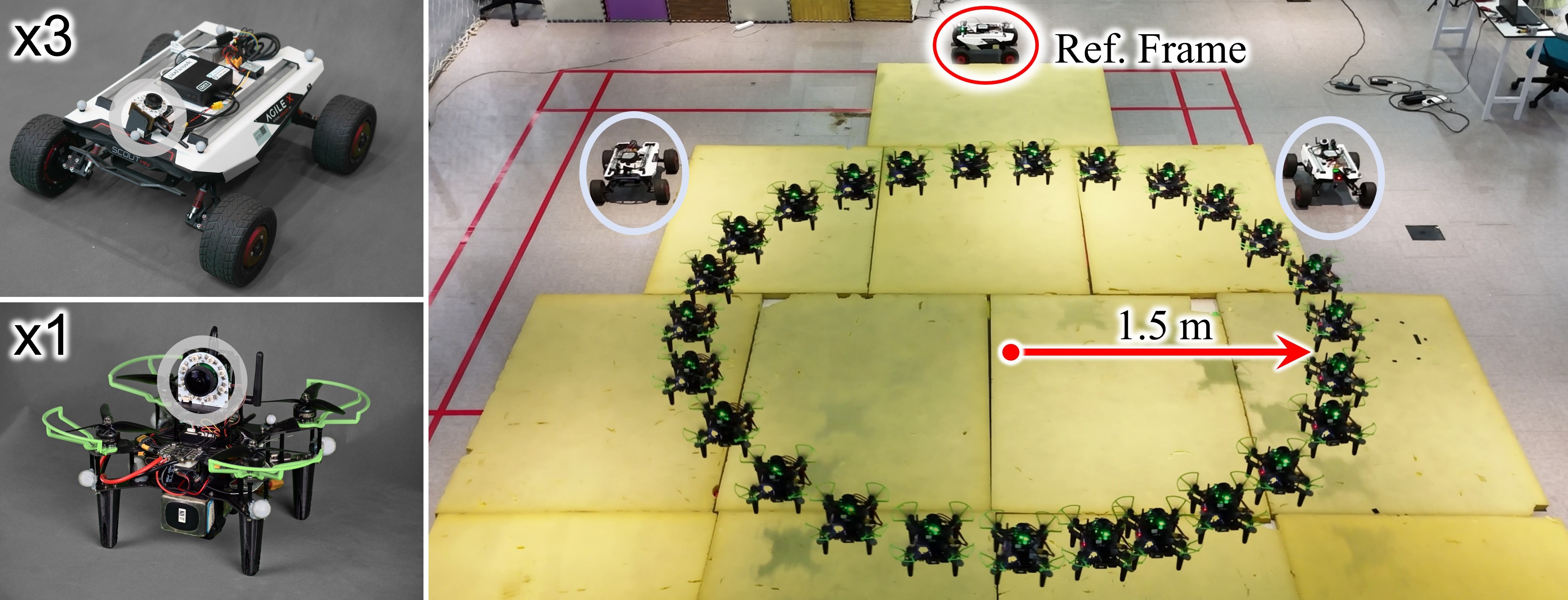}
    \caption{
        Closed-loop UAV-UGV relative-control setup with image stack.
        All robots carry CREPES-X devices, and the UAV tracks a circular reference trajectory expressed in the reference UGV frame.
    }
    \label{fig:r5-control-setup}
\end{figure}

We compare CREPES, CREPES-X-P, CT-Uncl-P, CT-Full-P, CT-BCD\textsubscript{(I)}, and CT-BCD\textsubscript{(H)} using the XY-plane tracking RMSE and $\mathrm{CVaR}_{0.95}$ (the mean of the largest $5\%$ of instantaneous tracking errors).
As shown in \autoref{fig:control-results},
CREPES yields a much larger RMSE\,/\,$\mathrm{CVaR}_{0.95}$ of $55.6\,/\,70.0\,\mathrm{cm}$, mainly due to velocity bias caused by omitting IMU-bias estimation, whereas CREPES-X-P reduces these errors to $8.9\,/\,16.0\,\mathrm{cm}$.
Among the CT variants, CT-Uncl-P yields $9.9\,/\,22.6\,\mathrm{cm}$, CT-Full-P yields $7.9\,/\,16.7\,\mathrm{cm}$, and CT-BCD\textsubscript{(I)}\,/\,CT-BCD\textsubscript{(H)} yield $6.1\,/\,11.8\,\mathrm{cm}$ and $6.1\,/\,10.6\,\mathrm{cm}$, respectively.
Compared with CT-Uncl-P, CT-Full-P reduces the tracking RMSE and $\mathrm{CVaR}_{0.95}$ by approximately $20\%$ and $26\%$, respectively, consistent with the shorter propagation interval enabled by clamping.
Relative to CREPES-X-P and CREPES, CT-BCD\textsubscript{(H)} reduces the errors by around $31\%/34\%$ and $89\%/85\%$, respectively.
Relative to CT-Full-P and CT-Uncl-P, the reductions are approximately $23\%/37\%$ and $38\%/53\%$, respectively.

These results indicate that \CT{} improves average tracking accuracy and tail-error robustness by providing latest-timestamp, measurement-corrected feedback, with periodic full-batch refinement further suppressing large errors.

\vspace{-8pt}
\subsection{Online Time-Offset Estimation}
\label{Sec: Online Time-Offset Estimation}

This section evaluates the real-world performance of the proposed online time-offset estimation in the \CT{} framework.
The experiments are conducted on the \texttt{\textbf{TUS}} sequences, and the quantitative results are summarized in \autoref{Tab: Ablation Study in Time-Offset Estimation}.

For each sequence, Devices $1 \sim 4$ exhibit initial temporal offsets ranging from several tens to hundreds of milliseconds with respect to Device 0.
Under these challenging conditions, CREPES-X suffers from severe performance degradation due to its inability to model or compensate for time-offsets.
In contrast, the CT-Full already demonstrates improved robustness even without explicit time-offset estimation. CT-Full consistently outperforms CREPES-X.
When online estimation of the time-offset $\tau$ is enabled, the performance improves substantially across all sequences.
Taking \texttt{TUS\_1} as an example, the average position error is reduced from $13.9\,\mathrm{cm}$ to $9.2\,\mathrm{cm}$, while the rotation error decreases from $11.8^{\circ}$ to $1.9^{\circ}$.
Similar trends are observed across \texttt{TUS\_2} to \texttt{TUS\_4}.
In addition, the estimated offsets closely match the measured ground truth. Across all sequences, the deviation remains within $1\,\mathrm{ms}$, confirming both accuracy and consistency.
Given that the sensors on the CREPES-X device provide timestamps with millisecond-level resolution, this level of accuracy effectively achieves reliable time synchronization in practice.

\begin{figure*}[tb]
    \centering
    \includegraphics[width=\textwidth]{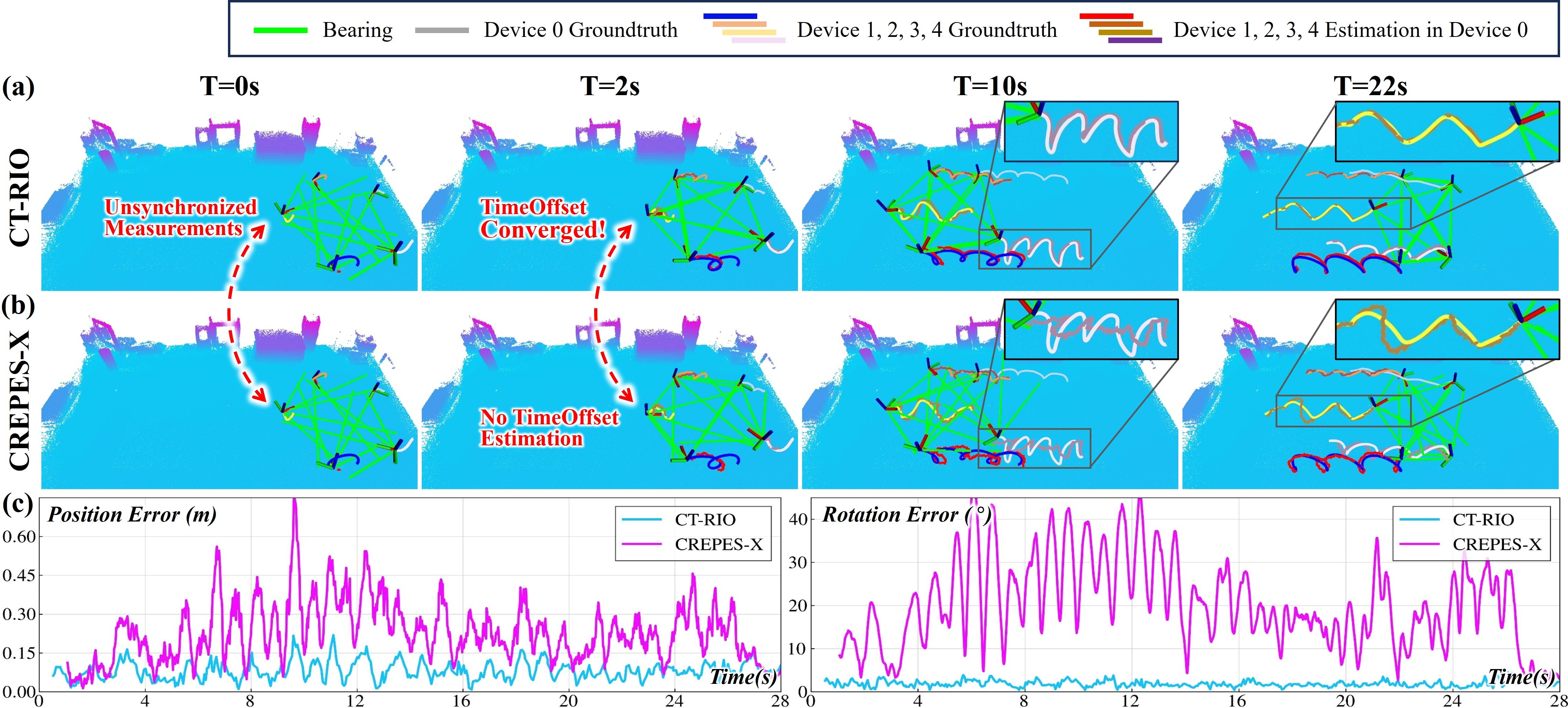}
    \caption{Results of online time-offset estimation in \texttt{TUS\_1}. The estimated relative trajectories are transformed to the world frame. (a) CT (\CT{}) outputs with online time-offset estimation. (b) CREPES-X outputs without time-offset estimation. (c) Position and rotation errors of device $4$ (reference frame is device $0$). The errors spike in CREPES-X due to the uncorrected time-offsets.}
    \label{fig:ctvsdt_time-offset}
\end{figure*}

\begin{table}[t]
    \centering
    \caption{
        Ablation Study in Time-Offset Estimation (Latest Output, Color Definition: \cfirst{best}, \csecond{second-best}, determined by position RMSE, and by rotation RMSE if equal)
    }
    \label{Tab: Ablation Study in Time-Offset Estimation}
    \resizebox{\halfwidth}{!}{
        \begin{tabular}{c@{\hskip 3pt}c@{\hskip 5pt}ccccc}
            \toprule
            \multirow{3}{*}{\vspace{-5pt}Exp. }
             &
            \multirow{3}{*}{\vspace{-5pt}Dev. }
             &
            CREPES-X
             &
            CT-Full
             &
            \multicolumn{3}{c}{CT-Full with $\tau$ Estimation}
            \\
            \cmidrule(r){3-3}\cmidrule(r){4-4}\cmidrule(r){5-7}
             &
             &
            Pos / Rot
             &
            Pos / Rot
             &
            Pos / Rot
             &
            Estimated ${\tau}$
             &
            True $\tau$
            \\
             &
             &
            ($\mathrm{cm}$) / ($\degree$)~
             &
            ($\mathrm{cm}$) / ($\degree$)~
             &
            ($\mathrm{cm}$) / ($\degree$)~
             &
            ($\mathrm{ms}$)
             &
            ($\mathrm{ms}$)
            \\
            \midrule
            \multirow{4}{*}{\vspace{-5pt} \texttt{TUS\_1}}
             & 1 & 25.5 / 9.2  & \csecond{13.6 / 8.5}  & \cfirst{10.3 / 2.0} & 80.2 $\pm$ 0.5  & 80  \\
             & 2 & 16.7 / 7.6  & \csecond{11.1 / 8.0}  & \cfirst{6.5 / 1.7}  & 114.1 $\pm$ 1.3 & 114 \\
             & 3 & 13.2 / 6.7  & \csecond{13.0 / 6.9}  & \cfirst{11.0 / 1.9} & -71.6 $\pm$ 0.9 & -71 \\
             & 4 & 31.5 / 24.4 & \csecond{17.4 / 23.9} & \cfirst{8.9 / 2.0}  & 239.3 $\pm$ 0.8 & 239 \\
            \midrule
            \multirow{4}{*}{\vspace{-5pt} \texttt{TUS\_2}}
             & 1 & 22.7 / 9.8  & \csecond{13.0 / 9.4}  & \cfirst{8.3 / 1.6} & 87.4 $\pm$ 2.7  & 87  \\
             & 2 & 20.3 / 11.7 & \csecond{12.7 / 10.1} & \cfirst{6.3 / 1.5} & 121.7 $\pm$ 1.1 & 122 \\
             & 3 & 12.5 / 7.7  & \csecond{12.0 / 7.7}  & \cfirst{9.7 / 1.7} & -64.0 $\pm$ 2.6 & -64 \\
             & 4 & 29.4 / 18.9 & \csecond{13.3 / 18.0} & \cfirst{6.8 / 1.9} & 248.6 $\pm$ 2.0 & 248 \\
            \midrule
            \multirow{4}{*}{\vspace{-5pt} \texttt{TUS\_3}}
             & 1 & 29.4 / 9.2  & \csecond{14.9 / 9.4}  & \cfirst{10.6 / 2.0} & 95.1 $\pm$ 1.1  & 95  \\
             & 2 & 19.3 / 12.3 & \csecond{12.8 / 10.8} & \cfirst{6.7 / 2.0}  & 129.7 $\pm$ 1.2 & 130 \\
             & 3 & 14.6 / 5.9  & \csecond{12.3 / 5.6}  & \cfirst{10.7 / 2.2} & -54.5 $\pm$ 1.0 & -55 \\
             & 4 & 28.8 / 22.0 & \csecond{15.2 / 20.9} & \cfirst{8.4 / 2.0}  & 255.9 $\pm$ 1.0 & 256 \\
            \midrule
            \multirow{4}{*}{\vspace{-5pt} \texttt{TUS\_4}}
             & 1 & 24.7 / 11.7 & \csecond{14.5 / 9.9}  & \cfirst{10.6 / 2.9} & 101.1 $\pm$ 1.0 & 101 \\
             & 2 & 20.6 / 14.5 & \csecond{15.0 / 13.6} & \cfirst{7.1 / 2.0}  & 135.2 $\pm$ 2.0 & 135 \\
             & 3 & 13.4 / 5.7  & \csecond{12.9 / 5.7}  & \cfirst{11.1 / 2.9} & -50.5 $\pm$ 1.9 & -50 \\
             & 4 & 22.3 / 19.1 & \csecond{13.5 / 18.6} & \cfirst{8.0 / 3.0}  & 263.1 $\pm$ 0.9 & 263 \\
            \bottomrule
        \end{tabular}
    }
    \vspace{5pt}
\end{table}

\begin{figure}[t]    
    \centering
    \includegraphics[width=0.9\halfwidth]{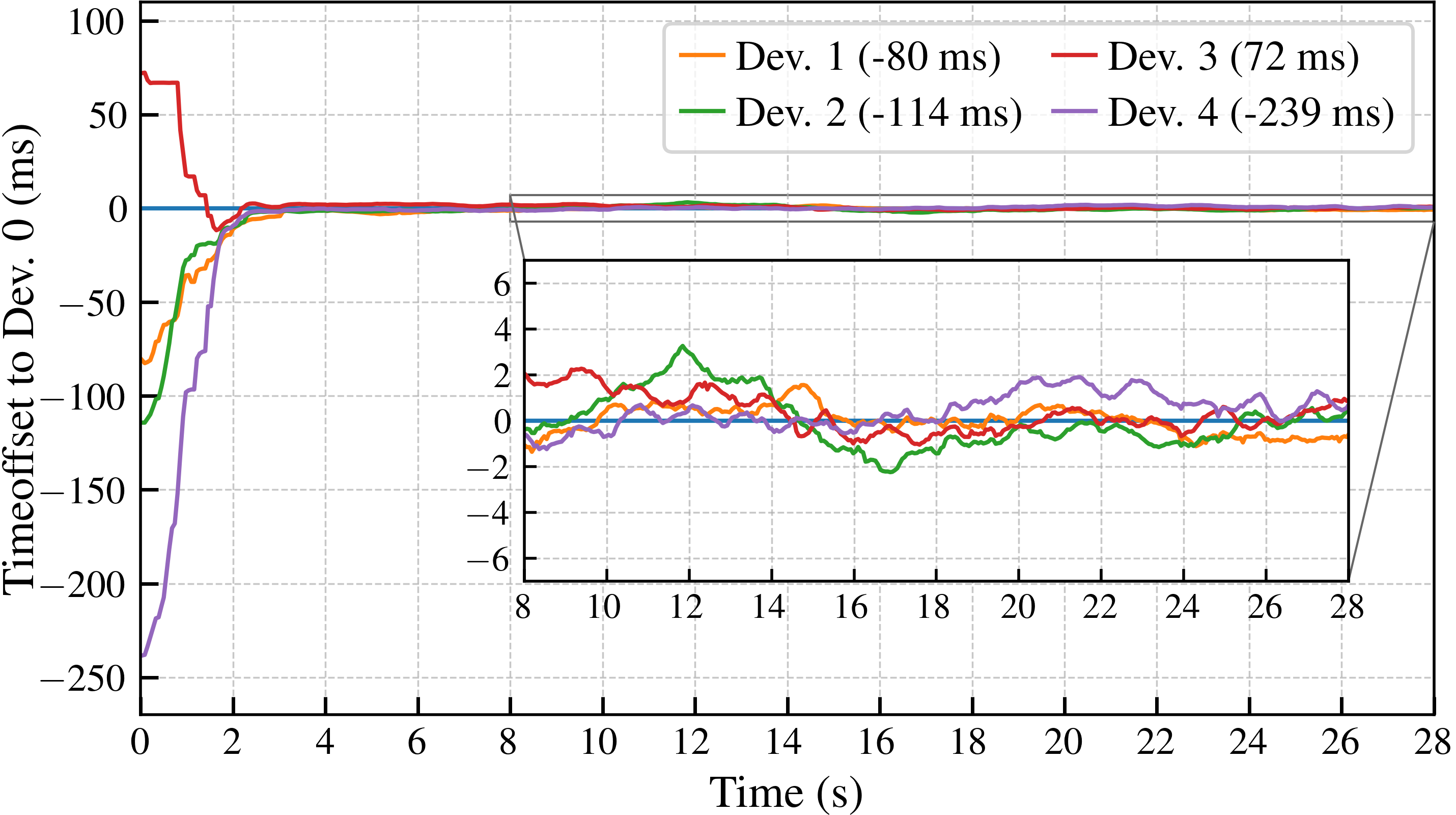}
    \caption{
        Time-offset estimation error curves of all devices in \texttt{TUS\_1}.
        The estimates converge rapidly from various initial offsets and remain stable throughout the sequence.
    }
    \label{Fig: Online Temporal Calibration}
\end{figure}

We further visualize the online time-offset estimation process of \CT{} on \texttt{TUS\_1} in \autoref{fig:ctvsdt_time-offset}.
As shown in \autoref{fig:ctvsdt_time-offset}(a), the initial misalignment (when $T=0\,\mathrm{s}$) of the bearing observations gradually improves, aligning with the target position as the time-offset is corrected (when $T=2\,\mathrm{s}$). This indicates that the estimated time-offsets converge within approximately $2$ seconds.
In contrast, CREPES-X lacks time-offset estimation, resulting in continuously misaligned bearing observations, as seen in \autoref{fig:ctvsdt_time-offset}(b).
Consequently, the incorrect observation timestamps lead to significant errors in both position and rotation estimation. This effect is evident in \autoref{fig:ctvsdt_time-offset}(c), where CREPES-X shows large error spikes, while \CT{} maintains accuracy similar to other hard-synchronized sequences.
Finally, \autoref{Fig: Online Temporal Calibration} illustrates the convergence behavior of the estimated time-offsets for all devices in \texttt{TUS\_1}.
The estimates rapidly converge from large initial offsets and remain stable throughout the sequence.

Overall, these results demonstrate that the proposed method not only achieves reliable time-offset estimation under real-world noise, but also directly improves the quality of trajectory estimation across diverse sequences.

\subsection{Scheduling and Runtime Analysis}
\label{Sec: Scheduling and Runtime Analysis}

To further analyze \CT{}, we evaluate how the full-batch optimization interval affects the accuracy of the Hybrid strategy across the real-world scenarios and quantify the resulting accuracy-computation trade-off.
We consider intervals of $100$, $200$, and $500\,\mathrm{ms}$, with all other estimator settings unchanged.
For each interval, we report CT-BCD, whose robot-wise IA-BCD updates incorporate the latest inter-robot measurements, and CT-Full-P, which propagates the latest full-batch state to the current IMU timestamp.
To obtain controlled comparisons for \texttt{\textbf{TEN}} at $100$ and $200\,\mathrm{ms}$, we replay the recorded sequences offline while varying only the full-batch interval.

\begin{table}[t]
    \centering
    \caption{
        Mean RMSE per Scenario at Different Full-Batch Intervals (Latest Output).
    }
    \label{Tab: RMSE at Different Full-Batch Intervals}
    \resizebox{\halfwidth}{!}{%
        \begin{tabular}{lcccccc}
            \toprule
            \multirow{4}{*}{\vspace{-4pt}Exp.}
             & \multicolumn{2}{c}{$100\,\mathrm{ms}$}
             & \multicolumn{2}{c}{$200\,\mathrm{ms}$}
             & \multicolumn{2}{c}{$500\,\mathrm{ms}$} \\
            \cmidrule(lr){2-3}\cmidrule(lr){4-5}\cmidrule(lr){6-7}
             & $E_{\text{CT-BCD}}$
             & $E_{\text{CT-Full-P}}$
             & $E_{\text{CT-BCD}}$
             & $E_{\text{CT-Full-P}}$
             & $E_{\text{CT-BCD}}$
             & $E_{\text{CT-Full-P}}$ \\
             & Pos / Rot
             & Pos / Rot
             & Pos / Rot
             & Pos / Rot
             & Pos / Rot
             & Pos / Rot \\
             & ($\mathrm{cm}$) / ($\degree$)
             & ($\mathrm{cm}$) / ($\degree$)
             & ($\mathrm{cm}$) / ($\degree$)
             & ($\mathrm{cm}$) / ($\degree$)
             & ($\mathrm{cm}$) / ($\degree$)
             & ($\mathrm{cm}$) / ($\degree$) \\
            \midrule
            \textbf{LOS}  & 6.4 / 2.1  & 8.0 / 2.1  & 6.4 / 2.1  & 9.2 / 2.1  & 6.5 / 2.1  & 28.7 / 2.2 \\
            \textbf{NLOS} & 5.7 / 2.0  & 7.0 / 2.0  & 5.7 / 2.0  & 9.5 / 2.0  & 5.9 / 2.0  & 13.7 / 2.0 \\
            \textbf{SDM}  & 4.5 / 2.1  & 6.0 / 2.3  & 4.6 / 2.1  & 7.4 / 2.6  & 4.7 / 2.2  & 11.5 / 2.8 \\
            \textbf{HDM}  & 14.5 / 2.4 & 18.0 / 2.6 & 15.1 / 2.4 & 20.7 / 2.5 & 16.1 / 2.4 & 34.5 / 2.8 \\
            \textbf{TEN}  & 6.3 / 2.3  & 7.9 / 2.3  & 6.4 / 2.3  & 8.7 / 2.3  & 6.5 / 2.3  & 14.2 / 2.3 \\
            \midrule
            \textbf{MEAN} & 7.5 / 2.2  & 9.4 / 2.3  & 7.6 / 2.2  & 11.1 / 2.3 & 7.9 / 2.2  & 20.5 / 2.4 \\
            \bottomrule
        \end{tabular}%
    }
    \vspace{2pt}
\end{table}

As shown in \autoref{Tab: RMSE at Different Full-Batch Intervals}, CT-BCD is relatively insensitive to the full-batch interval.
Its mean position RMSE is $7.5$, $7.6$, and $7.9\,\mathrm{cm}$ at intervals of $100$, $200$, and $500\,\mathrm{ms}$, respectively.
Across the individual sequences, increasing the interval from $100$ to $200\,\mathrm{ms}$ changes the CT-BCD position RMSE by no more than $5\%$.
At $500\,\mathrm{ms}$, the degradation remains limited in smooth-motion and ten-device experiments, but reaches $17\%$ in the most challenging high-dynamic sequence.
In contrast, the mean position RMSE of CT-Full-P increases from $9.4$ to $11.1$ and $20.5\,\mathrm{cm}$ as the interval grows.
Relative to CT-BCD at the same interval, CT-Full-P increases the mean position/rotation RMSE by approximately $26\%/3\%$, $46\%/5\%$, and $159\%/10\%$, respectively.
This widening gap shows that propagation becomes increasingly sensitive to the time elapsed since the last joint correction, whereas the latest CT-BCD output continues to incorporate newly available inter-robot measurements through IA-BCD updates.

We next quantify the aggregate computational cost of this accuracy-interval trade-off.
The complete estimator process tree is profiled during repeated runs on the real-world sequences, excluding the startup transient.
The experiments are conducted on an Intel NUC equipped with an Intel i7-1260P CPU and 16 logical cores.
As shown in \autoref{Tab: CPU Cost of Hybrid Optimization}, increasing the full-batch interval substantially reduces CPU utilization for both Full-Batch-Only and Hybrid.
Moreover, the overall CPU and memory usage remain moderate for both the five- and ten-device configurations on the Intel NUC, leaving computational headroom for other onboard tasks.

\begin{table}[t]
    \centering
    \caption{
        Computational load of different optimization settings of \CT{}.
        The ``core usage'' corresponds to the average number of CPU cores used by the specific component.
    }
    \label{Tab: CPU Cost of Hybrid Optimization}
    \resizebox{0.75\halfwidth}{!}{%
        \begin{tabular}{cccrrr}
            \toprule
            \multicolumn{1}{c}{\begin{tabular}[c]{@{}c@{}}Dev.\\Num.\end{tabular}}
             & \multicolumn{1}{c}{\begin{tabular}[c]{@{}c@{}}Interval\\($\mathrm{ms}$)\end{tabular}}
             & \multicolumn{1}{c}{\begin{tabular}[c]{@{}c@{}}Optimization\\Setting\end{tabular}}
             & \multicolumn{1}{c}{\begin{tabular}[c]{@{}c@{}}Core\\Usage\end{tabular}}
             & \multicolumn{1}{c}{\begin{tabular}[c]{@{}c@{}}Total\\CPU\\($\%$)\end{tabular}}
             & \multicolumn{1}{c}{\begin{tabular}[c]{@{}c@{}}Peak\\RAM\\(GiB)\end{tabular}}                                                             \\
            \midrule
            \multirow{7}{*}{5}
             & $100\,\mathrm{ms}$                                                                              & Full-Batch-Only & 1.46 & 9.14  & 0.189 \\
             & $200\,\mathrm{ms}$                                                                              & Full-Batch-Only & 0.98 & 6.12  & 0.188 \\
             & $500\,\mathrm{ms}$                                                                              & Full-Batch-Only & 0.72 & 4.48  & 0.192 \\
             & --                                                                                              & IA-BCD-Only     & 1.05 & 6.56  & 0.186 \\
             & $100\,\mathrm{ms}$                                                                              & Hybrid          & 2.10 & 13.11 & 0.199 \\
             & $200\,\mathrm{ms}$                                                                              & Hybrid          & 1.71 & 10.70 & 0.193 \\
             & $500\,\mathrm{ms}$                                                                              & Hybrid          & 1.27 & 7.95  & 0.195 \\
            \midrule
            \multirow{3}{*}{10}
             & $500\,\mathrm{ms}$                                                                              & Full-Batch-Only & 2.29 & 14.28 & 0.304 \\
             & --                                                                                              & IA-BCD-Only     & 2.71 & 16.91 & 0.281 \\
             & $500\,\mathrm{ms}$                                                                              & Hybrid          & 5.38 & 33.60 & 0.333 \\
            \bottomrule
        \end{tabular}%
    }
\end{table}



These results provide practical guidance for scheduling full-batch refinement.
We use $200\,\mathrm{ms}$ as a practical default for the five-device configuration on the evaluated platform, as it substantially reduces CPU usage relative to $100\,\mathrm{ms}$ while changing the position RMSE by at most $5\%$.
For aggressive motion or degraded inter-robot observations, a shorter $100$--$200\,\mathrm{ms}$ interval is preferable when computationally feasible.
For moderate motion under a constrained CPU budget, the interval can be relaxed to $500\,\mathrm{ms}$, if the potential accuracy loss in high-dynamic cases is acceptable.
For ten devices, one full-batch update requires more than $400\,\mathrm{ms}$; therefore, $500\,\mathrm{ms}$ is the shortest feasible interval that avoids overlapping full-batch jobs.
More generally, as the number of robots grows, the full-batch interval should be no shorter than the execution time of one full-batch update and should then be selected as the shortest feasible interval within the available CPU budget.

\section{Conclusion}
\label{Sec: Conclusion}


In this paper, we presented \CT{}, a continuous-time direct temporal-spatial relative localization system for multi-robot systems.
Operating purely on relative kinematics, \CT{} fuses asynchronous inter-robot observations and inertial data within a sliding-window framework.
By leveraging C-NUBS, it removes the query-time delay of prior unclamped methods and enables endpoint interpolation for real-time evaluation and initialization.
Closed-form extension and shrinkage routines modify only boundary control points, preserving previously optimized segments and supporting the proposed \textit{knot-keyknot} strategy for high-frequency outputs with motion-informative knots.
To ensure scalability, the full problem is decomposed into robot-wise subproblems solved via IA-BCD with periodic global refinement.
Together, these designs enable high-accuracy, low-latency, and high-frequency relative localization with online time-offset estimation.

Despite these advantages, several limitations remain.
Extending clamped splines to broader estimation tasks such as loop-closure SLAM introduces challenges in maintaining consistency, marginalization, and information preservation.
While we provide a convergence analysis for IA-BCD, the resulting conditions are conservative and mainly offer qualitative guidance; deriving tighter and practically meaningful bounds remains future work.
In addition, the keyknot selection strategy is empirically designed, and more principled approaches based on information gain are worth exploring.
Finally, the current system remains reference-centric, and developing fully distributed and certifiably convergent variants that scale to larger groups is an important direction for future research.

\appendices



\appendix
\section{Convergence Analysis for IA-BCD}
\label{App: Convergence Analysis for Asynchronous BCD}

We analyze IA-BCD on a fixed finite set of active control-point blocks.
Let \(x\) collect these blocks, and let
\(F_k(x):=f_{t(k)}(x)\)
denote the part of the objective that depends on \(x\), frozen during the \(k\)-th block update.
We assume that the iterates remain in a bounded set \(\Omega\), on which all \(F_k\) are uniformly lower bounded and have \(L\)-Lipschitz continuous gradients.
An important property of CT-RIO is that the objective variation associated with any fixed control-point block is finite.
Due to the local support of B-splines, each control point affects the trajectory only over a finite temporal interval.
Since all sensor streams have finite sampling rates, only finitely measurement factors can depend on any fixed control point.
Consequently, for any fixed finite set of control-point blocks, there exists a finite iteration \(K_x\) after which newly arriving measurements introduce no additional factors depending on \(x\).
Therefore,
\vspace{2pt}
\begin{equation}
    F_{k+1}(z)\!=\!F_k(z),
    \nabla F_{k+\!1}(z)\!=\!\nabla F_k(z),
    \forall z\!\in\!\Omega,
    k\!\ge\!K_x.
    \label{eq:eventually_fixed}
    \vspace{2pt}
\end{equation}
Thus, although the newest spline tail continues to evolve, the objective associated with every fixed historical trajectory region eventually becomes fixed.

Let \(i_k\) denote the block updated at iteration \(k\).
We assume that every active block is updated at least once within any \(B\) consecutive updates.
The asynchronous delay is bounded by \(\tau\), such that
\begin{equation}
    \|x^k-\hat{x}^k\|
    \le
    \sum_{q=k-\tau}^{k-1}\|\Delta^q\|.
\end{equation}

At iteration \(k\), the LM block update is
\begin{equation}
    \Delta^k
    =
    -(J_k^\top J_k+\lambda I)^{-1}
    J_k^\top r_k(\hat{x}^k),
    \quad
    x^{k+1}=x^k+\Delta^k,
\end{equation}
where \(\Delta^k\) is nonzero only on block \(i_k\).

For the frozen objective \(F_k\), \(L\)-smoothness, the bounded-delay relation, Young's inequality with \(\eta>0\), and the LM update give
\begin{align}
     & F_k(x^{k+1})-F_k(x^k)
    \nonumber                \\[-1ex]
     & \le
    \langle
    \nabla F_k(x^k)
    \!-\!
    \nabla F_k(\hat{x}^k),
    \Delta^k
    \rangle
    \!+\!
    \langle
    \nabla_{i_k}F_k(\hat{x}^k),
    \Delta^k
    \rangle
    \!+\!
    \frac{L}{2}\|\Delta^k\|^2
    \nonumber                \\[-1ex]
     & \le
    \frac{L}{2\eta}
    \sum_{q=k-\tau}^{k-1}\|\Delta^q\|^2
    +
    \left(
    \frac{\tau\eta L}{2}
    +
    \frac{L}{2}
    -
    \lambda
    \right)
    \|\Delta^k\|^2.
    \label{eq:frozen_descent}
\end{align}
The last inequality follows from $\nabla_{i_k}F_k(\hat{\mathbf{x}}^k)=-(\mathbf{J}_k^\top\mathbf{J}_k+\lambda\mathbf I)\Delta^k$,
which gives
$\langle\nabla_{i_k}F_k(\hat{\mathbf{x}}^k),\Delta^k\rangle
\leq-\lambda\|\Delta^k\|^2$.

Before iteration \(K_x\), incremental measurements may modify the objective associated with \(x\).
Define the corresponding objective and gradient perturbation bounds as
\begin{equation}
    \rho_k
    :=
    \sup_{z\in\Omega}
    \left[
        F_{k+1}(z)-F_k(z)
        \right]_+,
\end{equation}
and
\begin{equation}
    d_k
    :=
    \sup_{z\in\Omega}
    \|
    \nabla F_{k+1}(z)-\nabla F_k(z)
    \|.
\end{equation}
By \eqref{eq:eventually_fixed},
\(\rho_k=d_k=0\) for all \(k\ge K_x\).
Since only finitely many factor insertions can affect \(x\), the two perturbation sequences have finite support and therefore satisfy
\begin{equation}
    \sum_{k=0}^{\infty}\rho_k<\infty,
    \qquad
    \sum_{k=0}^{\infty}d_k<\infty.
    \label{eq:finite_perturbation}
\end{equation}

Define the delay-compensated Lyapunov function
\begin{equation}
    \mathcal{L}_k
    =
    F_k(x^k)
    +
    \frac{L}{2\eta}
    \sum_{q=k-\tau}^{k-1}
    (q-k+\tau+1)\|\Delta^q\|^2.
\end{equation}

Combining \eqref{eq:frozen_descent} with the finite objective perturbation yields
\begin{equation}
    \mathcal{L}_{k+1}-\mathcal{L}_k
    \le
    -\alpha\|\Delta^k\|^2+\rho_k,
\end{equation}
where
\begin{equation}
    \alpha
    =
    \lambda
    -
    \frac{L}{2}
    -
    \frac{\tau\eta L}{2}
    -
    \frac{L\tau}{2\eta}.
\end{equation}

Taking \(\eta=1\) and
\(
\lambda>L(\tau+\frac{1}{2})
\)
ensures \(\alpha>0\).
Using the uniform lower bound of \(F_k\) and
\(\sum_k\rho_k<\infty\), summing the Lyapunov inequality gives
\begin{equation}
    \sum_{k=0}^{\infty}\|\Delta^k\|^2<\infty,
    \qquad
    \|\Delta^k\|\to0.
    \label{eq:step_convergence}
\end{equation}

Assume further that the largest singular value of each block Jacobian is bounded by \(S\).
The LM update and the bounded-delay relation then imply, for the active block,
\vspace{-4pt}
\begin{equation}
    \|\nabla_{i_k}F_k(x^k)\|
    \le
    L
    \sum_{q=k-\tau}^{k-1}\|\Delta^q\|
    +
    (\lambda+S^2)\|\Delta^k\|
    \rightarrow0.
\end{equation}

For any block \(j\), let \(\ell_j(k)\) denote its most recent update before iteration \(k\).
By the bounded-update-gap assumption,
\(k-\ell_j(k)\le B\).
Therefore,
\begin{align}
    \|\nabla & _jF_k(x^k)\|
    \le{}
    \\[-1ex]
             & \|\nabla_jF_{\ell_j(k)}(x^{\ell_j(k)})\|
    +
    L\sum_{q=\ell_j(k)}^{k-1}\|\Delta^q\|
    +
    \sum_{q=\ell_j(k)}^{k-1}d_q
    \to0.
    \notag
\end{align}
Hence,$\|\nabla F_k(x^k)\|\rightarrow0$.

For CT-RIO, the above perturbation condition is therefore a consequence of the spline structure rather than an additional assumption.
The continually changing objective is confined to the newest spline tail.
For every fixed historical control-point region, local support implies that only finitely many factors can constrain it, so its associated objective eventually becomes fixed.
Accordingly, under the stated smoothness, bounded-delay, bounded-update-gap, and damping conditions, every accumulation point of the repeatedly optimized historical variables is a first-order stationary point of the corresponding fixed objective.



\vspace{-8pt}
\bibliographystyle{IEEEtran}
\bibliography{root}

\end{document}